\def\year{2021}\relax

\newif\ifarxiv\arxivtrue

\newif\ifnotarxiv
\ifarxiv \notarxivfalse \else \notarxivtrue \fi

\ifarxiv
\newcommand{\arxiv}[1]{#1}
\else
\newcommand{\arxiv}[1]{}
\fi

\ifarxiv
\newcommand{\notarxiv}[1]{}
\else
\newcommand{\notarxiv}[1]{#1}
\fi

\ifarxiv
    \documentclass[11pt]{article}
    \usepackage[margin=1in]{geometry}
\else
    \documentclass[letterpaper]{article} 
    \usepackage{aaai21}  
\fi
\usepackage{times}  
\usepackage{helvet} 
\usepackage{courier}  
\usepackage[hyphens]{url}  
\usepackage{graphicx} 
\ifarxiv
    \usepackage[authoryear]{natbib}
\else
    \usepackage{natbib}  
\fi
\usepackage{caption} 
\ifnotarxiv
\fi
\usepackage[table]{xcolor}
\ifarxiv
    \usepackage{color}
    \definecolor{cornellred}{rgb}{0.7, 0.11, 0.11}
    \definecolor{dgreen}{rgb}{0.0, 0.5, 0.0}
    \definecolor{ballblue}{rgb}{0.13, 0.67, 0.8}
    \definecolor{royalblue(web)}{rgb}{0.25, 0.41, 0.88}
    \definecolor{bleudefrance}{rgb}{0.19, 0.55, 0.91}
    \definecolor{royalazure}{rgb}{0.0, 0.22, 0.66}
    \usepackage[breaklinks]{hyperref}
    \hypersetup{
        colorlinks = true,
        linkcolor = cornellred,
        citecolor = royalazure,
        linkbordercolor = white,
        urlcolor  = cornellred
    }
\fi
\usepackage{refcount}

\title{Objective-Based Hierarchical Clustering of Deep Embedding Vectors}

\newcommand*\samethanks[1][\value{footnote}]{\footnotemark[#1]}
\ifarxiv
\author{
    Stanislav Naumov\thanks{ITMO University.}\\{\tt josdas@mail.ru} \and
    Grigory Yaroslavtsev\thanks{Indiana University, Bloomington. Research supported by NSF award CCF-1657477 and Facebook Faculty Research Award.}\\{\tt grigory@grigory.us} \and
    Dmitrii Avdiukhin\samethanks[2]\\{\tt davdyukh@iu.edu}
}
\else
\author{
    Stanislav Naumov\textsuperscript{\rm 1},
    Grigory Yaroslavtsev\textsuperscript{\rm 2},
    Dmitrii Avdiukhin\textsuperscript{\rm 2}
}
\affiliations{

    \textsuperscript{\rm 1}ITMO University\\
    \textsuperscript{\rm 2}Indiana University\thanks{Supported by NSF CCF-1657477 and Facebook Faculty Award.}
    \\
    
    josdas@mail.ru, grigory.yaroslavtsev@gmail.com, davdyukh@iu.edu
}
\fi

\usepackage[utf8]{inputenc} 
\usepackage{url}            
\usepackage{booktabs}       
\usepackage{amsfonts}       
\usepackage{nicefrac}       
\usepackage{microtype}      
\usepackage{amsmath,amssymb,amsfonts}
\usepackage{nicefrac}
\usepackage{tabu}
\usepackage{makecell}
\usepackage[ruled,lined,linesnumbered,noend]{algorithm2e}
\usepackage{algpseudocode}
\usepackage{verbatim}
\usepackage{booktabs}
\usepackage{mathtools}
\usepackage{ifthen}
\usepackage{braket}
\usepackage{xspace}
\usepackage{paralist}
\usepackage{enumitem}
\usepackage{pgfplots}
\pgfplotsset{compat=newest}
\usepackage{ragged2e}
\usepackage{soul}
\usepackage[switch]{lineno}

\usepackage{etoc}
\usepgfplotslibrary{groupplots}
\usepackage{subcaption}
\arxiv{\pagestyle{plain}}

\newtheorem{theorem}{Theorem}[section]

\newtheorem{lemma}[theorem]{Lemma}
\newtheorem{proposition}[theorem]{Proposition}
\newtheorem{definition}[theorem]{Definition}

\newcommand{\multiline}[1]{
        \begin{tabular}{@{}c@{}}
        #1
        \end{tabular}
}

\newenvironment{proof}{\noindent{\bf Proof : \ }}{\hfill$\Box$\par\medskip}

\DeclareMathOperator*{\argmin}{\arg\!\min}

\let\oldnl\nl
\newcommand{\nonl}{\renewcommand{\nl}{\let\nl\oldnl}}

\newcommand{\items}{V}
\newcommand{\tree}{\mathcal T}
\newcommand{\lca}{\mathrm{LCA}}
\newcommand{\clustercount}{K}

\newcommand{\simty}{w}
\newcommand{\dist}{d}
\newcommand{\purity}{\mathrm{DP}}
\newcommand{\objective}{Q}

\newcommand{\ourchar}{+}

\newcommand{\dasgupta}[1][]{\objective_D^{#1}}
\newcommand{\ckmm}[1][]{\objective_{C}^{#1}}
\newcommand{\mw}[1][]{\objective_{M}^{#1}}
\newcommand{\ourmw}[1][]{\objective_{M^{\ourchar}}^{#1}}
\newcommand{\ourckmm}[1][]{\objective_{C^{\ourchar}}^{#1}}

\newcommand{\aproxobjective}{\alpha}
\newcommand{\aproxmw}[1][]{\aproxobjective_{M}^{#1}}
\newcommand{\aproxckmm}[1][]{\aproxobjective_{C}^{#1}}
\newcommand{\aproxourckmm}[1][]{\aproxobjective_{C^{\ourchar}}^{#1}}
\newcommand{\aproxourmw}[1][]{\aproxobjective_{M^{\ourchar}}^{#1}}

\newcommand{\normchar}{*}

\newcommand{\upperbound}{\objective^{ub}}

\newcommand{\dtriple}{\objective^{\Delta}_{\tree,dis}}

\newcommand{\balancedBisection}{\textsc{Balanced-Bisection-HC}\xspace}
\newcommand{\avglink}{\textsc{Average-Linkage}\xspace}
\newcommand{\random}{\textsc{Random}\xspace}
\newcommand{\tpath}{\textsc{Path}\xspace}
\newcommand{\maxsat}{\textsc{Max-2-Sat}\xspace}
\newcommand{\bmaxsat}{\textsc{Balanced Max-2-Sat}\xspace}
\newcommand{\bmaxsatpar}{\bmaxsat \textsc{ Partitioning}\xspace}
\newcommand{\f}{f}
\newcommand{\optsat}{OPT_{SAT}}
\newcommand{\topt}{\tree^*}
\newcommand{\topts}{\widehat{\topt}}
\newcommand{\opt}{OPT}
\newcommand{\algsat}{ALG_{SAT}}
\newcommand{\alg}{ALG}
\newcommand{\sol}{\mathcal{O}}
\newcommand{\ouralgo}{\textsc{B++\&C}\xspace}
\newcommand{\ouralgofull}{\textsc{Bisect++ and Conquer}\xspace}
\newcommand{\oursatalgo}{\textsc{B2SAT\&C}\xspace}

\newcommand{\eps}{\varepsilon}
\newcommand{\R}{\mathbb R}

\newcommand{\vx}{\mathbf x}
\newcommand{\vy}{\mathbf y}

\newcommand{\xt}[1]{\vx^{(#1)}}
\newcommand{\yt}[1]{\vy^{(#1)}}

\newcommand{\xti}[2]{x^{(#1)}_{#2}}

\newcommand{\iternum}{I}
\newcommand{\noise}{r}
\newcommand{\rate}{\eta}
\newcommand{\body}{\mathcal B}
\newcommand{\imbalance}{\delta}

\newcommand{\maxavgsz}{\theta}

\newcommand{\avglName}{\textsc{AverageLinkage}}
\newcommand{\gdName}{\textsc{GradientDescentPartitioning}}

\newcommand{\simmat}{W}

\newcommand{\kerleft}{\phi}
\newcommand{\kerright}{\psi}
\newcommand{\kermatleft}{\Phi}
\newcommand{\kermatright}{\Psi}

\newcommand{\cossim}{{cos\text{-}sim}}
\newcommand{\Ltwosqr}{{L_2^2}}
\newcommand{\rbf}[1]{{RBF_{#1}}}
\newcommand{\laplacian}[1]{{Laplacian_{#1}}}

\newcommand{\ghhc}{\textsc{gHHC}\xspace}
\newcommand{\grinch}{\textsc{Grinch}\xspace}
\newcommand{\perch}{\textsc{Perch}\xspace}

\newcommand{\birch}{\textsc{BIRCH}\xspace}
\newcommand{\hkmeans}{\textsc{BKMeans}\xspace}
\newcommand{\hdbscan}{\textsc{RobustSL}\xspace}
\newcommand{\Random}{\textsc{Random}\xspace}
\newcommand{\randomCut}{\textsc{Random cut}\xspace}
\newcommand{\completeLinkage}{\textsc{CompleteL}\xspace}
\newcommand{\wards}{\textsc{Ward’sM}\xspace}
\newcommand{\averageLinkage}{\textsc{AverageL}\xspace}
\newcommand{\singleLinkage}{\textsc{SingleL}\xspace}
\newcommand{\affinityClustering}{\textsc{AffinityC}\xspace}
\newcommand{\bbisection}{\textsc{B++\&C}($\delta=0$)\xspace}
\newcommand{\twosat}{\textsc{2SAT}\xspace}

\newcommand{\imagenetVTwo}{ImageNetV2\xspace}
\newcommand{\nabirds}{NaBirds\xspace}
\newcommand{\imagenet}{ImageNet\xspace}
\newcommand{\imagenetins}{ImageNetI\xspace}
\newcommand{\sstTwo}{SST-2\xspace}
\newcommand{\twitter}{Twitter\xspace}
\newcommand{\wikipedia}{Wikipedia\xspace}

\newcommand{\classicSetup}{features\xspace}
\newcommand{\generalizationSetup}{generalization\xspace}
\newcommand{\shiftSetup}{shift\xspace}
\newcommand{\wordSetup}{word\xspace}
\newcommand{\sentenceSetup}{sentence\xspace}
\newcommand{\generalSetup}{general\xspace}

\newcommand{\columnWidth}{1.85cm}
\newcolumntype{P}{>{\Centering}p{\columnWidth}}
\newcommand{\columnWidthS}{1.6cm}
\newcolumntype{S}{>{\Centering}p{\columnWidthS}}

\newcommand{\narrowText}[1]{\scalebox{1}[1]{#1}}
\newcommand{\MimagenetVTwo}{\narrowText{\multiline{\imagenetVTwo\\ResNet34\\CV\\\generalizationSetup\\small}}}
\newcommand{\Mnabirds}{\narrowText{\multiline{\nabirds\\ResNet34\\CV\\\shiftSetup\\medium}}}
\newcommand{\MimagenetSmall}{\narrowText{\multiline{\imagenet\\ResNet34\\CV\\\classicSetup\\large}}}
\newcommand{\MimagenetBig}{\narrowText{\multiline{\imagenet\\Inception\\CV\\\classicSetup\\large}}}
\newcommand{\MsstTwo}{\narrowText{\multiline{\sstTwo\\SBERT\\NLP\\\sentenceSetup\\medium}}}
\newcommand{\Mtwitter}{\narrowText{\multiline{\twitter\\Glove\\NLP\\\wordSetup\\large}}}
\newcommand{\Mwikipedia}{\narrowText{\multiline{\wikipedia\\Word2vec\\NLP\\\wordSetup\\large}}}
\newcommand{\MtabSetup}{\narrowText{\multiline{Dataset\\Method\\Domain\\Setting\\Size}}}

\newcommand{\glass}{Glass\xspace}
\newcommand{\spambase}{Spambase\xspace}
\newcommand{\aloi}{ALOI\xspace}
\newcommand{\covType}{CovType\xspace}
\newcommand{\imageNetAdditional}{ImageNet\xspace}

\newcommand{\Mglass}{\narrowText{\multiline{\glass\\--\\--\\\generalSetup\\small}}}
\newcommand{\Mspambase}{\narrowText{\multiline{\spambase\\--\\--\\\generalSetup\\small}}}
\newcommand{\Maloi}{\narrowText{\multiline{\aloi\\--\\--\\\generalSetup\\medium}}}
\newcommand{\McovType}{\narrowText{\multiline{\covType\\--\\--\\\generalSetup\\medium}}}
\newcommand{\MimageNetAdditional}{\narrowText{\multiline{\imagenetins\\Inception\\CV\\\classicSetup\\large}}}

\newcommand{\hltodo}[2]{\hl{#1}\todo{#2}}

\newcommand{\avgImprovementCkmm}{20$\%$\xspace}
\newcommand{\avgImprovementMw}{5$\%$\xspace}
\newcommand{\rangeImprovementCkmm}{4-59$\%$\xspace}
\newcommand{\rangeImprovementMw}{2-17$\%$\xspace}

\newcommand{\lcaind}[1]{\hat{#1}}

\definecolor{bleudefrance}{rgb}{0.19, 0.55, 0.91}
\definecolor{bittersweet}{rgb}{1.0, 0.44, 0.37}
\definecolor{brightlavender}{rgb}{0.75, 0.58, 0.89}
\definecolor{dandelion}{rgb}{0.94, 0.88, 0.19}
\definecolor{chocolate}{rgb}{0.82, 0.41, 0.12}
\definecolor{emerald}{rgb}{0.31, 0.78, 0.47}
\definecolor{lightsalmon}{rgb}{1.0, 0.63, 0.48}
\newcommand{\newtext}[1]{{#1}}

\ifarxiv
\newcommand{\header}[1]{\paragraph{#1}}
\else
\newcommand{\header}[1]{\textbf{#1}}
\fi

\ifarxiv
\newcommand{\citepOur}[2][]{\citep*[#1]{#2}}
\else
\newcommand{\citepOur}[2][]{\citep[#1]{#2}}
\fi

\ifarxiv
\newcommand{\citetOur}[2][]{\citet*[#1]{#2}}
\else
\newcommand{\citetOur}[2][]{\citet[#1]{#2}}
\fi

\ifnotarxiv
\renewcommand{\cellcolor}[2][]{}
\fi

\arxiv{\usepackage{stfloats}}

\begin{document}

\maketitle




\etocdepthtag.toc{mtchapter}
\etocsettagdepth{mtchapter}{subsection}
\etocsettagdepth{mtappendix}{none}

\begin{abstract}

We initiate a comprehensive experimental study of objective-based hierarchical clustering methods on massive datasets consisting of deep embedding vectors from computer vision and NLP applications.
This includes a large variety of image embedding (ImageNet, ImageNetV2, NaBirds), word embedding (Twitter, Wikipedia), and sentence embedding (SST-2) vectors from several popular recent models (e.g. ResNet, ResNext, Inception V3, SBERT).
Our study includes datasets with up to $4.5$ million entries with embedding dimensions up to $2048$.

In order to address the challenge of scaling up hierarchical clustering to such large datasets we propose a new practical hierarchical clustering algorithm \ouralgo.
It gives a \avgImprovementMw/\avgImprovementCkmm improvement on average for the popular Moseley-Wang (MW) / Cohen-Addad et al. (CKMM) objectives (normalized) compared to a wide range of classic methods and recent heuristics. 
We also introduce a theoretical algorithm \oursatalgo which achieves a $0.74$-approximation for the CKMM objective in polynomial time.
This is the first substantial improvement over the trivial $2/3$-approximation achieved by a random binary tree.
Prior to this work, the best poly-time approximation of $\approx 2/3 + 0.0004$ was due to Charikar et al. (SODA'19).

\end{abstract}
\section{Introduction}

Vector embeddings, in particular those obtained via deep neural nets, are an extremely popular technique for representing unstructured data (e.g. images, text, videos, etc.) as vectors in a $d$-dimensional feature space. While resulting vectors are most frequently used for classification they can also serve as representations for other downstream machine learning tasks, including clustering, deduplication, recommendation systems, etc.
Flat clustering of vector embeddings has been studied extensively (e.g.~\citepOur{MinGLZCL18clusteringFeatures,GuerinGTN17cnnFeatures}). 
In this paper we focus on \emph{hierarchical clustering}, which has a large number of applications, including anomaly detection~\citepOur{DbD17outlier,ParwezRG17anmaly,ZhouyuWT05anomaly}, personalized recommendations~\citepOur{YuchenAVA14recommendation}
and construction of flat clusterings~\citepOur{SanderQLNK03}.
There are classical and recent approaches which allow one to learn a hierarchy on objects in either supervised ~\citepOur{WuTL19, NickelK17poincare} or unsupervised fashion~\citepOur{YangPB16,SuJinKI19,MathieuLMTT19}.
However, such approaches are substantially more expensive than hierarchical clustering of embedding vectors. Hence hierarchical clustering of deep vector embeddings has emerged as a computationally efficient alternative (e.g. for applications to face recognition~\citepOur{LinCC17face}). 

In this paper we focus on scalable algorithms for \emph{objective-based hierarchical clustering}, i.e. clustering which optimizes a certain well-defined objective function. Designing an objective function for hierarchical clustering which can be approximated efficiently is challenging, and only recently substantial progress has been made following the work by~\citetOur{Dasgupta15}.
In this paper we focus on two popular objectives inspired by it: a similarity-based objective introduced in~\citetOur{MoseleyW17} (MW) and a distance-based objective introduced in~\citetOur{CohenKMM19} (CKMM).

Intuitively, these objectives measure the quality of the resulting hierarchical clustering on a random triple of objects from the dataset. They incentivize solutions where the more similar pair in the triple is closer in the resulting hierarchy (see Sec~\ref{sec:prelims} for formal definitions).
Worst-case approximation algorithms for these objectives are known \citepOur{MoseleyW17,CharikarCN19,AhmadianCEMMLY20,CohenKMM19,AlonAV20}.
Beyond worst-case analysis has been given for the hierarchical stochastic block model~\citepOur{CohenKM17} and for vector data~\citepOur{CharikarCNY2019}.

We study performance of objective-based hierarchical clustering methods on large beyond worst-case datasets consisting of deep vector embeddings.
We perform experiments on massive datasets
    (number of objects $n$ is in range $[5 \cdot 10^4, 4.5 \cdot 10^6]$ and embedding dimension $d$ is in range $[100, 2048]$)
    of word, sentence, and image embedding vectors from the last hidden layer of various popular neural architectures. 
We study three types of algorithms:
    1) algorithms with rigorous guarantees for the MW/CKMM objectives,
    2) classic hierarchical agglomerative methods,
    3) some other popular algorithms without guarantees scalable to large data. 

While the best worst-case approximations for MW and CKMM objectives are $0.585$~\citepOur{AlonAV20} and $\approx \nicefrac 23 + 0.0004$~\citepOur{CharikarCN19} respectively, we show that in practice, for deep vector embeddings in computer vision and natural language processing, many algorithms achieve a much better approximation.
We conduct our experiments for cosine similarity (for MW) and squared Euclidean distance (for CKMM) due to their widespread use as similarity/distance measures in deep learning applications,
    but we believe that our findings are likely to hold for various other measures as well.

Given the popularity of various heuristics for hierarchical clustering, we don't aim to provide a full list of all possible approaches and objectives
    (see classic and recent surveys for an overview~\citepOur{MurtaghC12overview,MurtaghC17overview,JainMF1999overview,ChristopherPH08overview}).
A non-exhaustive list of other methods and objectives, which we omit in this study due to lack of rigorous guarantees for MW/CKMM or scalability issues, includes various spectral methods~\citepOur{WuCYXXA18kernelSpectral,HuangWWLK19spectral},
    LSH-based average-linkage~\citepOur{CochezM15avgLsh,AbboudCH19avgLsh}, various $k$-means algorithms~\citepOur{Zhong05cosineKmeans,WangGM19kernelKmaens,ChenZM20kernelKmeans} and a recent objective for bisecting $k$-means~\citepOur{WangM20hkmeansObjective}.
    \subsection{Preliminaries}
\label{sec:prelims}




\paragraph{Distances and similarities.} In machine learning applications, some of the most popular similarity and dissimilarity measures for feature vectors are cosine similarity $\cossim(\vx, \vy) = \frac{\langle \vx, \vy \rangle}{2\|\vx\|_2\|\vy\|_2} + \frac12$ (e.g. for deep representation learning~\citepOur{ReimersG19SBERT}) and squared Euclidean distance $\Ltwosqr(\vx, \vy) = \|\vx - \vy\|_2^2$ (e.g. used in $k$-means).
Another frequently used class of similarity functions is radial basis functions $\rbf{\gamma}(\vx, \vy) = e^{- \gamma \|\vx - \vy\|_2^2}$.
These measures are examples of asymmetric kernel functions, i.e. there exist kernel-defining functions $\kerleft$ and $\kerright$ which allow to compute these measures either exactly ($\cossim$, $\Ltwosqr$) or approximately ($\rbf{\gamma}$) as dot products $\langle \kerleft(\vx), \kerright(\vy) \rangle$ in some inner product spaces~\citepOur{RahimiR07}\arxiv{, see Appendix~\ref{sec:kernels} for details}.
Such representation allows us to use an important optimization which we call an \emph{inverse kernel trick} (see Section~\ref{sec:algorithm}).


\paragraph{Hierarchical clustering.}  Given a set of $n$ objects, the goal of \emph{hierarchical clustering} (HC) is to find a tree $\tree$ (also referred to as a \emph{dendrogram}) which contains them as leaves.
The internal nodes of $\tree$ then correspond to clusters of objects at various levels of granularity.
We hence refer to the internal nodes of $\tree$ as \emph{clusters}, while also treating them as sets of all objects in their subtrees.
For a pair of leaves $(e_1, e_2)$ let $\lca_\tree(e_1, e_2)$ be the cluster $C \in \tree$ of the smallest cardinality such that $e_1, e_2 \in C$.

\paragraph{Objectives for HC.} Measuring the quality of a HC is more challenging than evaluating the performance of basic flat clustering and label prediction tasks.
Two major obstacles which have inhibited progress on optimization algorithms for HC are:
\ifarxiv \begin{enumerate} \item \else 1) \fi
    difficulty of collecting accurate ground truth tree information (instead, triples~\citepOur{VaggosRM2018}, quadruples~\citepOur{GhoshdastidarP2019} and flat classes~\citepOur{KobrenMKM17perch} are often used) and
\ifarxiv \item \else 2) \fi
    diversity of methods using which such ground truth can be compared with the algorithm's output $\tree$: tree edit distance~\citepOur{CochezM2015,Bille2005}, flat clustering score~\citepOur{AbboudCH19avgLsh,BateniBDHKLM17,KobrenMKM17perch}.
\ifarxiv \end{enumerate} \fi
In a typical scenario when a ground truth partition $\set{C_i}_{i=1}^\clustercount$ into flat clusters is known, a popular quality measure is \emph{dendrogram purity} (DP)~\citepOur{HellerG05}, defined as maximizing:
\begin{align}
   \purity(\tree) = \frac 1 {\sum_{i=1}^\clustercount |C_i|^2} \sum_{i=1}^\clustercount \sum_{e_1,e_2 \in C_i} \frac {|C_i \cap \lca_\tree(e_1, e_2)|} {|\lca_\tree(e_1, e_2)|} 
\end{align}
However, DP says little about the quality of $\tree$ overall -- perfect DP can be achieved when each ground truth cluster corresponds to a subtree, regardless of hierarchy on top or inside of the subtrees.

Addressing the above challenges, a recent line of work by \citetOur{Dasgupta15,MoseleyW17,CohenKMM19} has proposed a family of related optimization objectives for HC. Instead of relying on ground truth information, these methods only use either distances ($\dist_{ij}$) or similarities ($\simty_{ij}$) between the data points.
\newtext{
\begin{definition}
Let $\simty\colon \items \times \items \to \R_{\ge 0}$ be a similarity function.
Then Dasgupta's objective minimizes
\begin{align}
\dasgupta(\tree) := \sum_{i<j}\simty_{ij} |\lca_\tree(e_i, e_j)| \to \min
\end{align}
A complementary Moseley-Wang's (MW) objective maximizes
\begin{align}
\mw(\tree)
:= \sum_{i<j}\simty_{ij} (n - |\lca_\tree(e_i, e_j)|) \to \max
\end{align}
\end{definition}}

\newtext{
\begin{definition}
Let $\dist\colon \items \times \items \to \R_{\ge 0}$ be a distance function.
Then Cohen-Addad et al. (CKMM) objective maximizes
\begin{align}
\ckmm(\tree)
&:= \sum_{i<j} \dist_{ij} |\lca_\tree(e_i, e_j)| \to \max
\end{align}
over \emph{binary} trees.
\end{definition}}

\newtext{Note that $\dasgupta(\tree) + \mw(\tree) = \sum_{i<j} \simty_{ij} n = const$, and hence minimizing $\dasgupta$ is equivalent to minimizing $\mw$.
$\dasgupta$ and $\ckmm$ have similar expressions; however, since one uses similarities and another one uses distance, $\dasgupta$ is minimized, while $\ckmm$ is maximized.
In the worst-case, optimizing these three objectives exactly is NP-hard~\citepOur{Dasgupta15, CohenKMM19}}

\paragraph{Approximations and normalized objectives for HC.} 
One of our goals is to measure approximations achieved by various algorithms.
Since computing the optimum is NP-hard, we use an upper bound $\upperbound \ge \opt$ on it instead.
\emph{Approximation factor} is then defined as $\aproxobjective(\tree) = \nicefrac {\objective(\tree)} {\upperbound}$.

One of the effects of considering similarity/dissimilarity graphs induced by high-dimensional vector representations is the concentration of measure.
The distributions of weights have small standard deviations and large means\footnote{E.g. for embeddings of \imagenetVTwo using ResNet34 mean cosine similarity $\cossim$ between vectors in the same class is $\approx0.88$ and between different classes is $\approx0.75$.}.
This is why the approach which returns a \arxiv{simple }random binary tree (\ifarxiv which we denote \else denoted \fi as \random) gives a good approximation for the objectives above.
To highlight the differences in quality between different algorithms, we propose normalized versions of the objectives which measures advantage over \random.
\begin{definition}
Let $\tree_R$ be a random binary tree clustering and $\objective$ be a maximization objective.
Then we define \emph{normalized} approximation factors for $\objective$ as
\begin{align}
  \aproxobjective^*(\tree) = \frac {\objective(\tree) - \mathbb E[\objective(\tree_R)]} {\upperbound - \mathbb E[\objective(\tree_R)]}  
\end{align}
\end{definition}
With a slight abuse of notation, for a triple $(i,j,k)$ let $(\lcaind{i},\lcaind{j},\lcaind{k})$ be a permutation of indices such that $\lcaind{i}$ and $\lcaind{j}$ have the smallest cardinality $\lca_\tree$ in the triple.
\citetOur[Section 4.3]{CharikarCN19} show that $\mw(\tree) = \sum_{i < j < k} \simty_{\lcaind{i}\lcaind{j}}$ and hence for the MW objective we can use a standard upper bound
$$\mw[ub]= \sum_{i < j < k} \max(\simty_{ij}, \simty_{ik}, \simty_{jk})$$

For Dasgupta's objective, \citetOur[Claim 1]{WangW2018} show that $\dasgupta(\tree) = \sum_{i < j < k}(\simty_{\lcaind{i}\lcaind{k}} + \simty_{\lcaind{j}\lcaind{k}}) + 2\sum_{i < j} \simty_{ij}$. 
Since the expressions for $\dasgupta$ and $\ckmm$ are similar, we have
$$\ckmm[ub] = \sum_{i < j < k} \max(\dist_{ij} + \dist_{ik}, \dist_{ik} + \dist_{jk}, \dist_{ij} + \dist_{jk}) + 2\sum_{i < j} \dist_{ij}$$

    \subsection{Our Contributions}

The main contributions of our paper are the following:

\textbf{Experimental study of objective-based HC.} We provide the first comprehensive experimental study of objective-based hierarchical clustering methods on massive datasets consisting of deep embedding vectors. For such vectors, we compare various HC clustering approaches, including classical hierarchical agglomerative clustering, top-down approaches and various other recent HC algorithms.

\textbf{New normalized objectives}. Due to the special structure of vector embedding data\footnote{Similarities/distances between vectors from the same classes and from different classes in such data can be very close.}, even the simplest approaches produce solutions with very high MW and CKMM scores (even a trivial random tree achieves 85-97\% approximation).
To address this issue, in Section~\ref{sec:prelims} we introduce \emph{normalized} MW/CKMM objectives which are better suited for deep vector embedding data. They capture the advantage over a trivial random tree and allow to better separate performance of different algorithms.

\textbf{New algorithm \ouralgo}. In Section~\ref{sec:algorithm}, we propose an algorithm \ouralgo inspired by~\citetOur{AhmadianCEMMLY20}. The main idea is to consider a graph whose vertices are objects and edges are similarities (dissimilarities). We perform a top-level split by partitioning this graph, so that the cut between the resulting parts is minimized (maximized). Our approach differs from~\citetOur{AhmadianCEMMLY20} in that we perform multiple levels of partitioning while also allowing imbalance at each level. We show that on deep vector embeddings this algorithm outperforms a wide range of alternatives (see Figure~\ref{fig:main_hist}, Tables~\ref{tab:distance_based} and~\ref{tab:mw_sim}).

\textbf{Scaling up \ouralgo.} One of the main advantages of \ouralgo is its efficiency. The algorithm is a gradient descent approach inspired by \citetOur{AvdiukhinPY19} applied to a quadratic function.
Using a technique which we refer to as \emph{inverse kernel trick} (see Section~\ref{sec:algorithm}), for several widely used similarities and distance measures, we can represent $A$ as a product of low-rank matrices, which allows us to compute the gradient efficiently.

\textbf{\oursatalgo and improved approximation for CKMM.} In Section~\ref{sec:bsat_algorithm}, we introduce a theoretical hierarchical clustering algorithm \oursatalgo which achieves a $0.74$-approximation for the CKMM objective in polynomial time, significantly improving existing $\approx \nicefrac 23 + 0.0004$ approximation~\citepOur{CharikarCN19}. The main idea is to reduce the problem of performing a top-level split to the \bmaxsat problem.

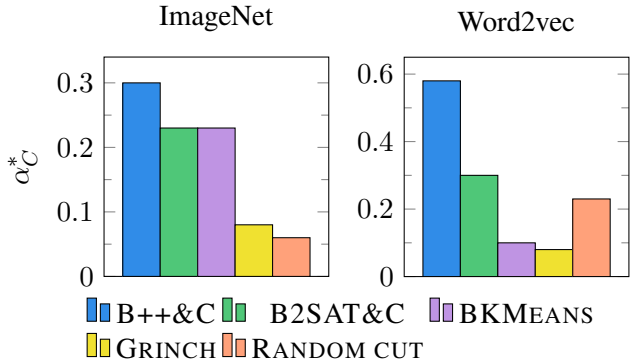
\begin{figure}[t]

\centering

\begin{tikzpicture}
 \begin{groupplot}[
     group style = {group size = 2 by 1},
     height = 4.5cm,
     width = 11cm,
    ]
    \nextgroupplot[
        title=ImageNet,
        width=130, 
        ybar interval,
        ymax=0.34, 
        ymin=0, 
        minor y tick num = 1,
        xticklabels=\empty,
        xtick=\empty,
        ylabel = {$\aproxckmm[\normchar]$},
    ]
    \addplot [fill=bleudefrance] coordinates { (0, 0.3) (5, 0) };
    \addplot [fill=emerald] coordinates { (0, 0.23)  (5, 0) };
    \addplot [fill=brightlavender] coordinates { (0, 0.23)  (5, 0) };
    \addplot [fill=dandelion] coordinates { (0, 0.08) (5, 0) };
    \addplot [fill=lightsalmon] coordinates { (0, 0.06) (5, 0) };
    \nextgroupplot[
        title=Word2vec,
        width=130, 
        ybar interval,
        ymax=0.65, 
        ymin=0, 
        minor y tick num = 1,
        legend columns=3,
        legend style={
            at={(-0.27,-0.07)},
            anchor=north,
            draw=none,
        },
        xticklabels=\empty,
        xtick=\empty,
    ]
    \addplot [fill=bleudefrance] coordinates { (0, 0.58) (5, 0) };
    \addplot [fill=emerald] coordinates { (0, 0.30)  (5, 0) };
    \addplot [fill=brightlavender] coordinates { (0, 0.10)  (5, 0) };
    \addplot [fill=dandelion] coordinates { (0, 0.08) (5, 0) };
    \addplot [fill=lightsalmon] coordinates { (0, 0.23) (5, 0) };
    \legend{\ouralgo, \oursatalgo, \hkmeans, \grinch, \randomCut} 
  \end{groupplot}
\end{tikzpicture}
\caption{Normalized distance-based CKMM objectives $\aproxckmm[\normchar]$ under squared Euclidean distance for embeddings of \imagenet using ResNet34 and word embeddings of Wikipedia using Word2vec. \ouralgo outperforms other approaches by 7\% for \imagenet and by 35\% for Word2vec. 
}
\label{fig:main_hist}
\end{figure}

\section{\ouralgofull}
\label{sec:algorithm}
Our algorithm \ouralgofull (Algorithm~\ref{alg:hiclustering}) is an extension of the \textsc{Bisect and Conquer} technique from~\citetOur{AhmadianCEMMLY20} with several crucial modifications which allow one to achieve high practical performance and solution quality. 
If the set of items $\items$ is small (less than a specified parameter $\maxavgsz$, typically $\maxavgsz \in [100,5000]$), we solve HC using average-linkage.
Otherwise, we reduce our problem to graph partitioning: we introduce a complete graph where vertices are objects and edge weights are similarities (distances) between them, and our goal is to partition the vertices into two sets of a fixed size so that the total weight of edges between the parts is minimized (maximized for distance-based objectives).

\newtext{Let $\vx \in \set{-1, 1}^n$ be a vector such that $\vx_i = 1$ if element $i$ belongs to the first part and $\vx_i=-1$ otherwise.
The graph partitioning problem can be reduced to optimizing a quadratic function $f(\vx) = \vx^\top \simmat \vx$, where $\simmat$ is the similarity matrix, under constraints $\vx_i \in \set{-1, 1}$ (each vertex belongs to some part) and $\sum_i x_i \approx 2\imbalance n$ (balance constraint).
Note that in $f(\vx)$, $\simmat_{uv}$ is taken with a positive sign if $u$ and $v$ are in the same part ($\vx_u = \vx_v$), and with a negative sign otherwise.
$\imbalance$ is an imbalance parameter controlling sizes of the parts, which are approximately $(\nicefrac 12 - \imbalance)|V|$ and $(\nicefrac 12 + \imbalance)|V|$.
If the parts should be equal, $\imbalance=0$; otherwise, we can tune $\imbalance$ to improve partition quality on imbalanced datasets.

Our algorithm is based on the approach described in~\citetOur{AvdiukhinPY19}:
we optimize a continuous relaxation ($\vx_i \in [-1, 1]$) of the function above.
\ifarxiv Algorithm~\ref{alg:gd} \else The algorithm \fi is a \emph{projected gradient descent approach} which optimizes $f(\vx) = \vx^\top \simmat \vx$ under constraints $\vx_i \in [-1, 1]$ and $\sum_i x_i = 2\imbalance n$.
In the end, $i$-th element goes to the first part with probability $\nicefrac {(\vx_i + 1)} 2$.
The key idea of our approach is the ``Inverse kernel trick'', which helps us to avoid building an explicit graph, as described below.}

    \begin{algorithm}[tb]
    	\SetKwInOut{Input}{input}
    	\SetKwInOut{Output}{output}
    	\SetKwRepeat{Do}{do}{while}
    	\nonl \textbf{parameters:} noise variance $\noise$, learning rates $\set{\rate_t}$, the number of iterations $\iternum$, kernel-defining functions $\kerleft, \kerright: \R^d \to \R^k$\\
    	\Input{Feature vectors $V = \set{v_1, \dots, v_n} \subseteq \R^d$,\\ imbalance $\imbalance \in [0, 0.5]$}
    	\Output{imbalanced partition of $\items$ into $(\items_1, \items_2)$ }
    	\caption{\gdName: $\imbalance$-imbalanced Graph $2$-Partitioning via Randomized Projected Gradient Descent}
        \label{alg:gd}
        $\body = [-1, 1]^n \cap \set{\vx \in \R^n | \sum_i x_i = 2\imbalance n}$ \\
        \For(\tcp*[f]{\texttt{Compute} $\kermatleft, \kermatright \in \R^{n \times k}$}){$i = 1$ \KwTo $n$}{\label{line:gd:init_Kernel_start}
    	    $\kermatleft_i, \kermatright_i \gets \phi(v_i), \psi(v_i)$\\ \label{line:gd:init_Kernel_end}
    	}
    	$\xt{0} \gets \argmin\limits_{\vx \in \body} \|\mathcal N_n(\mathbf 0, \noise) - \vx\|$ \\	
    	\nonl \texttt{// Projected gradient descent} \\
    	\For{$t = 0$ \KwTo $\iternum-1$}{
    		$\yt{t + 1} \gets \xt{t} - \rate_t \kermatleft \kermatright^\top \xt{t}$ \label{line:gd:grad-step}\\
    		$\xt{t + 1} \gets \argmin\limits_{\vx \in \body} \|\yt{t + 1} - \vx\|$ \label{line:gd:projection}\\
    	}
        $\items_1 \gets \items_2 \gets \emptyset$; \label{line:gd:randomized_rounding_start} \\
        \For(\tcp*[f]{\texttt{Randomized rounding}}){each $i \in \items$} {
        	\texttt{With probability $\frac {\xti ti + 1} 2$,
        	\notarxiv{\\}let $\items_1 \gets \items_1 \cup \{i\}$ \\
        	otherwise, \arxiv{let }$\items_2 \gets \items_2 \cup \{i\}$} \label{line:gd:randomized_rounding_end} \\
        }
    	\KwRet{$(\items_1, \items_2)$}
    \end{algorithm}

\begin{algorithm}[t]
	\SetKwInOut{Input}{input}
	\SetKwInOut{Output}{output}
	\SetKwRepeat{Do}{do}{while}
	\nonl \textbf{parameters:} required imbalance $\imbalance$, the maximum number of elements to run average linkage $\maxavgsz$\\
	\Input{Feature vectors $V = \set{v_1, \dots, v_n} \subseteq \R^d$}
	\Output{Clustering tree on $\items$}
	\caption{\ouralgofull: Hierarchical clustering via imbalanced graph partitioning}
    \label{alg:hiclustering}
    
    \If{$n < \maxavgsz$}{
        \KwRet{$\avglName(\items)$}
    }
    \mbox{$\items_1, \items_2 \gets \gdName(\items, \imbalance)$}
    
	\KwRet{$\set{\items_1, \items_2} \ \cup \ \ouralgo(\items_1) \ \cup \ \ouralgo(\items_2)$}
    
\end{algorithm}

\header{Inverse kernel trick.}
Note that computing the gradient $\nabla f(\vx) = 2 \simmat \vx$ na\"ively requires either $O(n^2 d)$ time or $O(n^2)$ space/time per iteration.
To scale \ouralgo, we use a technique which we call an \emph{inverse kernel trick}.
This technique is applicable for \textit{kernelizable} similarities and distance measures which can be represented as $\simty_{ij} = \langle \kerleft(v_i), \kerright(v_j) \rangle$ for functions $\kerleft, \kerright: \R^d \to \R^k$, which we call \emph{kernel-defining functions} (examples can be  found in \ifarxiv Appendix~\ref{sec:kernels}\else the full version\fi). These functions can be defined using matrices $\kermatleft, \kermatright \in \R^{n \times k}$, whose $i$-th rows are $\kermatleft_i = \kerleft(v_i)$ and $\kermatright_i = \kerright(v_i)$.
Then the gradient can be computed as $\simmat \vx = \kermatleft (\kermatright^\top \vx)$ in $O(n k)$ time.
\newtext{Some kernels (e.g. RBF) do not have finite-dimensional kernelization. In such cases we use an unbiased finite-dimensional estimation of the kernel (see Table~3 in \ifarxiv Appendix~\ref{sec:kernels}\else the full version\fi)}

We now outline \ifarxiv Algorithm~\ref{alg:gd}\else the algorithm\fi.
\ifarxiv In lines~\ref{line:gd:init_Kernel_start}-\ref{line:gd:init_Kernel_end} we \else We first \fi precompute matrices $\kermatleft$ and $\kermatright$ as described above.
\ifarxiv In line~\ref{line:gd:grad-step} \else Then \fi we \notarxiv{repeatedly} perform a gradient descent step\ifarxiv.
In line~\ref{line:gd:projection} we \else and \fi project the current point onto the feasible space.
\ifarxiv In lines~\ref{line:gd:randomized_rounding_start}-\ref{line:gd:randomized_rounding_end} \else In the end \fi we perform randomized rounding.
Time complexity of \ouralgo is $O(\iternum n k \log{n} + \maxavgsz n k)$, where $\iternum$ is the number of iterations (since when we use average linkage, we have $\nicefrac n \maxavgsz$ sets of size $\maxavgsz$, and therefore the total complexity of average linkage is $O(\maxavgsz n k)$). Space complexity
is $O(n k)$. \ouralgo is highly parallelizable since after each iteration, the tree is divided into two independent subtrees.
Additionally, each iteration is highly parallelizable since the most expensive operations are matrix-vector multiplications.



\section{\oursatalgo: Improved Approximation for the CKMM Objective}
\label{sec:bsat_algorithm}


\newtext{In this section, we introduce our main algorithm (Algorithm~\ref{alg:maxsathc}) which achieves $0.74$-approximation for the CKMM objective, significantly improving the previous $\approx \nicefrac 23 + 0.0004$ approximation.
The main idea behind our algorithm is to use \bmaxsatpar as a subroutine. In \bmaxsatpar the objective is to partition a set $V$ into two sets $S$ and $T$ of approximately the same size, such that the total weight of edges with at least one endpoint in $S$, i.e. $\sum_{u,v \in V \colon u \in S \text{ or } v \in S} \dist_{uv}$, is maximized.

This objective can be expressed as an instance of \bmaxsat:
given a $2$-SAT formula whose clauses have specified weights, the goal is to find an assignment for which exactly half of the variables are true and the total weight of satisfied clauses is maximized.
Our Algorithm~\ref{algo:max_2_sat_partitioning} constructs this instance of \bmaxsat as a collection of $n^2$ disjunctions: for each $u,v \in V$ we introduce a clause $(x_u \vee x_v)$ with weight $\dist_{uv}$ (Line~\ref{line:max2sat:add_clause}).
Given a solution for the instance let $S = \{u \colon x_u = 1\}$ and $T = \{u \colon x_u = 0\}$, we select the best of the following three solutions $\sol_1, \sol_2, \sol_3$ (see \ifarxiv Figure~\ref{fig:solutions} in Appendix \else Appendix in the full version \fi for an illustration):}
\begin{compactitem}
\item ($\sol_1$) We define $\random(S)$ and $\random(T)$ as random permutations of elements of $S$ and $T$. We then concatenate these permutations and define the solution as a path tree, where elements from $S$ are at the top and elements from $T$ are at the bottom.
\item ($\sol_2$) Union of recursive bisections of $S$ and $T$.
\item ($\sol_3$) Recursive bisection of $V$.
\end{compactitem}
The motivation for \newtext{$\textsc{Balanced Max-2-Sat Partitioning}$} is that, in the path solution, edges inside $S$ and edges between $S$ and $T$ will have a greater $\lca$ compared to edges inside $T$. Thus they make a more significant contribution to the CKMM objective.
As a result, we can show the following:

\begin{theorem}\label{thm:b2sat-approx}[Proof in \ifarxiv Appendix~\ref{sec:proof}\else the full version\fi]
Algorithm~\ref{alg:maxsathc} constructs a tree which gives a multiplicative $\gamma$-approximation of the optimum value of the CKMM objective, where \ifarxiv $\gamma  = \frac{4/3}{1 + \frac{3}{4 \cdot 0.94}} \ge 0.74.$ \else $\gamma \ge 0.74.$ \fi \end{theorem}

\begin{algorithm}[t]
\caption{\textsc{Balanced Max-2-Sat Partitioning}}
\label{algo:max_2_sat_partitioning}
	\SetKwInOut{Input}{input}
	\SetKwInOut{Output}{output}
    \Input{Distance function $\dist \colon \items \times \items \to \mathbb R_{\ge 0}$.}
	\Output{Partitioning of $S$}
    Set $C \gets \emptyset$\\
    \For{$(u,v) \in \items \times \items$}  {
        Add clause $(x_u \lor x_v)$ to $C$ with weight $\dist_{uv}$\label{line:max2sat:add_clause}
    }
    \Return \textsc{Balanced Max-2-Sat}(C)
\end{algorithm}
\begin{algorithm}[t]
	\SetKwInOut{Input}{input}
	\SetKwInOut{Output}{output}
	\SetKwRepeat{Do}{do}{while}
	\Input{Feature vectors $V = \set{v_1, \dots, v_n} \subseteq \R^d$,\\
	distance function $\dist \colon \items \times \items \to \mathbb R_{\ge 0}$}
	\Output{Clustering tree on $\items$}
    \caption{Hierarchical Clustering via \textsc{Max-2-Sat} (\oursatalgo)}
    \label{alg:maxsathc}
    Let $S \subseteq \items$ be the set of vertices corresponding to positive variables in the assignment given by $\bmaxsatpar(\dist)$\;
    $\begin{aligned}\sol_1 \gets \tpath(\textsc{Concat}(
        &\random(S),
        \notarxiv{}\random(\items \setminus S)))
        \end{aligned}$\\
    $\begin{aligned}\sol_2 \gets (
        &\balancedBisection(S,\dist),
        \notarxiv{\\&}\balancedBisection(\items \setminus S, \dist))
        \end{aligned}$\\
    $\sol_3 \gets \balancedBisection(\items,\dist)$\\
	\KwRet{Best of $\sol_1, \sol_2, \sol_3$}
    
\end{algorithm}

\section{Experiments}
\label{sec:experiments}

We give a comprehensive experimental evaluation of a wide range of hierarchical clustering algorithms, including:
\ifarxiv
\begin{itemize}
\else
\begin{compactitem}
\fi
\item Grafting and rotation-based incremental hierarchical clustering (\textbf{\grinch})~\citepOur{MonathKGM19grinch},
\item Local rotation-based incremental hierarchical clustering (\textbf{\perch})~\citepOur{KobrenMKM17perch},
\item Bisecting k-means (\textbf{\hkmeans}),
\item Robust single-linkage implementation from the HDBSCAN library~\citepOur{McinnesH2017hdbscan,McInnesHA17hdbscan} (\textbf{\hdbscan})~\citepOur{ChaudhuriD10robustSingleL,ChaudhuriDKV14robustSingleL},
\item Minimum spanning tree based affinity clustering (\textbf{\affinityClustering})~\citepOur{BateniBDHKLM17},
\item Ward's method (\textbf{\wards})~\citepOur{WardMethod}
\item Average linkage (\textbf{\averageLinkage})~\citepOur{SokalMAvgLinkage},
\item Single linkage (\textbf{\singleLinkage})~\citepOur{GowerRSingleLinkage}
\item Complete linkage (\textbf{\completeLinkage})~\citepOur{SorensenSSSSSB1948CompleteLinkage},
\item 1D projection based \textbf{\randomCut}~\citepOur{CharikarCNY2019}; we project all vectors onto random line and partition them based on the order they appear on the line,
\item Our own new algorithms \textbf{\oursatalgo}\footnote{Implementation of our theoretical algorithm using kernelized gradient descent based balanced \twosat solver.} and \textbf{\ouralgo}.
\item A special case of \ouralgo with no imbalance ($\imbalance=0$), which we denote as \textbf{\bbisection}. This algorithm is similar to a balanced bisection algorithm from \citetOur{AhmadianCEMMLY20}, with the main difference being is that we perform bisection on multiple levels.
\ifarxiv
    \end{itemize}
\else
    \end{compactitem}
\fi
\ifarxiv
    In Appendix~\ref{sec:additonal_results},
\else
    In the full version,
\fi
we additionally perform comparison with the following HC algorithms which produce non-binary trees:
\ifarxiv
    \begin{itemize}
\else
    \begin{compactitem}
\fi
\item Gradient-based optimization of representations of trees in hyperbolic space (\textbf{\ghhc})~\citepOur{MonathZSMA19ghhc},
\item Top-down incremental hierarchical clustering (\textbf{\birch})~\citepOur{ZhangRL96birch},
\ifarxiv
    \end{itemize}
\else
    \end{compactitem}
\fi

\subsection{Datasets}
\label{subsec:datasets}
We use a large number of vector embeddings including basic supervised and unsupervised constructions, embeddings with pre-trained networks on fresh data from the same distribution, embeddings with pre-trained networks on a different distribution and metric/similarity learning with triplet loss. 
We focus on large high-dimensional datasets ($n$ up to $4.5 \cdot 10^6$, $d$ up to $2048$) consisting of embedding vectors arising from applications to computer vision (CV) and natural language processing (NLP). 
We study five different types of datasets: three types of image data (``\classicSetup'', ``\generalizationSetup'', ``\shiftSetup''), and two types of text data (``\wordSetup'', ``\sentenceSetup''). 
In order to facilitate comparison with the previous work, we also provide results on smaller datasets for other machine learning applications
(\glass\footnote{\label{foot:data}\url{http://archive.ics.uci.edu/ml}}, \spambase\footnotemark[\getrefnumber{foot:data}], \covType\footnotemark[\getrefnumber{foot:data}],
\aloi~\citepOur{GeusebroekBS05aloiDataset}) and on larger datasets (ImageNet Inception~\citepOur{MonathZSMA19ghhc}) in
\ifarxiv Appendix~\ref{sec:additonal_results} (Table~\ref{tab:additional_ckmm_other_datasets}, Table~\ref{tab:additional_mw_other_datasets}, Table~\ref{tab:additional_dp_other_datasets}). \else the full version. \fi


\header{CV: supervised embeddings (``\classicSetup'').} This is the most vanilla setting, in which embedding vectors are constructed via supervised learning. Vectors are taken from the last hidden layer of a pre-trained neural net. We use image embeddings of \imagenet ILSVRC 2012~\citepOur{DengDSLLF09imagenet} via ResNet34~\citepOur{KaimingXSJ15Resnet}.

\header{CV: generalization (``\generalizationSetup'').}
The ``\generalizationSetup'' setting is similar to ``\classicSetup'' except that we perform evaluation on a fresh set of samples from a \textit{similar distribution}. We use ResNet34 pre-trained on \imagenet ILSVRC 2012 to compute embedding vectors of images from \imagenetVTwo~\citepOur{RechtRSS19imagenetv2}. 
As shown in~\citetOur{RechtRSS19imagenetv2}, Top-1 accuracy drops by 12.1\% on this dataset. 

\header{CV: distribution shift (``\shiftSetup'').} 
The ``\shiftSetup'' setting is similar to ``\classicSetup'' and ``\generalizationSetup'' except that we perform evaluation on a \textit{different distribution}. 
We use ResNet34 pre-trained on \imagenet ILSVRC 2012 to compute embedding vectors of \nabirds~\citepOur{VanBFHBIPB15nabirds}. 
These two datasets have very little intersection on classes: \imagenet ILSVRC 2012 is very general and contains 1000 classes, of which only 55 are birds, while \nabirds is more domain-specific and contains 555 classes of birds. 

\header{NLP: word embeddings (``\wordSetup'').} In the ``\wordSetup'' setting we use unsupervised word embedding vectors trained on Twitter~\footnote{\url{https://nlp.stanford.edu/projects/glove}} and Wikipedia~\citepOur{YamadaASSTTM20wiki} using two classic methods Glove~\citepOur{PenningtonSM14glove} and Word2vec~\citepOur{YamadaSTT16,MikolovSCD13word2vec}. 
We emphasize that this setting corresponds to a \textit{fully unsupervised pipeline} since both datasets and HC algorithms we use are unsupervised.

\header{NLP: sentence embeddings (``\sentenceSetup'').} In the ``\sentenceSetup'' setting we use a pre-trained  Sentence-BERT~\citepOur{ReimersG19SBERT} to construct embeddings from the sentiment analysis dataset of movie reviews \sstTwo~\citepOur{SocherPWCMNP13sst2}. We use a RoBERTa-based~\citepOur{Liu19OGDJCLLZS} model roberta-base-nli-stsb-mean-tokens~\footnote{\url{https://github.com/UKPLab/sentence-transformers}} which has been trained to produce meaningful sentence representations. 
Comparing to ``\shiftSetup'' this setting shows more advanced techniques of similarity and representation learning including siamese architecture and triplet loss. Cosine similarity $\cossim$ between sentence representations corresponds to semantic similarity.

\subsection{Results}

\begin{table*}[p]
\centering

\setlength{\tabcolsep}{0.05em}
\begin{tabular}{cPPPPPP}
\hline

\MtabSetup & \MimagenetSmall & \MimagenetVTwo & \Mnabirds &  \Mtwitter & \Mwikipedia & \MsstTwo \\
\hline

\textbf{\ouralgo} & \boldmath{\cellcolor[rgb]{0.73, 0.89, 0.70}$.30 / .95$} & \boldmath{\cellcolor[rgb]{0.17, 0.58, 0.30}$.83 / .99$} & \boldmath{\cellcolor[rgb]{0.23, 0.64, 0.34}$.71 / .97$} & \boldmath{\cellcolor[rgb]{0.29, 0.69, 0.38}$.60 / .97$} & \boldmath{\cellcolor[rgb]{0.33, 0.71, 0.40}$.58 / .94$} & \boldmath{\cellcolor[rgb]{0.46, 0.77, 0.47}$.49 / .97$}\\
\averageLinkage & -- & \cellcolor[rgb]{0.31, 0.70, 0.39}$.59 / .97$ & -- & -- & -- & --\\
\textbf{\oursatalgo} & \cellcolor[rgb]{0.80, 0.92, 0.77}$.23 / .95$ & \cellcolor[rgb]{0.77, 0.91, 0.74}$.26 / .95$ & \cellcolor[rgb]{0.34, 0.71, 0.41}$.57 / .96$ & \cellcolor[rgb]{0.90, 0.96, 0.89}$.12 / .93$ & \cellcolor[rgb]{0.72, 0.89, 0.69}$.30 / .90$ & \cellcolor[rgb]{0.52, 0.80, 0.52}$.46 / .96$\\
\hdbscan & -- & \cellcolor[rgb]{0.23, 0.64, 0.35}$.70 / .98$ & \cellcolor[rgb]{0.53, 0.80, 0.52}$.45 / .95$ & -- & -- & \cellcolor[rgb]{0.88, 0.95, 0.85}$.15 / .94$\\
\textbf{\bbisection} & \cellcolor[rgb]{0.80, 0.92, 0.77}$.23 / .95$ & \cellcolor[rgb]{0.77, 0.91, 0.74}$.26 / .95$ & \cellcolor[rgb]{0.34, 0.71, 0.41}$.57 / .96$ & \cellcolor[rgb]{0.91, 0.97, 0.90}$.10 / .93$ & \cellcolor[rgb]{0.86, 0.94, 0.83}$.17 / .88$ & \cellcolor[rgb]{0.52, 0.80, 0.52}$.46 / .96$\\
\completeLinkage & -- & \cellcolor[rgb]{0.47, 0.78, 0.48}$.48 / .97$ & -- & -- & -- & --\\
\hkmeans & \cellcolor[rgb]{0.80, 0.92, 0.77}$.23 / .95$ & \cellcolor[rgb]{0.77, 0.91, 0.75}$.26 / .95$ & \cellcolor[rgb]{0.29, 0.69, 0.38}$.61 / .96$ & \cellcolor[rgb]{0.91, 0.97, 0.89}$.10 / .93$ & \cellcolor[rgb]{0.91, 0.97, 0.90}$.10 / .87$ & \cellcolor[rgb]{0.52, 0.80, 0.52}$.45 / .96$\\
\grinch & \cellcolor[rgb]{0.92, 0.97, 0.91}$.08 / .94$ & \cellcolor[rgb]{0.93, 0.97, 0.92}$.06 / .94$ & \cellcolor[rgb]{0.50, 0.79, 0.50}$.47 / .95$ & \cellcolor[rgb]{0.96, 0.99, 0.96}$.01 / .92$ & \cellcolor[rgb]{0.93, 0.97, 0.91}$.08 / .86$ & \cellcolor[rgb]{0.87, 0.95, 0.85}$.16 / .94$\\
\perch & \cellcolor[rgb]{0.94, 0.98, 0.93}$.05 / .94$ & \cellcolor[rgb]{0.92, 0.97, 0.91}$.08 / .94$ & \cellcolor[rgb]{0.65, 0.86, 0.63}$.36 / .94$ & \cellcolor[rgb]{0.97, 0.99, 0.96}$.00 / .92$ & -- & \cellcolor[rgb]{0.88, 0.95, 0.86}$.15 / .94$\\
\randomCut & \cellcolor[rgb]{0.94, 0.98, 0.92}$.06 / .94$ & \cellcolor[rgb]{0.92, 0.97, 0.91}$.09 / .94$ & \cellcolor[rgb]{0.90, 0.96, 0.88}$.12 / .92$ & \cellcolor[rgb]{0.91, 0.96, 0.89}$.11 / .93$ & \cellcolor[rgb]{0.80, 0.92, 0.78}$.23 / .89$ & \cellcolor[rgb]{0.93, 0.97, 0.92}$.07 / .94$\\
\Random & \cellcolor[rgb]{0.97, 0.99, 0.96}$.00 / .94$ & \cellcolor[rgb]{0.97, 0.99, 0.96}$.00 / .94$ & \cellcolor[rgb]{0.97, 0.99, 0.96}$.00 / .91$ & \cellcolor[rgb]{0.97, 0.99, 0.96}$.00 / .92$ & \cellcolor[rgb]{0.97, 0.99, 0.96}$.00 / .85$ & \cellcolor[rgb]{0.97, 0.99, 0.96}$.00 / .93$\\
\hline
$n\approx$ & $1.2\cdot10^6$ & $10^4$ & $5 \cdot 10^4$ & $1.3\cdot10^6$ & $4.5\cdot10^6$ & $7 \cdot 10^4$\\
$d$ & $512$ & $512$ & $512$ & $200$ & $100$ & $768$\\
\#classes & $10^3$ & $10^3$ & $555$ & -- & -- & $2$\\
\hline
\end{tabular}

\caption{Normalized/unnormalized ($\aproxckmm[\normchar]$/$\aproxckmm$) distance-based CKMM objectives under squared Euclidean distance. On $\aproxckmm$ all algorithms (including \random) give at least 85-94$\%$ approximation. On $\aproxckmm[\normchar]$ \ouralgo outperforms other approaches on all datasets by \rangeImprovementCkmm. Among other scalable algorithms, \hkmeans shows good average performance. Among non-scalable algorithms, basic HAC methods (\averageLinkage, \completeLinkage) and robust single linkage (\hdbscan) show competitive performance. Our worst-case theoretical algorithm \oursatalgo also shows substantial gains. Even a simple 1D random projection technique (\randomCut) gives non-trivial results on NLP datasets. Algorithms that performed worse than \randomCut (\singleLinkage, \affinityClustering, \wards) are not shown. See \ifarxiv Appendix~\ref{sec:additonal_results} \else the full version \fi for complete results which include comparison with \ghhc and \birch.}
\label{tab:distance_based}
\end{table*}

\begin{table*}[p]
\centering


\setlength{\tabcolsep}{0.05em}
\begin{tabular}{cPPPPPP}
\hline

\MtabSetup & \MimagenetSmall & \MimagenetVTwo & \Mnabirds &  \Mtwitter & \Mwikipedia & \MsstTwo \\
\hline

\textbf{\ouralgo} & \boldmath{\cellcolor[rgb]{0.59, 0.83, 0.57}$.40 / .98$} & \boldmath{\cellcolor[rgb]{0.44, 0.76, 0.46}$.51 / .99$} & \boldmath{\cellcolor[rgb]{0.22, 0.63, 0.34}$.73 / .99$} & \boldmath{\cellcolor[rgb]{0.54, 0.81, 0.53}$.44 / .97$} & \boldmath{\cellcolor[rgb]{0.60, 0.84, 0.58}$.40 / .96$} & \boldmath{\cellcolor[rgb]{0.45, 0.76, 0.46}$.51 / .96$}\\
\completeLinkage & -- & \boldmath{\cellcolor[rgb]{0.43, 0.76, 0.45}$.51 / .99$} & -- & -- & -- & --\\
\hkmeans & \cellcolor[rgb]{0.63, 0.85, 0.61}$.37 / .98$ & \cellcolor[rgb]{0.61, 0.84, 0.59}$.39 / .98$ & \cellcolor[rgb]{0.23, 0.64, 0.34}$.71 / .99$ & \cellcolor[rgb]{0.76, 0.90, 0.73}$.27 / .96$ & \cellcolor[rgb]{0.62, 0.84, 0.60}$.38 / .96$ & \cellcolor[rgb]{0.53, 0.80, 0.52}$.45 / .95$\\
\textbf{\bbisection} & \cellcolor[rgb]{0.63, 0.85, 0.61}$.37 / .98$ & \cellcolor[rgb]{0.61, 0.84, 0.59}$.39 / .98$ & \cellcolor[rgb]{0.25, 0.66, 0.36}$.67 / .99$ & \cellcolor[rgb]{0.80, 0.92, 0.77}$.23 / .95$ & \boldmath{\cellcolor[rgb]{0.60, 0.84, 0.58}$.40 / .96$} & \cellcolor[rgb]{0.52, 0.80, 0.52}$.46 / .95$\\
\textbf{\oursatalgo} & \cellcolor[rgb]{0.63, 0.85, 0.61}$.37 / .98$ & \cellcolor[rgb]{0.61, 0.84, 0.59}$.39 / .98$ & \cellcolor[rgb]{0.25, 0.66, 0.36}$.67 / .99$ & \cellcolor[rgb]{0.80, 0.92, 0.77}$.23 / .95$ & \boldmath{\cellcolor[rgb]{0.60, 0.84, 0.58}$.40 / .96$} & \cellcolor[rgb]{0.52, 0.80, 0.52}$.46 / .95$\\
\averageLinkage & -- & \cellcolor[rgb]{0.63, 0.85, 0.61}$.38 / .98$ & -- & -- & -- & --\\
\wards & -- & \cellcolor[rgb]{0.66, 0.86, 0.64}$.35 / .98$ & -- & -- & -- & --\\
\grinch & \cellcolor[rgb]{0.91, 0.97, 0.90}$.10 / .98$ & \cellcolor[rgb]{0.90, 0.96, 0.88}$.12 / .98$ & \cellcolor[rgb]{0.38, 0.73, 0.43}$.54 / .98$ & \cellcolor[rgb]{0.92, 0.97, 0.90}$.09 / .95$ & \cellcolor[rgb]{0.92, 0.97, 0.90}$.09 / .94$ & \cellcolor[rgb]{0.86, 0.95, 0.84}$.17 / .93$\\
\hdbscan & -- & \cellcolor[rgb]{0.85, 0.94, 0.83}$.17 / .98$ & \cellcolor[rgb]{0.86, 0.95, 0.84}$.16 / .97$ & -- & -- & \cellcolor[rgb]{0.83, 0.93, 0.80}$.20 / .93$\\
\perch & \cellcolor[rgb]{0.93, 0.97, 0.91}$.07 / .97$ & \cellcolor[rgb]{0.90, 0.96, 0.88}$.13 / .98$ & \cellcolor[rgb]{0.56, 0.82, 0.55}$.43 / .98$ & \cellcolor[rgb]{0.97, 0.99, 0.96}$.01 / .94$ & -- & \cellcolor[rgb]{0.88, 0.95, 0.86}$.15 / .92$\\
\randomCut & \cellcolor[rgb]{0.94, 0.98, 0.92}$.06 / .97$ & \cellcolor[rgb]{0.94, 0.98, 0.92}$.06 / .97$ & \cellcolor[rgb]{0.90, 0.96, 0.88}$.13 / .97$ & \cellcolor[rgb]{0.94, 0.98, 0.92}$.06 / .94$ & \cellcolor[rgb]{0.89, 0.96, 0.87}$.13 / .95$ & \cellcolor[rgb]{0.93, 0.97, 0.91}$.07 / .92$\\
\Random & \cellcolor[rgb]{0.97, 0.99, 0.96}$.00 / .97$ & \cellcolor[rgb]{0.97, 0.99, 0.96}$.00 / .97$ & \cellcolor[rgb]{0.97, 0.99, 0.96}$.00 / .96$ & \cellcolor[rgb]{0.97, 0.99, 0.96}$.00 / .94$ & \cellcolor[rgb]{0.97, 0.99, 0.96}$.00 / .94$ & \cellcolor[rgb]{0.97, 0.99, 0.96}$.00 / .91$\\
\hline
\end{tabular}
\caption{Normalized/unnormalized ($\aproxmw[\normchar]$/$\aproxmw$) similarity-based MW objectives under cosine similarity. On $\aproxmw$ all algorithms give at least 94-97$\%$ approximation. For $\aproxmw[\normchar]$, \ouralgo outperforms other approaches on all medium-large datasets by \rangeImprovementMw. Among other scalable algorithms \hkmeans shows good average performance. Among non-scalable algorithms HAC methods (\completeLinkage, \averageLinkage, \wards) show competitive performance, while performance of \hdbscan drops compared to the distance-based CKMM objective. Our theoretical algorithm \oursatalgo shows the same performance as \bbisection. Algorithms which performed worse than \randomCut (\singleLinkage, \affinityClustering) are not shown. See \ifarxiv Appendix~\ref{sec:additonal_results} \else the full version \fi for complete results.}
\label{tab:mw_sim}
\end{table*}

We report our key experimental results in Table~\ref{tab:distance_based} and Table~\ref{tab:mw_sim}.
Experiments were performed on $8$ CPUs $2.0$GHz Intel Xeon Scalable Processor (Skylake), $90$Gb RAM.
Missing entries are due to timeouts (5 hours) or memory limits. The scalability of different methods is discussed in \ifarxiv Appendix~\ref{sec:scalability}. \else the full version. \fi
In order to highlight the quality of the resulting HC, the algorithms are sorted by their average rank.
Table~\ref{tab:distance_based} and Table~\ref{tab:mw_sim} contains only mean value of 5 repetitions,
    full results and standard deviations are provided in \ifarxiv Appendix~\ref{sec:additonal_results}. \else the full version. \fi
MW and CKMM objectives are not suitable for non-binary trees. 
In order to compare with algorithms which produce non-binary trees (\ghhc, \birch), in \ifarxiv Appendix~\ref{sec:additonal_results} \else the full version \fi
we introduce appropriate extensions of these objectives to non-binary trees. See \ifarxiv Appendix~\ref{sec:additonal_results} \else the full version \fi for a complete set of experimental results. 

\header{Na\"ive random baseline.}
For the MW objective, a random binary tree (\random) achieves a $\nicefrac13$-approximation in expectation, while for CKMM objective it achieves a $\nicefrac23$-approximation. 
Worst-case analysis predicts that beating these baselines is challenging: current best poly-time approximations are 0.42 for MW (\bbisection ~\citepOur{AhmadianCEMMLY20}) and 0.74 for CKMM (\oursatalgo, our work, Theorem~\ref{thm:b2sat-approx}).
Furthermore, for many practical algorithms worst-case analysis predicts that they either can't go above the na\"ive baselines (\averageLinkage~\citepOur{CharikarCN19}) or even fail to meet it (\hkmeans~\citepOur{MoseleyW17}). Such worst-case instances are folklore for $\completeLinkage$, $\affinityClustering$ and $\singleLinkage$.

\header{Approximation and normalization of objectives.} On deep-learned vector embeddings, almost all algorithms dramatically outperform the na\"ive approximation baselines. 
Due to concentration of measure (as discussed in the end of Section~\ref{sec:prelims}) even \random gets at least 85-94\% / 91-97\% of the optimum,
which noticeably outperform the worst case $\nicefrac13$ and $\nicefrac23$ approximations.
Furthermore, classic HAC algorithms (\averageLinkage, \completeLinkage, 
\wards) also work much better than predicted by the worst-case analysis.
Our results in Table~\ref{tab:mw_sim} show that approximations achieved by various algorithms are very close, with many of them being within $1\%$ difference.
Therefore, to highlight performance variations between the algorithms, we measure advantage over \random and focus on normalized objectives $\aproxmw[*]/\aproxckmm[*]$.

\header{Performance on different types of data.}
Our experiments show that the quality of HC produced by different algorithms can vary substantially across different types of data. 
For example, optimizing CKMM on unsupervised word embedding vectors (\twitter, \wikipedia) turns out to be rather challenging (see Table~\ref{tab:distance_based}).
Performance of most approaches drops drastically on these vectors, sometimes even below a simple \randomCut.
Nevertheless, despite this variablity, \ouralgo shows consistently best performance across the different types of vector embeddings.

\header{Performance of various types of algorithms.}
While HAC approaches (\averageLinkage, \completeLinkage, \wards) often show good performance, their running time scales superlinearly making them prohibitively slow on large datasets.
Among scalable algorithms, our results show the advantage of top-down methods (\ouralgo, \oursatalgo, \hkmeans) over
nearest-neighbor based approaches (\singleLinkage, \affinityClustering, \perch, and \grinch).
This is due to the fact that MW/CKMM objectives encourage solutions with good global structure which top-down methods tend to recover better than approaches focused on exploiting local structure.


\header{Performance of \ouralgo.} Our proposed combination of top-down unbalanced bisection (with kernelization and gradient descent optimization) and average-linkage \ouralgo appears to robustly give the highest quality performance across a diverse collection of datasets.
On normalized MW/CKMM objectives ($\aproxmw[\normchar]/\aproxckmm[\normchar]$), \ouralgo outperforms other approaches on all datasets by \rangeImprovementMw/\rangeImprovementCkmm respectively.

\header{Generalization properties of HC objectives.} Since we use pretrained neural nets to generate vector embeddings, we can compute HC on fresh samples from a similar distribution very quickly\footnote{At the cost of a single forward pass + running a HC algorithm.}. Applying this approach to \imagenetVTwo (a fresh sample from the \imagenet distribution) our results show that most algorithms show similar MW/CKMM scores compared to \imagenet. This is rather different from the consistent $\approx 10\%$ prediction accuracy drop reported in~\citetOur{RechtRSS19imagenetv2}. We believe that explanation of this phenomenon might be a promising direction for future work.

\section{Conclusion and Future Work}

In this paper, we initiate a comprehensive experimental study of HC algorithms on deep embedding vectors. Our results indicate that CKMM is particularly well-suited for such data as it captures the quality of the overall hierarchy and only relies on distances between vectors.
We introduce a new scalable algorithm \ouralgo which outperforms existing approaches on all considered datasets.
Moreover, we present a polynomial-time algorithm \oursatalgo that significantly improves the existing approximation for the CKMM objective to $0.74$.

A possible future direction is to show approximation guarantees for imbalanced bisection.
It might also be interesting to understand why there is no drop in HC quality (contrary to the drop in the classification accuracy shown in~\citetOur{RechtRSS19imagenetv2}) when generalizing \imagenet embeddings for \imagenetVTwo dataset.






\small
\arxiv{\bibliographystyle{plainnat}}
\bibliography{references}
\normalsize

\ifarxiv
    \appendix
    \etocdepthtag.toc{mtappendix}
\etocsettagdepth{mtchapter}{none}
\etocsettagdepth{mtappendix}{subsection}
\clearpage

\begin{figure*}[htb!]
    \centering
    \subcaptionbox{$\sol_1$ -- solution obtained by first randomly permuting $S$ and $\items \setminus S$, and then creating a path tree with $S$ on top of $\items \setminus S$\label{fig:sols_path}}[0.31\textwidth]{
    \centering
    \tikzset{every picture/.style={line width=0.75pt}} 

\begin{tikzpicture}[x=0.75pt,y=0.75pt,yscale=-0.7,xscale=0.7]

\draw   (276.59,73.1) .. controls (280.15,69.55) and (285.91,69.56) .. (289.46,73.12) .. controls (293.01,76.68) and (293.01,82.44) .. (289.45,85.99) .. controls (285.89,89.54) and (280.13,89.53) .. (276.58,85.97) .. controls (273.03,82.41) and (273.04,76.65) .. (276.59,73.1) -- cycle ;
\draw    (276.58,85.97) -- (263.27,99.25) ;
\draw [shift={(261.85,100.66)}, rotate = 315.07] [color={rgb, 255:red, 0; green, 0; blue, 0 }  ][line width=0.75]    (10.93,-3.29) .. controls (6.95,-1.4) and (3.31,-0.3) .. (0,0) .. controls (3.31,0.3) and (6.95,1.4) .. (10.93,3.29)   ;
\draw   (221.37,128.19) .. controls (224.93,124.64) and (230.69,124.65) .. (234.24,128.2) .. controls (237.79,131.76) and (237.78,137.52) .. (234.23,141.07) .. controls (230.67,144.62) and (224.91,144.62) .. (221.36,141.06) .. controls (217.81,137.5) and (217.81,131.74) .. (221.37,128.19) -- cycle ;
\draw    (248.97,113.51) -- (235.66,126.79) ;
\draw [shift={(234.24,128.2)}, rotate = 315.07] [color={rgb, 255:red, 0; green, 0; blue, 0 }  ][line width=0.75]    (10.93,-3.29) .. controls (6.95,-1.4) and (3.31,-0.3) .. (0,0) .. controls (3.31,0.3) and (6.95,1.4) .. (10.93,3.29)   ;
\draw   (193.76,155.73) .. controls (197.32,152.18) and (203.08,152.19) .. (206.63,155.75) .. controls (210.18,159.31) and (210.17,165.07) .. (206.62,168.62) .. controls (203.06,172.17) and (197.3,172.16) .. (193.75,168.6) .. controls (190.2,165.04) and (190.2,159.28) .. (193.76,155.73) -- cycle ;
\draw    (221.36,141.06) -- (208.05,154.34) ;
\draw [shift={(206.63,155.75)}, rotate = 315.07] [color={rgb, 255:red, 0; green, 0; blue, 0 }  ][line width=0.75]    (10.93,-3.29) .. controls (6.95,-1.4) and (3.31,-0.3) .. (0,0) .. controls (3.31,0.3) and (6.95,1.4) .. (10.93,3.29)   ;
\draw    (193.75,168.6) -- (180.44,181.88) ;
\draw [shift={(179.02,183.29)}, rotate = 315.07] [color={rgb, 255:red, 0; green, 0; blue, 0 }  ][line width=0.75]    (10.93,-3.29) .. controls (6.95,-1.4) and (3.31,-0.3) .. (0,0) .. controls (3.31,0.3) and (6.95,1.4) .. (10.93,3.29)   ;
\draw   (137.83,211.53) .. controls (141.39,207.98) and (147.15,207.98) .. (150.7,211.54) .. controls (154.25,215.1) and (154.24,220.86) .. (150.69,224.41) .. controls (147.13,227.96) and (141.37,227.95) .. (137.82,224.39) .. controls (134.27,220.84) and (134.27,215.07) .. (137.83,211.53) -- cycle ;
\draw    (165.43,196.85) -- (152.12,210.13) ;
\draw [shift={(150.7,211.54)}, rotate = 315.07] [color={rgb, 255:red, 0; green, 0; blue, 0 }  ][line width=0.75]    (10.93,-3.29) .. controls (6.95,-1.4) and (3.31,-0.3) .. (0,0) .. controls (3.31,0.3) and (6.95,1.4) .. (10.93,3.29)   ;
\draw  [dash pattern={on 0.84pt off 2.51pt}] (300.06,85.15) .. controls (302.56,82.65) and (306.61,82.66) .. (309.11,85.16) -- (322.67,98.76) .. controls (325.16,101.26) and (325.16,105.31) .. (322.66,107.81) -- (254.55,175.75) .. controls (252.05,178.24) and (247.99,178.24) .. (245.5,175.74) -- (231.94,162.14) .. controls (229.44,159.64) and (229.45,155.59) .. (231.95,153.09) -- cycle ;
\draw  [dash pattern={on 0.84pt off 2.51pt}] (218.57,166.44) .. controls (221.07,163.94) and (225.12,163.95) .. (227.62,166.45) -- (241.18,180.04) .. controls (243.68,182.55) and (243.67,186.6) .. (241.17,189.09) -- (170.37,259.72) .. controls (167.87,262.22) and (163.82,262.21) .. (161.32,259.71) -- (147.76,246.11) .. controls (145.26,243.61) and (145.27,239.56) .. (147.77,237.06) -- cycle ;
\draw   (298.49,96.46) .. controls (302.04,92.91) and (307.81,92.92) .. (311.36,96.48) .. controls (314.9,100.04) and (314.9,105.8) .. (311.34,109.35) .. controls (307.78,112.9) and (302.02,112.89) .. (298.47,109.33) .. controls (294.92,105.77) and (294.93,100.01) .. (298.49,96.46) -- cycle ;
\draw   (243.26,151.55) .. controls (246.82,148) and (252.58,148.01) .. (256.13,151.57) .. controls (259.68,155.12) and (259.68,160.89) .. (256.12,164.43) .. controls (252.56,167.98) and (246.8,167.98) .. (243.25,164.42) .. controls (239.7,160.86) and (239.71,155.1) .. (243.26,151.55) -- cycle ;
\draw   (215.65,179.09) .. controls (219.21,175.54) and (224.97,175.55) .. (228.52,179.11) .. controls (232.07,182.67) and (232.07,188.43) .. (228.51,191.98) .. controls (224.95,195.53) and (219.19,195.52) .. (215.64,191.96) .. controls (212.09,188.4) and (212.1,182.64) .. (215.65,179.09) -- cycle ;
\draw   (159.29,234.81) .. controls (162.84,231.26) and (168.61,231.26) .. (172.16,234.82) .. controls (175.7,238.38) and (175.7,244.14) .. (172.14,247.69) .. controls (168.58,251.24) and (162.82,251.23) .. (159.27,247.68) .. controls (155.72,244.12) and (155.73,238.36) .. (159.29,234.81) -- cycle ;
\draw    (206.62,168.62) -- (214.35,177.58) ;
\draw [shift={(215.65,179.09)}, rotate = 229.21] [color={rgb, 255:red, 0; green, 0; blue, 0 }  ][line width=0.75]    (10.93,-3.29) .. controls (6.95,-1.4) and (3.31,-0.3) .. (0,0) .. controls (3.31,0.3) and (6.95,1.4) .. (10.93,3.29)   ;
\draw    (179,196.16) -- (186.74,205.12) ;
\draw [shift={(188.04,206.64)}, rotate = 229.21] [color={rgb, 255:red, 0; green, 0; blue, 0 }  ][line width=0.75]    (10.93,-3.29) .. controls (6.95,-1.4) and (3.31,-0.3) .. (0,0) .. controls (3.31,0.3) and (6.95,1.4) .. (10.93,3.29)   ;
\draw    (150.69,224.41) -- (158.42,233.37) ;
\draw [shift={(159.72,234.89)}, rotate = 229.21] [color={rgb, 255:red, 0; green, 0; blue, 0 }  ][line width=0.75]    (10.93,-3.29) .. controls (6.95,-1.4) and (3.31,-0.3) .. (0,0) .. controls (3.31,0.3) and (6.95,1.4) .. (10.93,3.29)   ;
\draw    (234.23,141.07) -- (241.96,150.04) ;
\draw [shift={(243.26,151.55)}, rotate = 229.21] [color={rgb, 255:red, 0; green, 0; blue, 0 }  ][line width=0.75]    (10.93,-3.29) .. controls (6.95,-1.4) and (3.31,-0.3) .. (0,0) .. controls (3.31,0.3) and (6.95,1.4) .. (10.93,3.29)   ;
\draw    (261.84,113.53) -- (269.57,122.49) ;
\draw [shift={(270.88,124.01)}, rotate = 229.21] [color={rgb, 255:red, 0; green, 0; blue, 0 }  ][line width=0.75]    (10.93,-3.29) .. controls (6.95,-1.4) and (3.31,-0.3) .. (0,0) .. controls (3.31,0.3) and (6.95,1.4) .. (10.93,3.29)   ;
\draw    (289.45,85.99) -- (297.18,94.95) ;
\draw [shift={(298.49,96.46)}, rotate = 229.21] [color={rgb, 255:red, 0; green, 0; blue, 0 }  ][line width=0.75]    (10.93,-3.29) .. controls (6.95,-1.4) and (3.31,-0.3) .. (0,0) .. controls (3.31,0.3) and (6.95,1.4) .. (10.93,3.29)   ;

\draw (285,145) node [anchor=north west][inner sep=0.75pt]    {$\random(S)$};
\draw (206,223) node [anchor=north west][inner sep=0.75pt]    {$\random(V\setminus S)$};
\draw (262.82,100.53) node [anchor=north west][inner sep=0.75pt]  [rotate=-135]  {$...$};
\draw (179.4,183.16) node [anchor=north west][inner sep=0.75pt]  [rotate=-135]  {$...$};
\draw (287.83,124.6) node [anchor=north west][inner sep=0.75pt]  [rotate=-135]  {$...$};
\draw (202.71,209.93) node [anchor=north west][inner sep=0.75pt]  [rotate=-135]  {$...$};

\end{tikzpicture}}
    \quad
    \subcaptionbox{$\sol_2$ -- solution obtained by recursively applying bisection to both $S$ and $\items \setminus S$}[0.31\textwidth]{
    \centering
    \tikzset{every picture/.style={line width=0.75pt}} 

\begin{tikzpicture}[x=0.75pt,y=0.75pt,yscale=-0.6,xscale=0.6]
\tikzstyle{every node}=[font=\tiny]

\draw   (190.43,173.9) .. controls (190.43,168.87) and (194.51,164.8) .. (199.53,164.8) .. controls (204.56,164.8) and (208.63,168.87) .. (208.63,173.9) .. controls (208.63,178.93) and (204.56,183) .. (199.53,183) .. controls (194.51,183) and (190.43,178.93) .. (190.43,173.9) -- cycle ;
\draw   (210.43,149.9) .. controls (210.43,144.87) and (214.51,140.8) .. (219.53,140.8) .. controls (224.56,140.8) and (228.63,144.87) .. (228.63,149.9) .. controls (228.63,154.93) and (224.56,159) .. (219.53,159) .. controls (214.51,159) and (210.43,154.93) .. (210.43,149.9) -- cycle ;
\draw   (307.43,54.9) .. controls (307.43,49.87) and (311.51,45.8) .. (316.53,45.8) .. controls (321.56,45.8) and (325.63,49.87) .. (325.63,54.9) .. controls (325.63,59.93) and (321.56,64) .. (316.53,64) .. controls (311.51,64) and (307.43,59.93) .. (307.43,54.9) -- cycle ;
\draw   (234.43,173.9) .. controls (234.43,168.87) and (238.51,164.8) .. (243.53,164.8) .. controls (248.56,164.8) and (252.63,168.87) .. (252.63,173.9) .. controls (252.63,178.93) and (248.56,183) .. (243.53,183) .. controls (238.51,183) and (234.43,178.93) .. (234.43,173.9) -- cycle ;
\draw    (227.63,156.6) -- (237.22,166.19) ;
\draw [shift={(238.63,167.6)}, rotate = 225] [color={rgb, 255:red, 0; green, 0; blue, 0 }  ][line width=0.75]    (10.93,-3.29) .. controls (6.95,-1.4) and (3.31,-0.3) .. (0,0) .. controls (3.31,0.3) and (6.95,1.4) .. (10.93,3.29)   ;
\draw    (214.63,156.6) -- (205.91,167.06) ;
\draw [shift={(204.63,168.6)}, rotate = 309.81] [color={rgb, 255:red, 0; green, 0; blue, 0 }  ][line width=0.75]    (10.93,-3.29) .. controls (6.95,-1.4) and (3.31,-0.3) .. (0,0) .. controls (3.31,0.3) and (6.95,1.4) .. (10.93,3.29)   ;
\draw    (191.63,179.6) -- (181.82,192.99) ;
\draw [shift={(180.63,194.6)}, rotate = 306.25] [color={rgb, 255:red, 0; green, 0; blue, 0 }  ][line width=0.75]    (10.93,-3.29) .. controls (6.95,-1.4) and (3.31,-0.3) .. (0,0) .. controls (3.31,0.3) and (6.95,1.4) .. (10.93,3.29)   ;
\draw    (206,180) -- (213.56,191.91) ;
\draw [shift={(214.63,193.6)}, rotate = 237.59] [color={rgb, 255:red, 0; green, 0; blue, 0 }  ][line width=0.75]    (10.93,-3.29) .. controls (6.95,-1.4) and (3.31,-0.3) .. (0,0) .. controls (3.31,0.3) and (6.95,1.4) .. (10.93,3.29)   ;
\draw    (238.63,180.6) -- (228.82,193.99) ;
\draw [shift={(227.63,195.6)}, rotate = 306.25] [color={rgb, 255:red, 0; green, 0; blue, 0 }  ][line width=0.75]    (10.93,-3.29) .. controls (6.95,-1.4) and (3.31,-0.3) .. (0,0) .. controls (3.31,0.3) and (6.95,1.4) .. (10.93,3.29)   ;
\draw    (250,182) -- (257.56,193.91) ;
\draw [shift={(258.63,195.6)}, rotate = 237.59] [color={rgb, 255:red, 0; green, 0; blue, 0 }  ][line width=0.75]    (10.93,-3.29) .. controls (6.95,-1.4) and (3.31,-0.3) .. (0,0) .. controls (3.31,0.3) and (6.95,1.4) .. (10.93,3.29)   ;
\draw  [dash pattern={on 0.84pt off 2.51pt}] (221.13,116.15) -- (289.63,217.6) -- (152.63,217.6) -- cycle ;
\draw   (385.43,173.9) .. controls (385.43,168.87) and (389.51,164.8) .. (394.53,164.8) .. controls (399.56,164.8) and (403.63,168.87) .. (403.63,173.9) .. controls (403.63,178.93) and (399.56,183) .. (394.53,183) .. controls (389.51,183) and (385.43,178.93) .. (385.43,173.9) -- cycle ;
\draw   (405.43,149.9) .. controls (405.43,144.87) and (409.51,140.8) .. (414.53,140.8) .. controls (419.56,140.8) and (423.63,144.87) .. (423.63,149.9) .. controls (423.63,154.93) and (419.56,159) .. (414.53,159) .. controls (409.51,159) and (405.43,154.93) .. (405.43,149.9) -- cycle ;
\draw   (429.43,173.9) .. controls (429.43,168.87) and (433.51,164.8) .. (438.53,164.8) .. controls (443.56,164.8) and (447.63,168.87) .. (447.63,173.9) .. controls (447.63,178.93) and (443.56,183) .. (438.53,183) .. controls (433.51,183) and (429.43,178.93) .. (429.43,173.9) -- cycle ;
\draw    (422.63,156.6) -- (432.22,166.19) ;
\draw [shift={(433.63,167.6)}, rotate = 225] [color={rgb, 255:red, 0; green, 0; blue, 0 }  ][line width=0.75]    (10.93,-3.29) .. controls (6.95,-1.4) and (3.31,-0.3) .. (0,0) .. controls (3.31,0.3) and (6.95,1.4) .. (10.93,3.29)   ;
\draw    (409.63,156.6) -- (400.91,167.06) ;
\draw [shift={(399.63,168.6)}, rotate = 309.81] [color={rgb, 255:red, 0; green, 0; blue, 0 }  ][line width=0.75]    (10.93,-3.29) .. controls (6.95,-1.4) and (3.31,-0.3) .. (0,0) .. controls (3.31,0.3) and (6.95,1.4) .. (10.93,3.29)   ;
\draw    (386.63,179.6) -- (376.82,192.99) ;
\draw [shift={(375.63,194.6)}, rotate = 306.25] [color={rgb, 255:red, 0; green, 0; blue, 0 }  ][line width=0.75]    (10.93,-3.29) .. controls (6.95,-1.4) and (3.31,-0.3) .. (0,0) .. controls (3.31,0.3) and (6.95,1.4) .. (10.93,3.29)   ;
\draw    (401,180) -- (408.56,191.91) ;
\draw [shift={(409.63,193.6)}, rotate = 237.59] [color={rgb, 255:red, 0; green, 0; blue, 0 }  ][line width=0.75]    (10.93,-3.29) .. controls (6.95,-1.4) and (3.31,-0.3) .. (0,0) .. controls (3.31,0.3) and (6.95,1.4) .. (10.93,3.29)   ;
\draw    (433.63,180.6) -- (423.82,193.99) ;
\draw [shift={(422.63,195.6)}, rotate = 306.25] [color={rgb, 255:red, 0; green, 0; blue, 0 }  ][line width=0.75]    (10.93,-3.29) .. controls (6.95,-1.4) and (3.31,-0.3) .. (0,0) .. controls (3.31,0.3) and (6.95,1.4) .. (10.93,3.29)   ;
\draw    (445,182) -- (452.56,193.91) ;
\draw [shift={(453.63,195.6)}, rotate = 237.59] [color={rgb, 255:red, 0; green, 0; blue, 0 }  ][line width=0.75]    (10.93,-3.29) .. controls (6.95,-1.4) and (3.31,-0.3) .. (0,0) .. controls (3.31,0.3) and (6.95,1.4) .. (10.93,3.29)   ;
\draw  [dash pattern={on 0.84pt off 2.51pt}] (416.13,116.15) -- (484.63,217.6) -- (347.63,217.6) -- cycle ;
\draw    (325.63,60.6) -- (413.05,139.46) ;
\draw [shift={(414.53,140.8)}, rotate = 222.05] [color={rgb, 255:red, 0; green, 0; blue, 0 }  ][line width=0.75]    (10.93,-3.29) .. controls (6.95,-1.4) and (3.31,-0.3) .. (0,0) .. controls (3.31,0.3) and (6.95,1.4) .. (10.93,3.29)   ;
\draw    (309.63,60.6) -- (221.03,139.47) ;
\draw [shift={(219.53,140.8)}, rotate = 318.33000000000004] [color={rgb, 255:red, 0; green, 0; blue, 0 }  ][line width=0.75]    (10.93,-3.29) .. controls (6.95,-1.4) and (3.31,-0.3) .. (0,0) .. controls (3.31,0.3) and (6.95,1.4) .. (10.93,3.29)   ;

\draw (182.63,198) node [anchor=north west][inner sep=0.75pt]    {$\cdots $};
\draw (229.63,199) node [anchor=north west][inner sep=0.75pt]    {$\cdots $};
\draw (128,222.4) node [anchor=north west][inner sep=0.75pt]    {$\balancedBisection$};
\draw (377.63,198) node [anchor=north west][inner sep=0.75pt]    {$\cdots $};
\draw (424.63,199) node [anchor=north west][inner sep=0.75pt]    {$\cdots $};
\draw (323,222.4) node [anchor=north west][inner sep=0.75pt]    {$\balancedBisection$};
\draw (215,94.4) node [anchor=north west][inner sep=0.75pt]    {$\mathrm{S}$};
\draw (396,93.4) node [anchor=north west][inner sep=0.75pt]    {$V\setminus S$};

\end{tikzpicture}
    }
    \quad
    \subcaptionbox{$\sol_3$ -- solution obtained by recursively applying bisection to $\items$}[0.31\textwidth]{
    \centering
    \tikzset{every picture/.style={line width=0.75pt}} 

\begin{tikzpicture}[x=0.75pt,y=0.75pt,yscale=-1,xscale=1]

\draw   (297.43,126.9) .. controls (297.43,121.87) and (301.51,117.8) .. (306.53,117.8) .. controls (311.56,117.8) and (315.63,121.87) .. (315.63,126.9) .. controls (315.63,131.93) and (311.56,136) .. (306.53,136) .. controls (301.51,136) and (297.43,131.93) .. (297.43,126.9) -- cycle ;
\draw   (329.43,100.9) .. controls (329.43,95.87) and (333.51,91.8) .. (338.53,91.8) .. controls (343.56,91.8) and (347.63,95.87) .. (347.63,100.9) .. controls (347.63,105.93) and (343.56,110) .. (338.53,110) .. controls (333.51,110) and (329.43,105.93) .. (329.43,100.9) -- cycle ;
\draw   (365.43,124.9) .. controls (365.43,119.87) and (369.51,115.8) .. (374.53,115.8) .. controls (379.56,115.8) and (383.63,119.87) .. (383.63,124.9) .. controls (383.63,129.93) and (379.56,134) .. (374.53,134) .. controls (369.51,134) and (365.43,129.93) .. (365.43,124.9) -- cycle ;
\draw    (347.63,100.9) -- (366.11,116.64) ;
\draw [shift={(367.63,117.93)}, rotate = 220.42000000000002] [color={rgb, 255:red, 0; green, 0; blue, 0 }  ][line width=0.75]    (10.93,-3.29) .. controls (6.95,-1.4) and (3.31,-0.3) .. (0,0) .. controls (3.31,0.3) and (6.95,1.4) .. (10.93,3.29)   ;
\draw    (329.43,100.9) -- (312.09,117.23) ;
\draw [shift={(310.63,118.6)}, rotate = 316.73] [color={rgb, 255:red, 0; green, 0; blue, 0 }  ][line width=0.75]    (10.93,-3.29) .. controls (6.95,-1.4) and (3.31,-0.3) .. (0,0) .. controls (3.31,0.3) and (6.95,1.4) .. (10.93,3.29)   ;
\draw    (298.63,132.6) -- (288.82,145.99) ;
\draw [shift={(287.63,147.6)}, rotate = 306.25] [color={rgb, 255:red, 0; green, 0; blue, 0 }  ][line width=0.75]    (10.93,-3.29) .. controls (6.95,-1.4) and (3.31,-0.3) .. (0,0) .. controls (3.31,0.3) and (6.95,1.4) .. (10.93,3.29)   ;
\draw    (313,133) -- (320.56,144.91) ;
\draw [shift={(321.63,146.6)}, rotate = 237.59] [color={rgb, 255:red, 0; green, 0; blue, 0 }  ][line width=0.75]    (10.93,-3.29) .. controls (6.95,-1.4) and (3.31,-0.3) .. (0,0) .. controls (3.31,0.3) and (6.95,1.4) .. (10.93,3.29)   ;
\draw    (369.63,131.6) -- (359.82,144.99) ;
\draw [shift={(358.63,146.6)}, rotate = 306.25] [color={rgb, 255:red, 0; green, 0; blue, 0 }  ][line width=0.75]    (10.93,-3.29) .. controls (6.95,-1.4) and (3.31,-0.3) .. (0,0) .. controls (3.31,0.3) and (6.95,1.4) .. (10.93,3.29)   ;
\draw    (381,133) -- (388.56,144.91) ;
\draw [shift={(389.63,146.6)}, rotate = 237.59] [color={rgb, 255:red, 0; green, 0; blue, 0 }  ][line width=0.75]    (10.93,-3.29) .. controls (6.95,-1.4) and (3.31,-0.3) .. (0,0) .. controls (3.31,0.3) and (6.95,1.4) .. (10.93,3.29)   ;

\draw (296.63,151) node [anchor=north west][inner sep=0.75pt]    {$\cdots $};
\draw (367.63,150) node [anchor=north west][inner sep=0.75pt]    {$\cdots $};
\draw (260,161.4) node [anchor=north west][inner sep=0.75pt]    {$\balancedBisection$};
\draw (333,76.4) node [anchor=north west][inner sep=0.75pt]    {$V$};

\end{tikzpicture}
    }
    \caption{Solutions $\sol_1$, $\sol_2$ and $\sol_3$ found by Algorithm~\ref{alg:maxsathc}. We select the best of these solutions. $S$ is the set of vertices corresponding to positive variables in the assignment given by $\bmaxsatpar$.}
    \label{fig:solutions}
\end{figure*}
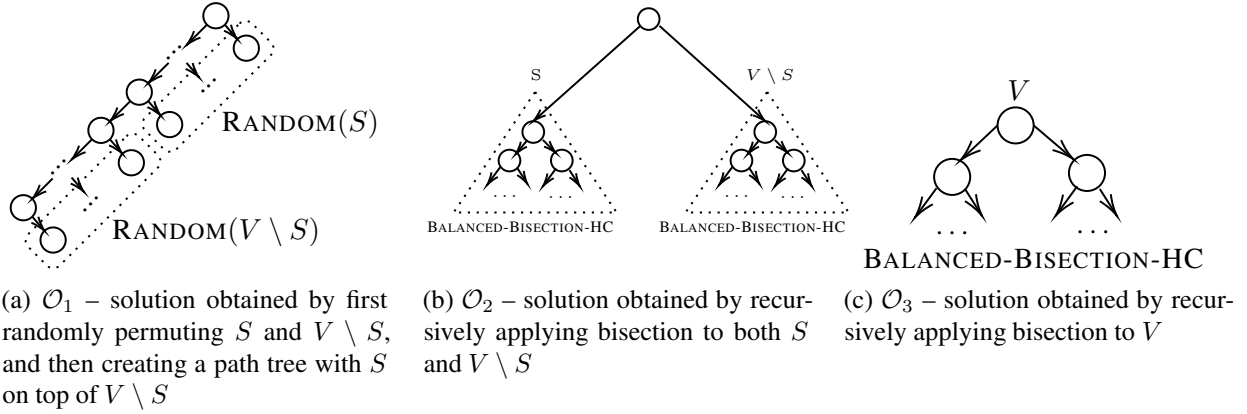

\section{Improved Approximation for the CKMM Objective}
\label{sec:proof}

In this section we prove that Algorithm~\ref{alg:maxsathc} achieves $0.74$ approximation. Recall that Algorithm~\ref{alg:maxsathc} finds solutions $\sol_1,\sol_2,\sol_3$ (see Figure~\ref{fig:solutions}) and selects the best of them.

\subsection{Notation and Proof Outline}

To simplify the presentation, we define the following version of the CKMM objective. Note that this preserves the multiplicative approximation.
\begin{definition}[Hierarchical clustering objective]\label{def:ckmm-obj}
We define a normalized version of the CKMM objective from~\citet*{CohenKMM19}: ${\f(\tree) = \frac{1}{|\items|}\ckmm(\tree)}$. 
\end{definition}

For a set $S \subseteq \items$ we use notation $d(S) = \sum_{(u < v) \in S \times S} d(u,v)$.
For a pair of sets $S,T \subseteq \items$ we use notation $d(S,T) = \sum_{u \in S, v \in T} d(u,v)$.
Similarly, $f(S)$ denotes the objective achieved by all pairs inside $S$ and $f(S,T)$ -- by all pairs in $S \times T$.
Let $D = d(\items)$ and note that by rescaling all distances w.l.o.g we can assume $D = 1$.

For \maxsat, we use propositional variables $x_v \in \{0,1\}$ for $v \in \items$. An input to \maxsat is a weighted collection of disjunctions on at most two variables each.
Our proof is based on the following ideas.

\paragraph{Split the set of items based on $\tree^*$}
Similarly to~\citet*{AhmadianCEMMLY20}, based on the optimal tree $\tree^*$, we can split $\items$ into $3$ sets $A$, $B$, $C$ as shown in Figure~\ref{fig:abc}, which we use for two purposes.
First, using the fact that $\tree^*$ is optimal, we can upper-bound the CKMM objective in terms of $A$, $B$ and $C$ (see Lemma~\ref{lem:ub}).
Second, considering balanced partitions of $V$ such that $A$ belongs to one part, $C$ belongs to another part and elements of $B$ are distributed randomly, we can lower-bound the optimal objective $\optsat$ for \bmaxsatpar (see Lemma~\ref{lem:sat-lb}).

\paragraph{Lower bound based on $\optsat$}
We can lower-bound the objective of the best of $\sol_1$ and $\sol_2$ in terms of $\optsat$ (see Lemma~\ref{lem:alglb}) using the fact that we can find a $0.94$-approximation of $\optsat$ in polynomial time~\citep*{AustrinBG17}.
Intuitively, when placing $S$ on top of $V \setminus S$ in $\sol_1$, $d(S)$ and $d(S, V \setminus S)$ get much larger coefficients than $d(V \times S)$, which explains why selecting $S$ based on $\bmaxsatpar$ results in good solution.

However, the relation between quality of solution of $\bmaxsatpar$ and $f(\sol_1)$ is not straightforward, since $d(S)$ and $d(S, V \setminus S)$ participate in $f(\sol_1)$ with different coefficients (see Proposition~\ref{prop:path}), while in \bmaxsatpar objective they participate with the same coefficient.
On the other hand, in $\sol_2$ they also participate with different coefficients (see Proposition~\ref{prop:cut}).
By balancing these solutions, we can express $d(S)$ in terms of $d(S, V \setminus S)$, and express a lower bound on $\max(f(\sol_1), f(\sol_2))$ in terms of $\optsat$.

\paragraph{When $\sol_1$ and $\sol_2$ don't suffice, we can upper-bound $\opt$}
Finally, if the best of $\sol_1$ and $\sol_2$ doesn't achieve a required approximation, we can show that $\opt$ is non-trivially bounded.
In this case, we use $\sol_3$ which is computed using \balancedBisection(Algorithm~\ref{algo:bbisectionclustering}~\citep*{AhmadianCEMMLY20}), a top-down approach which recursively splits the set using maximum balanced bisection.
\balancedBisection finds a solution with CKMM objective at least $\frac 23$ (see Proposition~\ref{lem:avglink}), which achieves non-trivial approximation when $\opt < 1$ (see the proof of Theorem~\ref{thm:main_thm}).


\begin{algorithm}[h]
\caption{\balancedBisection}
\label{algo:bbisectionclustering}
	\SetKwInOut{Input}{input}
	\SetKwInOut{Output}{output}
    \Input{Set $\items$, distance function $d \colon \items \times \items \to \mathbb R_{\ge 0}$.}
	\Output{Hierarchical clustering of $S$}
	\If{$|\items| = 1$} {
	  \Return $V$
	}
    Set $C \gets \emptyset$\;
    \For{$(u,v) \in \items \times \items$}  {
        Set $C \gets \cup d(u,v) \cdot (x_u, x_v)$
    }
    $S, T \gets \textsc{Max Balanced Bisection}(C)$\\
    \Return $\set{S, T} \cup \balancedBisection(S, d) \cup \balancedBisection(T, d)$
\end{algorithm}

\begin{figure*}[htb!]
    \centering
    \subcaptionbox{Sets $A$, $B$ and $C$ obtained from the optimal tree: $A$ and $B$ are nodes of size $\le \nicefrac n2$ such that their parent has size $> \nicefrac n2$; $C$ is composed of the rest of the elements.\label{fig:abc}}[0.44\textwidth]{
    \centering
    \input{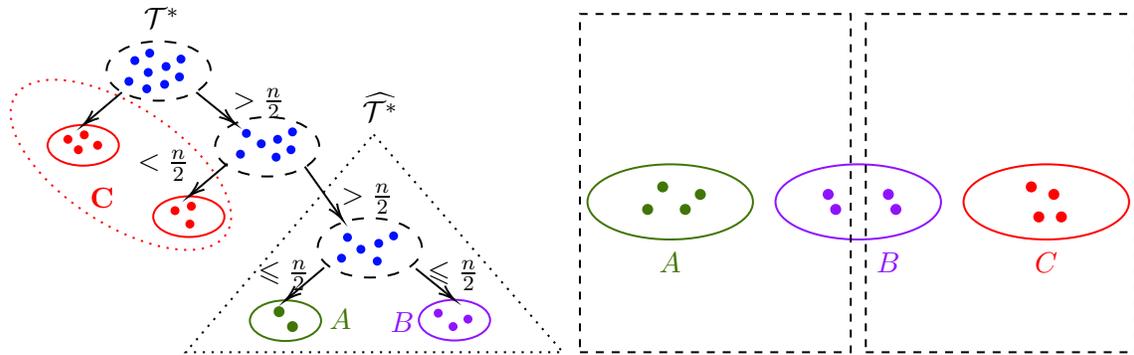}}
    \quad
    \subcaptionbox{Partition based on $A$, $B$ and $C$. $A$ belongs to one part, $C$ belongs to another one, and elements from $B$ are randomly distributed between parts. \label{fig:part_abc}}[0.44\textwidth]{
    \centering
    \tikzset{every picture/.style={line width=0.75pt}} 

\begin{tikzpicture}
\draw [color={rgb, 255:red, 65; green, 117; blue, 5 }] (0,0) ellipse (11mm and 5mm);
\draw [color={rgb, 255:red, 65; green, 117; blue, 5 }] (0,-0.8) node {$A$};
\draw [color={rgb, 255:red, 255; green, 0; blue, 0 }] (5,0) ellipse (11mm and 5mm);
\draw [color={rgb, 255:red, 255; green, 0; blue, 0 }] (5,-0.8) node {$C$};
\draw [color={rgb, 255:red, 144; green, 19; blue, 254 }] (2.5,0) ellipse (11mm and 5mm);
\draw [color={rgb, 255:red, 144; green, 19; blue, 254 }] (2.9,-0.8) node {$B$};

\draw [draw=black, dashed] (2.4,2.5) rectangle (-1.2,-2);
\draw [draw=black, dashed] (6.2,2.5) rectangle (2.6,-2);

\draw [fill={rgb, 255:red, 65; green, 117; blue, 5 }, color={rgb, 255:red, 65; green, 117; blue, 5 }] (-0.1,0.2) ellipse (0.6mm and 0.6mm);
\draw [fill={rgb, 255:red, 65; green, 117; blue, 5 }, color={rgb, 255:red, 65; green, 117; blue, 5 }] (0.2,-0.1) ellipse (0.6mm and 0.6mm);
\draw [fill={rgb, 255:red, 65; green, 117; blue, 5 }, color={rgb, 255:red, 65; green, 117; blue, 5 }] (-0.3,-0.1) ellipse (0.6mm and 0.6mm);
\draw [fill={rgb, 255:red, 65; green, 117; blue, 5 }, color={rgb, 255:red, 65; green, 117; blue, 5 }] (0.4,0.1) ellipse (0.6mm and 0.6mm);

\draw [fill={rgb, 255:red, 255; green, 0; blue, 0 }, color={rgb, 255:red, 255; green, 0; blue, 0 }] (5.1,0.1) ellipse (0.6mm and 0.6mm);
\draw [fill={rgb, 255:red, 255; green, 0; blue, 0 }, color={rgb, 255:red, 255; green, 0; blue, 0 }] (5.2,-0.2) ellipse (0.6mm and 0.6mm);
\draw [fill={rgb, 255:red, 255; green, 0; blue, 0 }, color={rgb, 255:red, 255; green, 0; blue, 0 }] (4.8,0.2) ellipse (0.6mm and 0.6mm);
\draw [fill={rgb, 255:red, 255; green, 0; blue, 0 }, color={rgb, 255:red, 255; green, 0; blue, 0 }] (4.9,-0.2) ellipse (0.6mm and 0.6mm);

\draw [fill={rgb, 255:red, 144; green, 19; blue, 254 }, color={rgb, 255:red, 144; green, 19; blue, 254 }] (2.9,0.1) ellipse (0.6mm and 0.6mm);
\draw [fill={rgb, 255:red, 144; green, 19; blue, 254 }, color={rgb, 255:red, 144; green, 19; blue, 254 }] (3.0,-0.1) ellipse (0.6mm and 0.6mm);

\draw [fill={rgb, 255:red, 144; green, 19; blue, 254 }, color={rgb, 255:red, 144; green, 19; blue, 254 }] (2.2,-0.1) ellipse (0.6mm and 0.6mm);
\draw [fill={rgb, 255:red, 144; green, 19; blue, 254 }, color={rgb, 255:red, 144; green, 19; blue, 254 }] (2.1,0.1) ellipse (0.6mm and 0.6mm);

\end{tikzpicture}}
    \caption{Sets $A$, $B$ and $C$ and the partition based on them}
\end{figure*}


\subsection{Auxiliary lemmas}
\newtext{Note that $\textsc{Max Balanced Bisection}(S)$ is at least as good} as an algorithm which randomly partitions $S$ into two disjoint sets $X$ and $Y$ with cardinalities $|X| = \lfloor\nicefrac{|S|}{2}\rfloor, |Y| = \lceil\nicefrac{|S|}{2}\rceil$.
Each edge is cut with probability
$2 \frac{\lceil\nicefrac{|S|}{2}\rceil \cdot \lfloor\nicefrac{|S|}{2}\rfloor}{|S| (|S| - 1)}$.
When $|S|$ is even, the probability is $\frac {|S|} {2 (|S| - 1)} \ge \frac 12$.
When $|S|$ is odd, the probability is $\frac {\lceil\nicefrac{|S|}{2}\rceil}{|S|}$.

\begin{proposition}\label{lem:avglink}
\newtext{Let $\items$ be a set of items and $d \colon \items \times \items \to \mathbb R_{\ge 0}$ be a distance function.
Let \textsc{Max Balanced Bisection}(S) be any bisection algorithm which partitions the set $S \subseteq \items$ into two equal-sized parts $(X,Y)$ such that the size of the cut between these parts $d(X,Y)$ is at least as good as the expected value of a random partitioning.
Then for any $S \subseteq \items$, for $\balancedBisection(S, d)$ (Algorithm~\ref{algo:bbisectionclustering}) which uses \textsc{Max Balanced Bisection} as a subroutine at every step we have:}
\begin{align}
    f(\balancedBisection(S,d)) \ge \frac23 \cdot d(S) \cdot \frac {|S|}{|\items|}
\end{align}
\end{proposition}
\begin{proof}
By induction by size of subtree. For each vertex consider the topmost split on $X$ and $Y$ assuming $|X| \leq |Y|$.
\begin{align*}
\arxiv{\f(X \cup Y)}
\notarxiv{&\f(X \cup Y) \\}
&= d(X, Y) \frac{|X| + |Y|}{|\items|} + \f(X) + \f(Y)  \\
&\geq d(X, Y) \frac{|X| + |Y|}{|\items|} + \frac 23 \left(\frac{|X|}{|\items|}d(X) + \frac{|Y|}{|\items|}d(Y)\right)  \\
&\geq d(X, Y) \frac{|X| + |Y|}{|\items|} + \frac 23 \cdot \frac{|X|}{|\items|} \left(d(X) + d(Y)\right)  \\
&\geq \frac{2}{3|\items|}(|X|d(S) + |Y| d(X, Y)) + \frac{|X| + |Y|}{3|\items|}d(X, Y) \\
&\geq \frac{d(S)}{3} \cdot \frac {|S|} {|\items|} \left(2\frac {|X|}{|S|} + \left(1 + 2 \frac {|Y|}{|S|}\right) \frac {d(X,Y)} {d(S)}\right)
\end{align*}
If $|S|$ is even then $\frac {|X|}{|S|} = \frac {|Y|}{|S|} = \frac 12$. Since $d(X, Y) \geq \frac{1}{2} d(X \cup Y)$,

\begin{align*}
\f(S)
\geq \frac{d(S)}{3} \cdot \frac {|S|} {|\items|} \left(1 + \left(1 + 2 \frac 12\right) \frac 12\right)
= \frac23 \cdot d(S) \cdot \frac {|S|}{|\items|}
\end{align*}

If $|S|$ is odd and $|X| = a$, then $|Y| = a + 1$ and $|S| = 2a + 1$.  Since $d(X, Y) \geq \frac{a + 1}{2a + 1}d(X \cup Y)$,

\begin{align*}
\f(S)
&\geq \frac{d(S)}{3} \cdot \frac {|S|} {(2a + 1)|\items|} \left(2a + \left(4a + 3\right) \frac {a+1} {2a + 1} \right) \\
&\geq \frac{d(S)}{3} \cdot \frac {|S|} {|\items|} \frac {2a (2a + 1) + (4a + 3)(a+1)} {(2a+1)^2} \\
&\geq \frac{d(S)}{3} \cdot \frac {|S|} {|\items|} \frac {8 a^2 + 9a + 2} {(2a+1)^2} \\
&\ge \frac23 \cdot d(S) \cdot \frac {|S|}{|\items|}
\end{align*}

\end{proof}

\begin{lemma}\label{lem:path}
\newtext{Let $\items$ be a set of items and $d \colon \items \times \items \to \mathbb R_{\ge 0}$ be a distance function.
For $S \subseteq \items$, let $\tpath(\random(S))$ be a random path tree obtained from $S$ (see Figure~\ref{fig:sols_path}).
Then we have:}
\begin{align}
    \mathbb E[f(\tpath(\random(S)))] \ge \frac23 d(S) \cdot \frac {|S|}{|\items|}    
\end{align}
\end{lemma}
\begin{proof}
For each $u,v \in S$ consider constructing $\random(S)$ as follows: first pick positions for $u,v$ and then sample positions for all other vertices. Notice that for every other vertex $w$ we have $\Pr[w \text{ is under }\lca(u,v)] = \nicefrac 23$. Thus we have $\mathbb E[|\lca(u,v)|] = 2 + \nicefrac23 \cdot (|S| - 2)$.
Hence
\begin{align*}
    \mathbb E[f(\tpath(\random(S)))]
    & = \frac{d(S) \cdot (2 + \nicefrac23 \cdot (|S| - 2))}{|\items|}
    \notarxiv{\\&} \ge \frac23 \cdot d(S) \cdot \frac {|S|}{|\items|}
\end{align*}
\end{proof}

Let $\optsat$ be the weight of satisfied clauses in the optimum solution to \bmaxsatpar(d).

\begin{lemma}\label{lem:alglb}
The best of the two solutions $\sol_1$ and $\sol_2$ computed in Algorithm~\ref{alg:maxsathc} has expected value at least
\begin{align}
    \algsat \ge \frac13 + \frac49 \cdot 0.94 \cdot \optsat
\end{align}
\end{lemma}
\begin{proof}
Let $S \subseteq \items$ be the set of vertices corresponding to positive variables in the assignment given by \bmaxsatpar$(d)$ in Algorithm~\ref{alg:maxsathc}.
Let $T = \items \setminus S$, $D_1 = d(S)$ and $D_2 = d(S, T)$.
\begin{proposition}\label{prop:path}
For $\sol_1$, $D_1$ and $D_2$ as defined above,
$
    \mathbb E[\f(\sol_1)] \ge \frac13 + \frac12 D_1 + \frac{5}{12} D_2.
$
\end{proposition}
\begin{proof}
Note that by construction of $\sol_1$, for edges $(u,v) \in S \times T$ we have
$$\mathbb E[|\lca(u,v)|] = 1 + \frac {|\items|} 2 +  \frac12 \left(\frac {|\items|} 2 - 1 \right) \ge \frac34 |\items|.$$
Thus we have:
\begin{align*}
\mathbb E[\f(\sol_1)]
&= \mathbb E[f(T)] + \left(\frac {d(S)} 2 + \mathbb E[f(S)]\right) + \frac 34 d(S, T)\\
&\ge \frac  {d(T)} 3 + \left(\frac {d(S)} 2 + \frac {d(S)} 3 \right) + \frac 34 d(S, T) \\
&= \frac13  + \frac {d(S)} 2 + \frac{5}{12} d(S, T) \\
&= \frac13 + \frac12 D_1 + \frac{5}{12} D_2.
\end{align*}
\end{proof}

\begin{proposition}\label{prop:cut}
For $\sol_1$, $D_1$ and $D_2$ as defined above,
$\f(\sol_2) \ge \frac13 + \frac23 \cdot D_2.$
\end{proposition}
\begin{proof}
By Proposition~\ref{lem:avglink} we have:
\begin{align*}
\f(\sol_2) &= d(S, T) + f(\balancedBisection(S))\\
& \ \ \ \ + f(\balancedBisection(T)) \\
& \ge D_2 + \frac23 \cdot \frac 12 (d(S) + d(T)) \\
&= \frac13 + \frac23 D_2.
\end{align*}
\end{proof}

\paragraph{Continuation of the proof of Lemma~\ref{lem:alglb}.} Taking the best of the two solutions, by Proposition~\ref{prop:path} and Proposition~\ref{prop:cut} we have:
\begin{align*}
    \max(\f(\sol_1),\f(\sol_2)) &\ge \frac13 + \frac{5}{12} \cdot D_2 + \frac14 \cdot \max(2 D_1, D_2)
\end{align*}
The minimum is achieved when $2D_1 = D_2$, and therefore $D_2 = \frac {2 (D_1 + D_2)} 3$
\begin{align*}
    \max(\f(\sol_1),\f(\sol_2)) &\ge \frac13 + \frac{2}{3} \cdot \frac {2 (D_1 + D_2)} 3
    \notarxiv{\\&}= \frac13 + \frac49 (D_1 + D_2)
    \notarxiv{\\&}\ge \frac13 + \frac49 \cdot 0.94 \cdot \optsat.
\end{align*}
\end{proof}

\subsection{Proof of the main theorem}

\begin{theorem}
Algorithm~\ref{alg:maxsathc} constructs a tree which gives a $\gamma$-approximation of the optimum value of the objective $\f$ over all binary trees, where $\gamma  = \frac{\nicefrac 43}{1 + \frac{3}{4 \cdot 0.94}} \ge 0.74$.
\label{thm:main_thm}
\end{theorem}
\begin{proof}
We denote the value of the best tree constructed by our algorithm as $\alg$.
Consider the optimum tree $\topt$ and let $\opt = \f(\topt)$. In $\topt$ there exists a subtree $\topts$ such that the number of leaves in it is at least $n/2$ while the number of leaves in each of its subtrees is less than $n/2$. Let the sets of leaves in two children of $\topts$ be denoted as $A$ and $B$ respectively and let $C = \items \setminus (A \cup B)$ (see Figure~\ref{fig:abc}).
We further denote $a = \nicefrac {|A|}{|\items|}, b = \nicefrac {|B|}{|\items|}$ and $c = \nicefrac {|C|}{|\items|}$.

The rest of the proof is the following.
Lemma~\ref{lem:ub} specifies an upper bound on $OPT$ in terms of $A$, $B$ and $C$.
Lemma~\ref{lem:sat-lb} finds a lower bound on the optimum solution of $\bmaxsatpar$; using Lemma~\ref{lem:alglb}, it also gives a lower bound on solution found by the algorithm.
Based on these results, Lemma~\ref{lem:derivative} shows that for lower bound it suffices to consider the case $d(A)=d(B)=0$, simplifying the bounds.
Finally, we combine this with the fact that the objective for $\sol_3$ is $\nicefrac 23 d(\items)$ and compute the approximation.

\begin{lemma}\label{lem:ub}
The optimum value of the hierarchical clustering objective $\f(\topt)$ can be bounded as:
\begin{align*}
OPT \le\
&a \cdot d(A) + b \cdot d(B) + (1 - c) \cdot d(A,B)
\notarxiv{\\&}+ d(C) + d(A,C) + d(B,C)    
\end{align*}
\end{lemma}
\begin{proof}
Follows by Definition~\ref{def:ckmm-obj} of $\f$.
Indeed, for pairs of leaves in $A$ and $B$ the fraction of leaves under their $\lca$ is at most $a$ and $b$ respectively. For pairs of leaves $(u,v) \in A \times B$ this fraction is at most $a + b = 1 - c$. 
\end{proof}

\begin{lemma}\label{lem:sat-lb}
    The weight $\optsat$ of satisfied clauses in the optimum solution to \bmaxsatpar(d) can be bounded as:
    \begin{align}
        \optsat \ge
            &\ d(C) + d(A,C) + d(B,C)
            \notarxiv{\\&}+ \left(1 - \left(\frac{\nicefrac 12 - a}{b}\right)^2\right) d(B) + \frac{\nicefrac 12 - c}{b} d(A,B)
    \end{align}
\end{lemma}
\begin{proof}

Consider a solution which sets (see Figure~\ref{fig:part_abc}):
\begin{enumerate}
    \item $x_u = 1$ for all $u \in C$,
    \item $x_u = 0$ for all $u \in A$,
    \item $x_u = 1$ for a random $\nicefrac {(\nicefrac 12 - c)} b$ fraction of $B$ and $x_u = 0$ for the rest of $B$ (a $\nicefrac{(\nicefrac12 - a)}{b}$ fraction).
\end{enumerate}
Note that this solution is balanced as it sets exactly $\nicefrac n2$ variables to $1$. 
Since all variables in $C$ are set to $1$, the solution has value at least $d(C) + d(A,C) + d(B,C)$.
For $u,v \in B$, we have $\Pr[x_u = x_v = 0 | u,v \in B] = \left(\frac{\nicefrac 12 - a}{b}\right)^2$ so the expected value for edges with both endpoints in $B$ is at least $\left(1 - \left(\frac{\nicefrac 12 - a}{b}\right)^2\right) \cdot d(B)$.
Similarly, the expected value for edges with $u \in A$ and $v \in B$ is $\frac{\nicefrac 12 - c}{b} \cdot d(A,B)$ since $\Pr[x_v = 1] = \frac{\nicefrac 12 - c}{b}$. Hence there exists a solution which has value of at least the expectation of our random solution, completing the proof.
\end{proof}

Recall that $d(A \cup B) = d(A) + d(B) + d(A, B)$.
We let $\algsat = \nicefrac13 + \nicefrac49 \cdot 0.94 \cdot \optsat$ be our lower bound on the value given by the algorithm by Lemma~\ref{lem:alglb}.

\begin{lemma}\label{lem:derivative}
Let $a$, $b$, and $d(A \cup B)$ be fixed an satisfy the following inequalities:
\begin{compactitem}
    \item $0 \le a,b \le \frac 12$
    \item $a+b \ge \frac 12$
    \item $d(A \cup B) \le 1$
\end{compactitem}
Let the lower bound on $\algsat$ be
\begin{align*}
    &f(d(A), d(B)) \\
    &=\frac 13 + \frac 49 \cdot 0.94 \left((1 - d(A \cup B)) + \left(1 - \left(\frac{\nicefrac 12 - a}{b}\right)^2\right) d(B) + \frac{a + b - \nicefrac 12}{b} (d(A \cup B) - d(A) - d(B))\right)
,\end{align*}
and the upper bound on optimum be
\[
    g(d(A), d(B)) = a \cdot d(A) + b \cdot d(B) + (a + b) (d(A \cup B) - d(A) - d(B)) + (1 - d(A \cup B))
.\]
Then the following holds:
\begin{compactenum}
    \item For fixed $d(A)$, the lower bound on the approximation factor of Algorithm~\ref{alg:hiclustering},
    $\frac{\max(\nicefrac 23, f(d(A), d(B)))}{g(d(A), d(B))}$, is minimized when $d(B) = 0$.
    \item $\frac{\nicefrac 23}{g(d(A), d(B))}$ monotonically increases w.r.t. $d(A)$ and $d(B)$.
    \item For $d(B)=0$, $\frac{f(d(A), 0)}{g(d(A), 0)}$ depends on $d(A)$ monotonically.
\end{compactenum}
\end{lemma}
\begin{proof}
The third statement follows from the fact that $\frac{f(d(A),0)}{g(d(A),0)}$ is a linear fractional function w.r.t. $d(A)$, and hence monotone on any continuity interval.
The second statement follows from
\begin{align*}
    g(d(A), d(B)) = 1 - (1 - a - b) d(A \cup B) - b \cdot d(A) - a \cdot d(B) \label{eq:rewrite_g}
,\end{align*}
and the fact that the coefficients in front of $d(A)$ and $d(B)$ are non-positive.
It remains to prove the first statement.

From the second statement, $\frac{\nicefrac 23}{g(d(A), d(B))}$ is minimized when $d(B) = 0$.
The other term, $\frac{f(d(A), d(B))}{g(d(A), d(B))}$, is a linear fractional function w.r.t. $d(B)$, i.e.
\[
    \frac{f(d(A), d(B))}{g(d(A),d(B))} = \frac{\alpha_1 d(B) + c_1}{\alpha_2 d(B) + c_2}
,\]
where
\begin{align*}
    \alpha_1 &= \frac 49 \cdot 0.94 \cdot \left(1 - \left(\frac{\nicefrac 12 - a}{b}\right)^2 - \frac{a + b - \nicefrac 12}{b}\right) \\
    c_1 &= \frac 13 + \frac 49 \cdot 0.94 \cdot \left((1 - d(A \cup B)) + \left(1 - \left(\frac{\nicefrac 12 - a}{b}\right)^2\right) d(A) + \frac{a + b - \nicefrac 12}{b} (d(A \cup B) - d(A))\right) \\
    \alpha_2 &=  -a \\
    c_2 &= 1 - (1 - a - b) d(A \cup B) - b \cdot d(A)
\end{align*}

To show that $\frac{f(d(A), d(B))}{g(d(A), d(B))}$ is minimized when $d(B)=0$, it suffices to show that the derivative is non-negative, meaning $\alpha_1 c_2 - \alpha_2 c_1 \ge 0$.
We clearly have that $c_1 \ge 0$, $\alpha_2 \le 0$, and $c_2 \ge 0$.
For $\alpha_1$, we have:
\[
    1 - \left(\frac{\nicefrac 12 - a}{b}\right)^2 - \frac{a + b - \nicefrac 12}{b}
    = \frac{a + b - \nicefrac 12}{b} \cdot \frac{b + \nicefrac 12 - a}{b} - \frac{a + b - \nicefrac 12}{b}
    =  \frac{a + b - \nicefrac 12}{b} \cdot \frac{\nicefrac 12 - a}{b} \ge 0
.\]
Hence, $\alpha_1 c_2 - \alpha_2 c_1 \ge 0$, finishing the proof.
\end{proof}

\paragraph{Continuation of the proof of Theorem~\ref{thm:main_thm}.}
By Lemma~\ref{lem:derivative} we can assume that $d(B) = 0$.
To improve the presentation, we use $d_A$ to denote $d(A)$ and $d_{A \cup B}$ to denote $d(A \cup B)$ and let $\xi = \frac 49 \cdot 0.94$.
Similarly to the Lemma, for fixed $a$, $b$, and $d_{A \cup B}$ we define
\begin{align*}
    f(d_A) =\frac 13 + \xi \left(1 - \frac{\nicefrac 12 - a}{b} d_{A \cup B} - \frac{a + b - \nicefrac 12}{b} d_A\right)
,\end{align*}
and
\[
    g(d_A) = 1 - (1 - a - b) d_{A \cup B} - b \cdot d_A
.\]
We show that for the lower bound on the approximation factor $\gamma = \frac{\max(\nicefrac 23, f(d(A)))}{g(d(A))}$ we have $\gamma \ge 0.741$.

From Lemma~\ref{lem:derivative}, we know that $\frac{\nicefrac 23}{g(d(A))}$ increases w.r.t. $d_A$.
For fixed $a,b,d, d(A\cup b)$, let $d_A^*$ be the value of $d_A$ which minimizes the approximation.
We consider the following cases:
\begin{itemize}
    \item If $\frac{f(d_A)}{g(d_A)}$ increases, then $d_A^* = 0$.
    \item If $\frac{f(d_A)}{g(d_A)}$ decreases and $f(d_A) > \nicefrac 23$ for all $d_A \in (0, d_{A \cup B})$,
        then $d_A^* = d_{A \cup B}$.
    \item If $\frac{f(d_A)}{g(d_A)}$ decreases and $f(d_A) < \nicefrac 23$ for all $d_A \in (0, d_{A \cup B})$,
        Then $d_A^* = 0$.
    \item If $\frac{f(d_A)}{g(d_A)}$ decreases and there exists $\tilde d_A \in (0, d_{A \cup B})$ such that $f(\tilde d_A) = \nicefrac 23$,
        then $d_A^* = \tilde d_A$.
\end{itemize}
Summing it up, we need to analyze three cases: $d_A^*=0$, $d_A^* = d_{A \cup B}$, and, under certain conditions, $d_A^* \in (0, d_{A \cup B})$.
We also can assume that $g(d_A^*) \ge \frac{\nicefrac 23}{\gamma}$, since otherwise $\frac{\nicefrac 23}{g(d_A^*)}$ achieves $\gamma$-approximation.

\paragraph{Case 1: $d_A^* = 0$.}
In this case, $f(0) = \frac 13 + \xi \left(1 - \frac{\nicefrac 12 - a}{b} d_{A \cup B}\right)$ and $g(0) = 1 - (1 - a - b) d_{A \cup B}$.
Using that $a,b \le \nicefrac 12$ and $(1 - a - b) d_{A \cup B} = 1 - g(0)$, we have:
\begin{align*}
    1 - \frac{\nicefrac 12 - a}{b} d_{A \cup B}
    &\ge 1 - 2(\nicefrac 12 - a) d_{A \cup B} \\
    &\ge 1 - 2(1 - a - b) d_{A \cup B} \\
    &= 1 - 2 (1 - g(0)) \\
    &= 2 g(0) - 1
\end{align*}
Hence, using $g(0) \ge \frac{\nicefrac 23}{\gamma}$, we have:
\[
    \frac{f(0)}{g(0)}
    \ge \frac{\frac 13 + \xi (2 g(0) - 1)}{g(0)}
    = 2 \xi + \frac{\frac 13 - \xi}{g(0)}
    \ge 0.8355 - \gamma \cdot 0.1267
,\]
which exceeds $\gamma$ when $\gamma \le \frac{0.8355}{1.1267} \ge 0.741$.

\paragraph{Case 2: $d_A^* = d_{A \cup B}$}
In this case, $f(d_{A \cup B}) = \frac 13 + \xi \left(1 - d_{A \cup B}\right)$ and $g(d_{A \cup B}) = 1 - (1 - a) d_{A \cup B}$.
\[
    \frac{f(d_{A \cup B})}{g(d_{A \cup B})}
    = \frac{\frac 13 + \xi \left(1 - d_{A \cup B}\right)}{1 - (1 - a) d_{A \cup B}}
.\]
Hence, $\frac{f(d_{A \cup B})}{g(d_{A \cup B})}$ is a linear fractional function w.r.t. $d_{A \cup B}$, and hence it's monotone w.r.t. $d_{A \cup B}$.
When $\frac{f(d_{A \cup B})}{g(d_{A \cup B})}$ increases, the minimum is achieved for $d_{A \cup B} = 0$, implying $d_A^* = 0$, which is analyzed in Case 1 above.
Otherwise, by taking the derivative, $\frac{f(d_{A \cup B})}{g(d_{A \cup B})}$ decreases iff $\xi \le (1 - a) (\frac 13 + \xi)$, which is equivalent to $\frac{1}{1 - a} \le \frac{\frac 13 + \xi}{\xi}$.
From $g(d_{A \cup B}) \ge \frac{\nicefrac 23}{\gamma}$, the maximum possible $d_{A \cup B}$ is achieved when we have equality:
\[
    g(d_{A \cup B}) = \frac{\nicefrac 23}{\gamma}
    \iff d_{A \cup B} = \frac{1 - \frac{\nicefrac 23}{\gamma}}{1 - a}
    \le \frac{1}{\xi}(\frac 13 + \xi)(1 - \frac{\nicefrac 23}{\gamma})
.\]
Putting everything together:
\[
    \frac{f(d_{A \cup B})}{g(d_{A \cup B})}
    = \frac{\gamma}{\nicefrac 23}\left(\frac 13 + \xi \left(1 - d_{A \cup B}\right)\right)
    \ge \frac{\gamma}{\nicefrac 23}\left(\frac 13 + \xi - (\frac 13 + \xi)(1 - \frac{\nicefrac 23}{\gamma})\right)
    \ge \gamma
.\]

\paragraph{Case 3: $d_A^* \in (0, d_{A \cup B})$ with $f(d_A^*) = \nicefrac 23$ and $\frac{f(d_A)}{g(d_A)}$ decreases}
Since $f(d_A^*) = \nicefrac 23$, we have:
\[
    \frac 13 + \xi \left(1 - \frac{\nicefrac 12 - a}{b} d_{A \cup B} - \frac{a + b - \nicefrac 12}{b} d_A^*\right) = \frac 23
    \iff d_{A \cup B} - d_A^* = \frac{b}{a + b - \nicefrac 12} (d_{A \cup B} - \delta)
,\]
where $\delta = 1 - \frac{1}{3\xi}$.
Since $f(0) \ge \nicefrac 23$ and $f(d_{A \cup B}) \le \nicefrac 23$, we also have:
\begin{align}
    \delta \le d_{A \cup B} \le \frac{b}{\nicefrac 12 - a} \delta \label{eq:min_max_dab}
.\end{align}
Since $f(d_A^*) = \nicefrac 23$, it suffices to show that $g(d_A^*) \le \frac{\nicefrac 23}{\gamma}$ where
\[  g(d_A^*)
    = 1 - (1 - a - b) d_{A \cup B} - b \cdot d_A^*
    = 1 - (1 - a) d_{A \cup B} + \frac{b^2}{a + b - \nicefrac 12} (d_{A \cup B} - \delta)
.\]
Since $g(d_A^*)$ is monotone w.r.t. $d_{A \cup B}$, it's maximized only when $d_{A \cup B}$ is either minimized of maximized, i.e., based on Equation~\ref{eq:min_max_dab}, either $d_{A \cup B} = \delta$ or $d_{A \cup B} = \frac{b}{\nicefrac 12 - a} \delta$.
But for $d_{A \cup B} = \delta$, we have $d_A^* = d_{A \cup B}$, which is analyzed in Case 2, and for $d_{A \cup B}$ we have $d_A^* = 0$, which is analyzed in Case 1.
Hence, for any $d(A \cup B)$ the algorithm provides approximation at least $\gamma$.
\end{proof}

\clearpage
    \begin{table*}[htb!]
	\centering
	
    \setlength{\tabcolsep}{0.2em}
	\begin{tabular}{ccccc}
		\hline
		Name & Function & Dimension & $\kerleft(\vx)$, $\kerright(\vy)$
		\\
		\hline
		\hline
		
		$\cossim$ & $\frac{\langle \vx, \vy \rangle}{2\|\vx\|_2\|\vy\|_2} + \frac12$ & $d + 1$ &
		$\frac 1 {\sqrt{2} \|x\|}(x_1, x_2, \dotsc, x_d, \|x\|)$
		\\ \hline
		
		\multiline{
		$\rbf{\gamma}$\\
		~\citep*{RahimiR07}
		} & $e^{- \gamma \|x - y\|_2^2}$ & 
		$O(d \eps^{-2} \log{\eps^{-2}})$
		&
		\multiline{
			$\sqrt{\frac 2 {D}}(\cos(w_1^\top x + b_1), \ldots, \cos(w_D^\top x + b_D))$ \\ 
			$b_i \sim \mathcal{U}[0, 2\pi], w_i \sim \mathcal{N}(0, 2\gamma \mathcal{I}_{D})$
		}
		\\ \hline
		
		\multiline{
		$\laplacian{\gamma}$\\
		~\citep*{RahimiR07}
		} & $e^{-\gamma \|x - y\|_1}$ & 
		$O(d \eps^{-2} \log{\eps^{-2}})$
		&
		Random Binning Features
		\\ \hline
			$\Ltwosqr$-distance & $\|x - y\|_2^2$ & $d+2$ &
			\multiline{
			    $(\|x\|_2^2, 1, x_1, \ldots, x_d)$,\\
			    $(1, \|y\|_2^2, -2 y_1, \ldots, -2 y_d)$
			}\\
		\hline
	\end{tabular}
	\caption{Kernel similarities and $\Ltwosqr$-distance with their kernel-defining functions.}
	\label{tab:kernels_sim}
\end{table*}
\section{Extended Preliminaries}\label{sec:extended-preliminaries}

In this section we describe additional technical details regarding an appropriate generalization of MW/CKMM objectives to non-binary trees (Appendix~\ref{app:non-binary-mw-ckmm}), the inverse kernel trick used to speed up our algorithm (Appendix~\ref{sec:kernels}) and the full details regarding our experimental setup (Appendix~\ref{app:experimental-setup}).

\subsection{Extensions of MW/CKMM objectives to non-binary trees}\label{app:non-binary-mw-ckmm}
Since there is no standard way to compare binary trees with non-binary ones in terms of MW/CKMM objectives, and most algorithms (except \ghhc, \birch) produce binary trees, we give the expected score of random binarization for non-binary internal nodes 
\footnote{
This might put \ghhc and \birch at a disadvantage for the CKMM objective and it boosts the performance of \ghhc and \birch for MW because binarization improves the score. 
}.
More formally, in order to non-trivially extend the MW/CKMM objectives to non-binary trees we propose slight modifications of the objectives above which assign the value of a random partitioning to all triplets which share the same LCA. Denoting the set of such triplets as $S$, we define 
\[{\ourmw(\tree) := \sum_{(i < j < k) \notin S}\simty_{\lcaind{i}\lcaind{j}} + \frac13 \sum_{(i < j < k) \in S} (\simty_{ij} + \simty_{ik} + \simty_{jk})}\]
and 
\begin{align*}
\ourckmm(\tree) :=
&\sum_{(i < j < k) \notin S}(\dist_{\lcaind{i}\lcaind{k}} + \dist_{\lcaind{j}\lcaind{k}}) 
\notarxiv{\\&}+ \frac23\sum_{(i < j < k) \in S} (\dist_{ij} + \dist_{ik} + \dist_{jk}) + 2\sum_{i < j} \dist_{ij}
\end{align*}
We emphasize that these modifications coincide with MW/CKMM when restricted to binary trees. Furthermore they have the same optimum values since every non-binary tree can be binarized without changing these objectives.

\subsection{Inverse Kernel Trick}\label{sec:kernels}
Table~\ref{tab:kernels_sim} shows examples of kernalizable similarity and distance measures, which allows one to use them in \ouralgo.

\subsection{Experimental Setup}\label{app:experimental-setup}
\paragraph{Datasets}
Table~\ref{tab:datasets} shows information about datasets used in our experiments.

\begin{table*}[p]
\centering

\begin{tabular}{lccccc}
\hline
Dataset & $n$ & $d$ & \#classes & Setting & Comments \\
\hline
\hline

\glass & $2.1 \cdot 10^2$ & $10$ & $6$ & \generalSetup & \multiline{Chemical features of\\glass types} \\
\hline

\spambase & $4.6 \cdot 10^3$ & $57$ & $2$ & \generalSetup & \multiline{Features of\\spam/non-spam emails} \\
\hline

\imagenetVTwo & $1.0 \cdot 10^4$ & $512$ & $1000$ & \generalizationSetup & \multiline{Emeddings of images\\from ImageNetV2\\using ResNet34} \\
\hline

\nabirds & $4.8\cdot10^4$ & $512$ & $555$ & \shiftSetup & \multiline{Emeddings of images\\from \nabirds\\using ResNet34} \\
\hline

\sstTwo & $6.7 \cdot 10^4$ & $768$ & $2$ & \sentenceSetup & \multiline{Emeddings of phrases\\from \sstTwo\\using SBERT} \\
\hline

\aloi & $1.1 \cdot 10^5$ & $128$ & $1000$ & \generalSetup & \multiline{Pixels of images} \\
\hline

\covType & $5.8 \cdot 10^5$ & $54$ & $7$ & \generalSetup & \multiline{Features of\\forest covers} \\
\hline

\twitter & $1.2 \cdot 10^6$ & $200$ & -- & \wordSetup & \multiline{Glove embeddings of\\ words from \twitter} \\
\hline

\imagenet & $1.3 \cdot 10^6$ & $512$ & $1000$ & \classicSetup & \multiline{Emeddings of images \\from \imagenet ILSVRC\\using ResNet34} \\
\hline

\imagenetins & $1.3 \cdot 10^6$ & $2048$ & $1000$ & \classicSetup & \multiline{Emeddings of images \\from \imagenet ILSVRC\\using Inception\\~\citep{MonathZSMA19ghhc}} \\
\hline

\wikipedia & $4.5 \cdot 10^6$ & $100$ & -- & \wordSetup & \multiline{Word2vec embeddings of\\ words from english \wikipedia} \\
\hline

\end{tabular}
\caption{Information about datasets. Datasets are ordered by $n$ (number of points). Description of the various settings is in Subsection~\ref{subsec:datasets}.}
\label{tab:datasets}
\end{table*}

\paragraph{Similarity measures} For datasets consisting of embedding vectors arising from deep learning applications, we use a simple $\cossim$ as a measure of similarity.
This is motivated by popularity of $\cossim$ in such applications.
We remark that we didn't use any preprocessing of embedding vectors (e.g. centering, normalization, etc.) in order to focus on the gains from the hierarchical clustering algorithms as opposed to careful feature engineering. This leaves room for potential gains and further studies. 

For non-deep learning datasets it can be even more challenging to pick a suitable similarity measure. For completeness of our comparison with the previous work we use $\cossim$ for such datasets as well. However, under a more suitable choice of a similarity measure for each dataset one might expect to obtain higher quality hierarchies.


\paragraph{Hyperparameters} For all complex algorithms with a wide range of hyperparameters, (\ghhc, \grinch, \perch, \birch, \hdbscan) we use recommended approaches for hyperparameter tuning from the corresponding papers and/or repositories.

\clearpage
    \section{Full Results for MW/CKMM Objectives}\label{sec:additonal_results}

In this section we provide full details of our study of hierarchical clustering algorithms on a large number of datasets.
We provide results in terms of CKMM~(Table~\ref{tab:additional_ckmm}) and MW~(Table~\ref{tab:additional_mw}) objectives, while results for dendrogram purity are given in Section~\ref{sec:dp}.
For these objectives, we additionally provide results on various standard benchmarks (mostly non-deep learning), including smaller classic machine learning datasets: Table~\ref{tab:additional_ckmm_other_datasets} for CKMM and Table~\ref{tab:additional_mw_other_datasets} for MW.
We present complete tables of results (with standard deviations) for all algorithms under study.
Experimental setup is the same as described in Section~\ref{sec:experiments}.

We rank the algorithms in the following way: for each dataset, for each algorithm we compute the ratio of its (normalized) objective to the best objective on this dataset; after that, we average these values and sort algorithms based on this average value.
The goal of this additional normalization is to prevent ranking from being determined only by datasets where the algorithms achieve large objective.







\subsection{Hierarchical Agglomerative Clustering (HAC)}
\paragraph{Approach}
Hierarchical Agglomerative Clustering (HAC) algorithms are bottom-up: initially each data point forms its own cluster, then iteratively a pair of \emph{closest clusters} is selected (according to a certain rule specific to the algorithm) and merged.
HAC methods include classical approaches such as \averageLinkage, \wards, \singleLinkage, \completeLinkage, which vary in their definition of the closest cluster pair selected in every step.
\averageLinkage merges clusters with the smallest average distance between their data points, \wards uses a variance-based measure of proximity, \singleLinkage uses the closest pair of points between clusters and \completeLinkage uses the furthest pair points.
    
HAC methods have been studied extensively (see~\citep*{ManningRS08}) and are still one of the main workhorses in applications.
However, finding an exact solution in a every step in HAC methods requires quadratic running time for sufficiently high-dimensional data~\citep*{Indyk00, AlmanW15}, which can make them infeasible for large datasets (see e.g.~\citep*{AbboudCH19avgLsh} for approximation algorithms for HAC methods). Furthermore, HAC algorithms typically don't correspond to any global HC objectives and hence can be considered heuristics (except for \averageLinkage which gives $\nicefrac13$/$\nicefrac23$-approximations for MW/CKMM objectives~\citep*{MoseleyW17,CohenKMM19}). In beyond worst-case setting a variant of single-linkage clustering is known to have approximation guarantees for the hierarchical stochastic block model~\citep*{CohenKM17}.

\paragraph{Results}  Due to their superlinear (i.e. at least quadratic) scalability, HAC methods (\averageLinkage, \completeLinkage) are not applicable to large datasets.
However, for small datasets (\glass, \spambase, \imagenetVTwo) (Table~\ref{tab:additional_ckmm}, Table~\ref{tab:additional_mw}, Table~\ref{tab:additional_ckmm_other_datasets},  Table~\ref{tab:additional_mw_other_datasets}), they show advantage over many scalable algorithms (\grinch, \birch, \perch, \hdbscan, \randomCut) for both CKMM and MW objectives.
\wards also shows competitive performance for the MW objective.
In general, HAC methods provide a strong deterministic baseline which does not require hyperparameter tuning.

\subsection{Robust Single-Linkage from HDBSCAN (\hdbscan)}
\paragraph{Approach} 
Hierarchical Density-Based Spatial Clustering of Applications with Noise (HDBSCAN)~\citep*{McInnesHA17hdbscan} is a hierarchical extension of the extremely popular DBSCAN algorithm~\citep*{EstherKSX96}. 
HDBSCAN builds on top of DBSCAN by converting it into a HC algorithm and then uses the resulting hierarchy to extract a flat clustering based on the stability properties of the clusters.
We refer to the process of building the hierarchy before the extraction of the flat clustering as Robust Single-Linkage (\hdbscan).

\hdbscan is a spanning tree based method which redefines distance between points to make the spanning tree robust to outliers. 
The key idea is to define a \emph{mutual reachability distance} as $d^*(a, b) = max(core_k(a), core_k(b), 1 / \alpha \cdot d(a, b))$ where core distance $core_k(a)$ is a distance to the $k$-th nearest neighbor of $a$. In dense regions core distance is small and hence $d^*$ is equal to $d$. For low-density regions this transformation increases the distance and thus outliers are less likely to affect the structure of the spanning tree. The parameter $\alpha \in [1,2]$ can be used to adjust the mutual reachability distance to make it less sensitive to high core distance. This approach is guaranteed to converge to a density-based cluster tree~\citep*{ChaudhuriD10robustSingleL}, but doesn't have any known approximation for the MW/CKMM objectives.

\paragraph{Results}
In our experiments \hdbscan performs fairly well for MW/CKMM objectives, but is typically below the best approaches. Its running time degrades substantially in higher dimensions ($d \ge 200$), which made it infeasible to run \hdbscan on high-dimensional data.

\subsection{Bisecting K-means (\hkmeans)}
\paragraph{Approach} Bisecting $k$-means (\hkmeans) is an extremely popular approach to hierarchical clustering in practice since it can be implemented using vanilla $k$-means as a black box. At the top level it applies $k$-means for $k = 2$ and then continues in the same fashion recursively on each part of the resulting partition. 
For more details on \hkmeans see a classic monograph by Jain~\citep*{Jain10}.
Our implementation of \hkmeans is based on a popular K-Means++~\citep*{ArthurV06kmeansPP} initialization and Lloyd's algorithm.

Despite its widespread use in practice, \hkmeans is not known to correspond to any natural global optimization objective (with a notable exception of the recent work~\citep*{WangM20hkmeansObjective} who propose an objective specifically tailored to \hkmeans and show a constant-factor approximation). In particular, for the MW objective considered in this work negative results about approximation achieved by \hkmeans are known~\citep*{MoseleyW17}.

\paragraph{Results}
While \hkmeans doesn't have a formal approximation guarantee, on many datasets it gives surprisingly high-quality results with low variance for the MW/CKMM objectives (some exceptions to this rule are \twitter, \wikipedia and \imagenetVTwo for CKMM and \imagenetins and \spambase for MW). Analyzing this performance in a suitable beyond worst-case model might be an interesting subject for future work.

\subsection{Incremental Hierarchical Clustering (\perch, \grinch, \birch)}
\paragraph{Approach} Incremental hierarchical clustering approaches (\perch, \grinch, \birch) start from an empty HC tree and, for each new input data point, add it to the tree.
\perch~\citep*{KobrenMKM17perch} and \grinch~\citep*{MonathKGM19grinch} are nearest-neighbor based approaches which attach a new vertex to its nearest neighbor in the tree and then perform some tree rearrangements.

The main features of \perch~\citep*{KobrenMKM17perch} include balance-based rotations which preserve tree balance, bounding-box approximations which facilitate finding the nearest neighbor efficiently and masking-based rotations which locally improve the HC tree.
Under a certain separability assumption (nodes in the same cluster are closer than nodes in different clusters), \perch finds an HC tree with perfect dendrogram purity.

Similarly to \perch, \grinch~\citep*{MonathKGM19grinch} performs rotations which locally improve the HC tree. Additionally, \grinch uses a ``graft'' procedure. For a node $v$ it finds a node $v'$ such that $v'$ is more similar to $v$ than to its sibling. If such a node is found, $v'$ disconnects from its sibling and becomes a sibling of $v$.

Finally, \birch~\citep*{ZhangRL96birch} is another incremental approach optimized for database applications. This method incrementally passes new points from the root to a leaf cluster in a top-down fashion and builds a balanced tree with a fixed branching factor. 

\paragraph{Results}
For most datasets (with the exception of \spambase, \imagenetins, see Table~\ref{tab:additional_ckmm_other_datasets} and Table~\ref{tab:additional_mw_other_datasets}) incremental clustering methods \perch, \grinch and \birch show fairly high average quality according to MW/CKMM objectives, even though they haven't been designed with these objectives in mind.

\subsection{Gradient-Based Hyperbolic Hierarchical Clustering (\ghhc)}
\paragraph{Approach} Gradient-Based Hyperbolic Hierarchical Clustering \ghhc~\citet*{MonathZSMA19ghhc} is a gradient-based approach which optimizes a continuous tree representation in hyperbolic space, following a popular hyperbolic embedding approach pioneered in~\citep*{NickelK17poincare}. The key idea of the method is to define a continuous loss function and optimize it using gradient descent. The loss function of \ghhc is based on triplets and is inspired by CKMM/MW objectives. The algorithm is applicable to large datasets which do not fit in memory because it uses a mini-batch stochastic gradient descent.

\paragraph{Results}
In our experiments \ghhc gives fairly modest results for the CKMM/MW objectives, which might be due to the fact that it produces non-binary trees, resulting in relatively low scores. However, it works well for the dendrogram purity measure of performance, especially for \imagenet and \imagenetins (see Section~\ref{sec:dp} for more details).

\subsection{Scalable Single-Linkage (\affinityClustering)}
\paragraph{Approach} \affinityClustering~\citep*{BateniBDHKLM17} is a scalable graph-based algorithm for distributed systems based on Boruvka's algorithm for building a minimum spanning tree.
At every round, each cluster finds a cheapest edge connecting it to another cluster, and all clusters connected by such edges are merged together.
Boruvka's algorithm is closely related to Single-Linkage HAC, but instead of merging the pair of closest clusters, it merges all such pairs of clusters for which one cluster is the closest for some other one.
This results makes the algorithm more efficient for distributed settings:
    since the size of the smallest component at least doubles after each round, \affinityClustering is guaranteed to find a hierarchy after $O(\log |\items|)$ rounds.
In~\citet*{BateniBDHKLM17} an MPC version of Boruvkas's algorithm is given. 

\citep*{BateniBDHKLM17} also consider variants of \affinityClustering based on average and complete linkage for cluster similarity measure, but only provide experimental results. Due to lack of publicly available code and exact algorithm descriptions we didn't attempt to reproduce these results, which could be an interesting subject for a future study.

\paragraph{Results}
Similarly to HAC, in our setup \affinityClustering was not scalable to large datasets since on the complete graph of similarities/distances it requires at least quadratic running time. For MW/CKMM objectives, \affinityClustering doesn't show high performance on \imagenetVTwo and shows average performance on \glass and \spambase. While, similarly to Single Linkage HAC, \affinityClustering doesn't have formal guarantees for the MW/CKMM objectives, our results show that it substantially outperforms \Random solution.

We believe that variations of \affinityClustering with average and complete linkage mentioned above have a high potential for being competitive with some of the best approaches studied here. However, in our experience a suitable combination of top-down and bottom-up methods, as in \ouralgo, tends to produce the best results.

\subsection{Random Projection (\randomCut)}
\paragraph{Approach} \randomCut~\citep*{CharikarCNY2019} is our simplest and most efficient baseline, which can be implemented in almost linear time and parallelizes trivially. The algorithm computes projections of the input vectors $v_1, \dots, v_n \in \mathbb R^d$ on a single random Gaussian $g \in \mathbb R^d$ and then only works with these 1D projections $x_1 = \langle v_1, g \rangle, \dots, x_n = \langle v_n, g \rangle$. For these projected points it applies a uniformly random cut recursively. It picks a uniformly random point $u$ in the range containing $(x_1, \dots, x_n)$, splits the data into two subtrees $L = \{x_i | x_i < u\}$ and $R = \{x_i | x_i \ge u\}$ and continues recursively on $L$ and $R$. 
For smooth similarity measure \randomCut is known to perform strictly better than the naive \nicefrac13 \Random/\averageLinkage baseline for MW. See~\citep*{CharikarCNY2019} for more details.

\paragraph{Results} Despite its simplicity \randomCut serves well as the first cut in our evaluation of various methods. Our results show that, as predicted in~\citet*{CharikarCNY2019}, it substantially outperforms \Random and in certain cases even beats much more sophisticated approaches.

\begin{table*}[p]
\centering

\setlength{\tabcolsep}{0.05em}
\begin{tabular}{cPPPPPP}
\hline
\MtabSetup & \MimagenetSmall & \MimagenetVTwo & \Mnabirds &  \Mtwitter & \Mwikipedia & \MsstTwo \\
\hline
\hline
\textbf{\ouralgo} & \boldmath{\cellcolor[rgb]{0.73, 0.89, 0.70}\multiline{$.30 \pm .00$\\$.95 \pm .00$}} & \boldmath{\cellcolor[rgb]{0.17, 0.58, 0.30}\multiline{$.83 \pm .00$\\$.99 \pm .00$}} & \boldmath{\cellcolor[rgb]{0.23, 0.64, 0.34}\multiline{$.71 \pm .00$\\$.97 \pm .00$}} & \boldmath{\cellcolor[rgb]{0.29, 0.69, 0.38}\multiline{$.60 \pm .00$\\$.97 \pm .00$}} & \boldmath{\cellcolor[rgb]{0.33, 0.71, 0.40}\multiline{$.58 \pm .00$\\$.94 \pm .00$}} & \boldmath{\cellcolor[rgb]{0.46, 0.77, 0.47}\multiline{$.49 \pm .00$\\$.97 \pm .00$}}\\
\hline
\averageLinkage & -- & \cellcolor[rgb]{0.31, 0.70, 0.39}\multiline{$.59 \pm .00$\\$.97 \pm .00$} & -- & -- & -- & --\\
\hline
\textbf{\oursatalgo} & \cellcolor[rgb]{0.80, 0.92, 0.77}\multiline{$.23 \pm .00$\\$.95 \pm .00$} & \cellcolor[rgb]{0.77, 0.91, 0.74}\multiline{$.26 \pm .00$\\$.95 \pm .00$} & \cellcolor[rgb]{0.34, 0.71, 0.41}\multiline{$.57 \pm .00$\\$.96 \pm .00$} & \cellcolor[rgb]{0.90, 0.96, 0.89}\multiline{$.12 \pm .00$\\$.93 \pm .00$} & \cellcolor[rgb]{0.72, 0.89, 0.69}\multiline{$.30 \pm .00$\\$.90 \pm .00$} & \cellcolor[rgb]{0.52, 0.80, 0.52}\multiline{$.46 \pm .01$\\$.96 \pm .00$}\\
\hline
\hdbscan & -- & \cellcolor[rgb]{0.23, 0.64, 0.35}\multiline{$.70 \pm .00$\\$.98 \pm .00$} & \cellcolor[rgb]{0.53, 0.80, 0.52}\multiline{$.45 \pm .00$\\$.95 \pm .00$} & -- & -- & \cellcolor[rgb]{0.88, 0.95, 0.85}\multiline{$.15 \pm .00$\\$.94 \pm .00$}\\
\hline
\textbf{\bbisection} & \cellcolor[rgb]{0.80, 0.92, 0.77}\multiline{$.23 \pm .00$\\$.95 \pm .00$} & \cellcolor[rgb]{0.77, 0.91, 0.74}\multiline{$.26 \pm .00$\\$.95 \pm .00$} & \cellcolor[rgb]{0.34, 0.71, 0.41}\multiline{$.57 \pm .00$\\$.96 \pm .00$} & \cellcolor[rgb]{0.91, 0.97, 0.90}\multiline{$.10 \pm .01$\\$.93 \pm .00$} & \cellcolor[rgb]{0.86, 0.94, 0.83}\multiline{$.17 \pm .01$\\$.88 \pm .00$} & \cellcolor[rgb]{0.52, 0.80, 0.52}\multiline{$.46 \pm .01$\\$.96 \pm .00$}\\
\hline
\completeLinkage & -- & \cellcolor[rgb]{0.47, 0.78, 0.48}\multiline{$.48 \pm .00$\\$.97 \pm .00$} & -- & -- & -- & --\\
\hline
\hkmeans & \cellcolor[rgb]{0.80, 0.92, 0.77}\multiline{$.23 \pm .01$\\$.95 \pm .00$} & \cellcolor[rgb]{0.77, 0.91, 0.75}\multiline{$.26 \pm .00$\\$.95 \pm .00$} & \cellcolor[rgb]{0.29, 0.69, 0.38}\multiline{$.61 \pm .01$\\$.96 \pm .00$} & \cellcolor[rgb]{0.91, 0.97, 0.89}\multiline{$.10 \pm .01$\\$.93 \pm .00$} & \cellcolor[rgb]{0.91, 0.97, 0.90}\multiline{$.10 \pm .00$\\$.87 \pm .00$} & \cellcolor[rgb]{0.52, 0.80, 0.52}\multiline{$.45 \pm .01$\\$.96 \pm .00$}\\
\hline
\multiline{\birch\\non-binary} & \cellcolor[rgb]{0.84, 0.94, 0.82}\multiline{$.18 \pm .03$\\$.95 \pm .00$} & \cellcolor[rgb]{0.89, 0.96, 0.87}\multiline{$.14 \pm .05$\\$.95 \pm .00$} & \cellcolor[rgb]{0.69, 0.88, 0.67}\multiline{$.33 \pm .20$\\$.94 \pm .02$} & \cellcolor[rgb]{0.81, 0.93, 0.79}\multiline{$.22 \pm .12$\\$.94 \pm .01$} & \cellcolor[rgb]{0.85, 0.94, 0.82}\multiline{$.18 \pm .05$\\$.88 \pm .01$} & \cellcolor[rgb]{0.78, 0.91, 0.75}\multiline{$.25 \pm .01$\\$.95 \pm .00$}\\
\hline
\multiline{\ghhc\\non-binary} & \cellcolor[rgb]{0.84, 0.94, 0.82}\multiline{$.19 \pm .03$\\$.95 \pm .00$} & \cellcolor[rgb]{0.85, 0.94, 0.83}\multiline{$.18 \pm .05$\\$.95 \pm .00$} & \cellcolor[rgb]{0.45, 0.76, 0.46}\multiline{$.50 \pm .24$\\$.96 \pm .02$} & \cellcolor[rgb]{0.94, 0.98, 0.92}\multiline{$.06 \pm .01$\\$.92 \pm .00$} & \cellcolor[rgb]{0.93, 0.97, 0.91}\multiline{$.07 \pm .02$\\$.86 \pm .00$} & \cellcolor[rgb]{0.79, 0.92, 0.76}\multiline{$.25 \pm .05$\\$.95 \pm .00$}\\
\hline
\grinch & \cellcolor[rgb]{0.92, 0.97, 0.91}\multiline{$.08 \pm .03$\\$.94 \pm .00$} & \cellcolor[rgb]{0.93, 0.97, 0.92}\multiline{$.06 \pm .02$\\$.94 \pm .00$} & \cellcolor[rgb]{0.50, 0.79, 0.50}\multiline{$.47 \pm .07$\\$.95 \pm .01$} & \cellcolor[rgb]{0.96, 0.99, 0.96}\multiline{$.01 \pm .01$\\$.92 \pm .00$} & \cellcolor[rgb]{0.93, 0.97, 0.91}\multiline{$.08 \pm .03$\\$.86 \pm .01$} & \cellcolor[rgb]{0.87, 0.95, 0.85}\multiline{$.16 \pm .02$\\$.94 \pm .00$}\\
\hline
\perch & \cellcolor[rgb]{0.94, 0.98, 0.93}\multiline{$.05 \pm .02$\\$.94 \pm .00$} & \cellcolor[rgb]{0.92, 0.97, 0.91}\multiline{$.08 \pm .03$\\$.94 \pm .00$} & \cellcolor[rgb]{0.65, 0.86, 0.63}\multiline{$.36 \pm .13$\\$.94 \pm .01$} & \cellcolor[rgb]{0.97, 0.99, 0.96}\multiline{$.00 \pm .02$\\$.92 \pm .00$} & -- & \cellcolor[rgb]{0.88, 0.95, 0.86}\multiline{$.15 \pm .02$\\$.94 \pm .00$}\\
\hline
\randomCut & \cellcolor[rgb]{0.94, 0.98, 0.92}\multiline{$.06 \pm .02$\\$.94 \pm .00$} & \cellcolor[rgb]{0.92, 0.97, 0.91}\multiline{$.09 \pm .03$\\$.94 \pm .00$} & \cellcolor[rgb]{0.90, 0.96, 0.88}\multiline{$.12 \pm .06$\\$.92 \pm .01$} & \cellcolor[rgb]{0.91, 0.96, 0.89}\multiline{$.11 \pm .05$\\$.93 \pm .00$} & \cellcolor[rgb]{0.80, 0.92, 0.78}\multiline{$.23 \pm .06$\\$.89 \pm .01$} & \cellcolor[rgb]{0.93, 0.97, 0.92}\multiline{$.07 \pm .02$\\$.94 \pm .00$}\\
\hline
\wards & -- & \cellcolor[rgb]{0.84, 0.94, 0.81}\multiline{$.19 \pm .00$\\$.95 \pm .00$} & -- & -- & -- & --\\
\hline
\affinityClustering & -- & \cellcolor[rgb]{0.88, 0.95, 0.86}\multiline{$.15 \pm .00$\\$.95 \pm .00$} & -- & -- & -- & --\\
\hline
\singleLinkage & -- & \cellcolor[rgb]{0.92, 0.97, 0.90}\multiline{$.09 \pm .00$\\$.94 \pm .00$} & -- & -- & -- & --\\
\hline
\Random & \cellcolor[rgb]{0.97, 0.99, 0.96}\multiline{$.00 \pm .00$\\$.94 \pm .00$} & \cellcolor[rgb]{0.97, 0.99, 0.96}\multiline{$.00 \pm .00$\\$.94 \pm .00$} & \cellcolor[rgb]{0.97, 0.99, 0.96}\multiline{$.00 \pm .00$\\$.91 \pm .00$} & \cellcolor[rgb]{0.97, 0.99, 0.96}\multiline{$.00 \pm .00$\\$.92 \pm .00$} & \cellcolor[rgb]{0.97, 0.99, 0.96}\multiline{$.00 \pm .00$\\$.85 \pm .00$} & \cellcolor[rgb]{0.97, 0.99, 0.96}\multiline{$.00 \pm .00$\\$.93 \pm .00$}\\

\hline
$n\approx$ & $1.2\cdot10^6$ & $10^4$ & $5 \cdot 10^4$ & $1.3\cdot10^6$ & $4.5\cdot10^6$ & $7 \cdot 10^4$\\
$d$ & $512$ & $512$ & $512$ & $200$ & $100$ & $768$\\
\#classes & $10^3$ & $10^3$ & $555$ & -- & -- & $2$\\
\hline
\end{tabular}

\caption{Extended version of Table~\ref{tab:distance_based}. Normalized/unnormalized ($\aproxourckmm[\normchar]$/$\aproxourckmm$) distance-based CKMM objectives (generalized to allow non-binary trees) under squared Euclidean distance with standard deviations. \ouralgo and \oursatalgo show stable high-quality results with low variance. \hkmeans also has low variance, \birch and \ghhc have higher variance. HAC approaches (\averageLinkage, \completeLinkage, \wards, \affinityClustering, \singleLinkage) are deterministic.}
\label{tab:additional_ckmm}
\end{table*}
\begin{table*}[p]
\centering

\setlength{\tabcolsep}{0.05em}
\begin{tabular}{cPPPPPP}
\hline
\MtabSetup & \MimagenetSmall & \MimagenetVTwo & \Mnabirds &  \Mtwitter & \Mwikipedia & \MsstTwo \\
\hline
\hline

\textbf{\ouralgo} & \boldmath{\cellcolor[rgb]{0.59, 0.83, 0.57}\multiline{$.40 \pm .00$\\$.98 \pm .00$}} & \boldmath{\cellcolor[rgb]{0.44, 0.76, 0.46}\multiline{$.51 \pm .00$\\$.99 \pm .00$}} & \boldmath{\cellcolor[rgb]{0.22, 0.63, 0.34}\multiline{$.73 \pm .00$\\$.99 \pm .00$}} & \boldmath{\cellcolor[rgb]{0.54, 0.81, 0.53}\multiline{$.44 \pm .00$\\$.97 \pm .00$}} & \boldmath{\cellcolor[rgb]{0.60, 0.84, 0.58}\multiline{$.40 \pm .00$\\$.96 \pm .00$}} & \boldmath{\cellcolor[rgb]{0.45, 0.76, 0.46}\multiline{$.51 \pm .00$\\$.96 \pm .00$}}\\
\hline
\completeLinkage & -- & \boldmath{\cellcolor[rgb]{0.43, 0.76, 0.45}\multiline{$.51 \pm .00$\\$.99 \pm .00$}} & -- & -- & -- & --\\
\hline
\hkmeans & \cellcolor[rgb]{0.63, 0.85, 0.61}\multiline{$.37 \pm .00$\\$.98 \pm .00$} & \cellcolor[rgb]{0.61, 0.84, 0.59}\multiline{$.39 \pm .00$\\$.98 \pm .00$} & \cellcolor[rgb]{0.23, 0.64, 0.34}\multiline{$.71 \pm .01$\\$.99 \pm .00$} & \cellcolor[rgb]{0.76, 0.90, 0.73}\multiline{$.27 \pm .03$\\$.96 \pm .00$} & \cellcolor[rgb]{0.62, 0.84, 0.60}\multiline{$.38 \pm .00$\\$.96 \pm .00$} & \cellcolor[rgb]{0.53, 0.80, 0.52}\multiline{$.45 \pm .01$\\$.95 \pm .00$}\\
\hline
\textbf{\bbisection} & \cellcolor[rgb]{0.63, 0.85, 0.61}\multiline{$.37 \pm .00$\\$.98 \pm .00$} & \cellcolor[rgb]{0.61, 0.84, 0.59}\multiline{$.39 \pm .01$\\$.98 \pm .00$} & \cellcolor[rgb]{0.25, 0.66, 0.36}\multiline{$.67 \pm .00$\\$.99 \pm .00$} & \cellcolor[rgb]{0.80, 0.92, 0.77}\multiline{$.23 \pm .00$\\$.95 \pm .00$} & \boldmath{\cellcolor[rgb]{0.60, 0.84, 0.58}\multiline{$.40 \pm .00$\\$.96 \pm .00$}} & \cellcolor[rgb]{0.52, 0.80, 0.52}\multiline{$.46 \pm .01$\\$.95 \pm .00$}\\
\hline
\textbf{\oursatalgo} & \cellcolor[rgb]{0.63, 0.85, 0.61}\multiline{$.37 \pm .00$\\$.98 \pm .00$} & \cellcolor[rgb]{0.61, 0.84, 0.59}\multiline{$.39 \pm .01$\\$.98 \pm .00$} & \cellcolor[rgb]{0.25, 0.66, 0.36}\multiline{$.67 \pm .00$\\$.99 \pm .00$} & \cellcolor[rgb]{0.80, 0.92, 0.77}\multiline{$.23 \pm .00$\\$.95 \pm .00$} & \boldmath{\cellcolor[rgb]{0.60, 0.84, 0.58}\multiline{$.40 \pm .00$\\$.96 \pm .00$}} & \cellcolor[rgb]{0.52, 0.80, 0.52}\multiline{$.46 \pm .01$\\$.95 \pm .00$}\\
\hline
\averageLinkage & -- & \cellcolor[rgb]{0.63, 0.85, 0.61}\multiline{$.38 \pm .00$\\$.98 \pm .00$} & -- & -- & -- & --\\
\hline
\wards & -- & \cellcolor[rgb]{0.66, 0.86, 0.64}\multiline{$.35 \pm .00$\\$.98 \pm .00$} & -- & -- & -- & --\\
\hline
\multiline{\ghhc\\non-binary} & \cellcolor[rgb]{0.75, 0.90, 0.72}\multiline{$.28 \pm .02$\\$.98 \pm .00$} & \cellcolor[rgb]{0.71, 0.88, 0.68}\multiline{$.31 \pm .03$\\$.98 \pm .00$} & \cellcolor[rgb]{0.33, 0.71, 0.40}\multiline{$.58 \pm .26$\\$.98 \pm .01$} & \cellcolor[rgb]{0.89, 0.96, 0.87}\multiline{$.14 \pm .01$\\$.95 \pm .00$} & \cellcolor[rgb]{0.89, 0.96, 0.87}\multiline{$.13 \pm .01$\\$.95 \pm .00$} & \cellcolor[rgb]{0.77, 0.91, 0.74}\multiline{$.26 \pm .05$\\$.93 \pm .00$}\\
\hline
\multiline{\birch\\non-binary} & \cellcolor[rgb]{0.80, 0.92, 0.77}\multiline{$.23 \pm .03$\\$.98 \pm .00$} & \cellcolor[rgb]{0.81, 0.92, 0.78}\multiline{$.23 \pm .04$\\$.98 \pm .00$} & \cellcolor[rgb]{0.62, 0.85, 0.60}\multiline{$.38 \pm .17$\\$.98 \pm .01$} & \cellcolor[rgb]{0.91, 0.97, 0.89}\multiline{$.10 \pm .05$\\$.95 \pm .00$} & \cellcolor[rgb]{0.74, 0.90, 0.71}\multiline{$.29 \pm .04$\\$.96 \pm .00$} & \cellcolor[rgb]{0.77, 0.91, 0.74}\multiline{$.26 \pm .06$\\$.93 \pm .01$}\\
\hline
\grinch & \cellcolor[rgb]{0.91, 0.97, 0.90}\multiline{$.10 \pm .02$\\$.98 \pm .00$} & \cellcolor[rgb]{0.90, 0.96, 0.88}\multiline{$.12 \pm .03$\\$.98 \pm .00$} & \cellcolor[rgb]{0.38, 0.73, 0.43}\multiline{$.54 \pm .08$\\$.98 \pm .00$} & \cellcolor[rgb]{0.92, 0.97, 0.90}\multiline{$.09 \pm .03$\\$.95 \pm .00$} & \cellcolor[rgb]{0.92, 0.97, 0.90}\multiline{$.09 \pm .02$\\$.94 \pm .00$} & \cellcolor[rgb]{0.86, 0.95, 0.84}\multiline{$.17 \pm .03$\\$.93 \pm .00$}\\
\hline
\hdbscan & -- & \cellcolor[rgb]{0.85, 0.94, 0.83}\multiline{$.17 \pm .00$\\$.98 \pm .00$} & \cellcolor[rgb]{0.86, 0.95, 0.84}\multiline{$.16 \pm .00$\\$.97 \pm .00$} & -- & -- & \cellcolor[rgb]{0.83, 0.93, 0.80}\multiline{$.20 \pm .00$\\$.93 \pm .00$}\\
\hline
\perch & \cellcolor[rgb]{0.93, 0.97, 0.91}\multiline{$.07 \pm .02$\\$.97 \pm .00$} & \cellcolor[rgb]{0.90, 0.96, 0.88}\multiline{$.13 \pm .03$\\$.98 \pm .00$} & \cellcolor[rgb]{0.56, 0.82, 0.55}\multiline{$.43 \pm .14$\\$.98 \pm .01$} & \cellcolor[rgb]{0.97, 0.99, 0.96}\multiline{$.01 \pm .01$\\$.94 \pm .00$} & -- & \cellcolor[rgb]{0.88, 0.95, 0.86}\multiline{$.15 \pm .02$\\$.92 \pm .00$}\\
\hline
\randomCut & \cellcolor[rgb]{0.94, 0.98, 0.92}\multiline{$.06 \pm .01$\\$.97 \pm .00$} & \cellcolor[rgb]{0.94, 0.98, 0.92}\multiline{$.06 \pm .01$\\$.97 \pm .00$} & \cellcolor[rgb]{0.90, 0.96, 0.88}\multiline{$.13 \pm .06$\\$.97 \pm .00$} & \cellcolor[rgb]{0.94, 0.98, 0.92}\multiline{$.06 \pm .01$\\$.94 \pm .00$} & \cellcolor[rgb]{0.89, 0.96, 0.87}\multiline{$.13 \pm .02$\\$.95 \pm .00$} & \cellcolor[rgb]{0.93, 0.97, 0.91}\multiline{$.07 \pm .02$\\$.92 \pm .00$}\\
\hline
\singleLinkage & -- & \cellcolor[rgb]{0.91, 0.97, 0.90}\multiline{$.10 \pm .00$\\$.98 \pm .00$} & -- & -- & -- & --\\
\hline
\affinityClustering & -- & \cellcolor[rgb]{0.92, 0.97, 0.90}\multiline{$.09 \pm .00$\\$.97 \pm .00$} & -- & -- & -- & --\\
\hline
\Random & \cellcolor[rgb]{0.97, 0.99, 0.96}\multiline{$.00 \pm .00$\\$.97 \pm .00$} & \cellcolor[rgb]{0.97, 0.99, 0.96}\multiline{$.00 \pm .00$\\$.97 \pm .00$} & \cellcolor[rgb]{0.97, 0.99, 0.96}\multiline{$.00 \pm .00$\\$.96 \pm .00$} & \cellcolor[rgb]{0.97, 0.99, 0.96}\multiline{$.00 \pm .00$\\$.94 \pm .00$} & \cellcolor[rgb]{0.97, 0.99, 0.96}\multiline{$.00 \pm .00$\\$.94 \pm .00$} & \cellcolor[rgb]{0.97, 0.99, 0.96}\multiline{$.00 \pm .00$\\$.91 \pm .00$}\\

\hline
$n\approx$ & $1.2\cdot10^6$ & $10^4$ & $5 \cdot 10^4$ & $1.3\cdot10^6$ & $4.5\cdot10^6$ & $7 \cdot 10^4$\\
$d$ & $512$ & $512$ & $512$ & $200$ & $100$ & $768$\\
\#classes & $10^3$ & $10^3$ & $555$ & -- & -- & $2$\\
\hline
\end{tabular}

\caption{Extended version of Table~\ref{tab:mw_sim}. Normalized/unnormalized ($\aproxourmw[\normchar]$/$\aproxourmw$) similarity-based MW objectives (generalized to allow non-binary trees) under cosine similarity with standard deviation. \ouralgo and \oursatalgo show stable high quality results with low variance. \hkmeans also has low variance, \birch, \grinch and \perch have higher variance. HAC approaches (\completeLinkage, \averageLinkage, \wards, \singleLinkage, \affinityClustering) are deterministic.}
\label{tab:additional_mw}
\end{table*}
\begin{table*}[p]
\centering
\setlength{\tabcolsep}{0.2em}
\begin{tabular}{cPPPPP}
\hline 
\MtabSetup & \Mglass & \Mspambase & \Maloi & \McovType & \MimageNetAdditional \\
\hline 
\hline
 
\textbf{\ouralgo} & \boldmath{\cellcolor[rgb]{0.09, 0.51, 0.24}\multiline{$.98 \pm .00$\\$1.0 \pm .00$}} & \cellcolor[rgb]{0.09, 0.51, 0.24}\multiline{$.98 \pm .00$\\$.99 \pm .00$} & \boldmath{\cellcolor[rgb]{0.26, 0.67, 0.37}\multiline{$.64 \pm .00$\\$.94 \pm .00$}} & \boldmath{\cellcolor[rgb]{0.16, 0.57, 0.29}\multiline{$.84 \pm .00$\\$.96 \pm .00$}} & \boldmath{\cellcolor[rgb]{0.39, 0.74, 0.43}\multiline{$.54 \pm .00$\\$.96 \pm .00$}}\\
\hline
\averageLinkage & \boldmath{\cellcolor[rgb]{0.09, 0.51, 0.24}\multiline{$.98 \pm .00$\\$1.0 \pm .00$}} & \boldmath{\cellcolor[rgb]{0.09, 0.50, 0.23}\multiline{$.99 \pm .00$\\$1.0 \pm .00$}} & -- & -- & --\\
\hline
\completeLinkage & \cellcolor[rgb]{0.15, 0.55, 0.28}\multiline{$.88 \pm .00$\\$.97 \pm .00$} & \cellcolor[rgb]{0.11, 0.52, 0.25}\multiline{$.95 \pm .00$\\$.98 \pm .00$} & -- & -- & --\\
\hline
\hkmeans & \cellcolor[rgb]{0.16, 0.57, 0.29}\multiline{$.86 \pm .02$\\$.96 \pm .01$} & \cellcolor[rgb]{0.10, 0.52, 0.25}\multiline{$.96 \pm .02$\\$.99 \pm .01$} & \cellcolor[rgb]{0.38, 0.73, 0.42}\multiline{$.55 \pm .07$\\$.93 \pm .01$} & \cellcolor[rgb]{0.17, 0.58, 0.30}\multiline{$.83 \pm .00$\\$.96 \pm .00$} & \cellcolor[rgb]{0.57, 0.82, 0.56}\multiline{$.41 \pm .00$\\$.95 \pm .00$}\\
\hline
\wards & \cellcolor[rgb]{0.18, 0.59, 0.30}\multiline{$.82 \pm .00$\\$.95 \pm .00$} & \cellcolor[rgb]{0.16, 0.57, 0.29}\multiline{$.85 \pm .00$\\$.95 \pm .00$} & -- & -- & --\\
\hline
\textbf{\oursatalgo} & \cellcolor[rgb]{0.19, 0.60, 0.31}\multiline{$.79 \pm .00$\\$.95 \pm .00$} & \cellcolor[rgb]{0.14, 0.55, 0.27}\multiline{$.89 \pm .00$\\$.96 \pm .00$} & \cellcolor[rgb]{0.43, 0.75, 0.45}\multiline{$.52 \pm .00$\\$.93 \pm .00$} & \cellcolor[rgb]{0.18, 0.59, 0.31}\multiline{$.80 \pm .00$\\$.95 \pm .00$} & \cellcolor[rgb]{0.69, 0.87, 0.66}\multiline{$.33 \pm .00$\\$.94 \pm .00$}\\
\hline
\affinityClustering & \cellcolor[rgb]{0.21, 0.62, 0.33}\multiline{$.75 \pm .00$\\$.94 \pm .00$} & \cellcolor[rgb]{0.15, 0.55, 0.28}\multiline{$.88 \pm .00$\\$.96 \pm .00$} & -- & -- & --\\
\hline
\textbf{\bbisection} & \cellcolor[rgb]{0.20, 0.61, 0.32}\multiline{$.77 \pm .01$\\$.94 \pm .00$} & \cellcolor[rgb]{0.16, 0.57, 0.29}\multiline{$.84 \pm .00$\\$.95 \pm .00$} & \cellcolor[rgb]{0.43, 0.75, 0.45}\multiline{$.52 \pm .00$\\$.93 \pm .00$} & \cellcolor[rgb]{0.18, 0.59, 0.31}\multiline{$.80 \pm .00$\\$.95 \pm .00$} & \cellcolor[rgb]{0.70, 0.88, 0.67}\multiline{$.32 \pm .00$\\$.94 \pm .00$}\\
\hline
\singleLinkage & \cellcolor[rgb]{0.21, 0.62, 0.33}\multiline{$.74 \pm .00$\\$.93 \pm .00$} & \cellcolor[rgb]{0.18, 0.59, 0.30}\multiline{$.82 \pm .00$\\$.94 \pm .00$} & -- & -- & --\\
\hline
\multiline{\birch\\non-binary} & \cellcolor[rgb]{0.19, 0.60, 0.31}\multiline{$.79 \pm .18$\\$.95 \pm .05$} & \cellcolor[rgb]{0.94, 0.98, 0.92}\multiline{$.06 \pm .15$\\$.70 \pm .05$} & \cellcolor[rgb]{0.55, 0.81, 0.54}\multiline{$.43 \pm .05$\\$.91 \pm .01$} & \cellcolor[rgb]{0.35, 0.72, 0.41}\multiline{$.57 \pm .14$\\$.89 \pm .04$} & \cellcolor[rgb]{0.45, 0.77, 0.46}\multiline{$.50 \pm .06$\\$.96 \pm .01$}\\
\hline
\hdbscan & \cellcolor[rgb]{0.12, 0.53, 0.26}\multiline{$.93 \pm .00$\\$.98 \pm .00$} & \cellcolor[rgb]{0.74, 0.90, 0.72}\multiline{$.28 \pm .29$\\$.77 \pm .09$} & \cellcolor[rgb]{0.67, 0.87, 0.64}\multiline{$.35 \pm .05$\\$.90 \pm .01$} & \cellcolor[rgb]{0.30, 0.69, 0.39}\multiline{$.59 \pm .05$\\$.90 \pm .01$} & --\\
\hline
\multiline{\ghhc\\non-binary} & \cellcolor[rgb]{0.51, 0.80, 0.51}\multiline{$.46 \pm .17$\\$.86 \pm .04$} & \cellcolor[rgb]{0.29, 0.69, 0.38}\multiline{$.60 \pm .05$\\$.87 \pm .01$} & \cellcolor[rgb]{0.61, 0.84, 0.59}\multiline{$.39 \pm .03$\\$.91 \pm .00$} & \cellcolor[rgb]{0.23, 0.65, 0.35}\multiline{$.69 \pm .04$\\$.92 \pm .01$} & \cellcolor[rgb]{0.84, 0.94, 0.81}\multiline{$.19 \pm .07$\\$.93 \pm .01$}\\
\hline
\randomCut & \cellcolor[rgb]{0.55, 0.81, 0.54}\multiline{$.43 \pm .14$\\$.85 \pm .04$} & \cellcolor[rgb]{0.13, 0.54, 0.27}\multiline{$.91 \pm .10$\\$.97 \pm .03$} & \cellcolor[rgb]{0.71, 0.88, 0.68}\multiline{$.31 \pm .06$\\$.89 \pm .01$} & \cellcolor[rgb]{0.41, 0.75, 0.44}\multiline{$.53 \pm .16$\\$.88 \pm .04$} & \cellcolor[rgb]{0.90, 0.96, 0.88}\multiline{$.12 \pm .07$\\$.93 \pm .01$}\\
\hline
\grinch & \cellcolor[rgb]{0.13, 0.54, 0.27}\multiline{$.91 \pm .06$\\$.98 \pm .02$} & \cellcolor[rgb]{0.80, 0.92, 0.78}\multiline{$.23 \pm .36$\\$.75 \pm .12$} & \cellcolor[rgb]{0.69, 0.87, 0.66}\multiline{$.33 \pm .09$\\$.90 \pm .01$} & \cellcolor[rgb]{0.23, 0.65, 0.35}\multiline{$.70 \pm .03$\\$.92 \pm .01$} & \cellcolor[rgb]{0.97, 0.99, 0.96}\multiline{$-.01 \pm .03$\\$.91 \pm .00$}\\
\hline
\perch & \cellcolor[rgb]{0.82, 0.93, 0.80}\multiline{$.20 \pm .07$\\$.79 \pm .02$} & \cellcolor[rgb]{0.87, 0.95, 0.85}\multiline{$.15 \pm .06$\\$.73 \pm .02$} & \cellcolor[rgb]{0.75, 0.90, 0.72}\multiline{$.28 \pm .04$\\$.89 \pm .01$} & \cellcolor[rgb]{0.25, 0.67, 0.36}\multiline{$.65 \pm .08$\\$.91 \pm .02$} & \cellcolor[rgb]{0.96, 0.99, 0.95}\multiline{$.01 \pm .01$\\$.92 \pm .00$}\\
\hline
\Random & \cellcolor[rgb]{0.97, 0.99, 0.96}\multiline{$.00 \pm .01$\\$.74 \pm .00$} & \cellcolor[rgb]{0.97, 0.99, 0.96}\multiline{$.00 \pm .00$\\$.68 \pm .00$} & \cellcolor[rgb]{0.97, 0.99, 0.96}\multiline{$.00 \pm .00$\\$.84 \pm .00$} & \cellcolor[rgb]{0.97, 0.99, 0.96}\multiline{$.00 \pm .00$\\$.74 \pm .00$} & \cellcolor[rgb]{0.97, 0.99, 0.96}\multiline{$.00 \pm .00$\\$.91 \pm .00$}\\

\hline
$n\approx$ & $214$ & $4601$ & $108 \cdot 10^3$ & $581 \cdot 10^3$ & $1.3 \cdot 10^6$\\
$d$ & $10$ & $57$ & $128$ & $54$ & $2048$ \\
\#classes & $6$ & $2$ & $1000$ & $7$ & $1000$\\
\hline
\end{tabular}

\caption{\small Normalized/unnormalized ($\aproxourckmm[\normchar]$/$\aproxourckmm$) distance-based CKMM objectives (generalized to allow non-binary trees) under squared Euclidean distance with standard deviation for additional datasets. On $\aproxourckmm[\normchar]$ \ouralgo outperforms other approaches on medium and large datasets by 1-9\% and shows low variance. Among other scalable algorithms, \hkmeans has high-quality results with high variance. Among non-scalable algorithms, basic deterministic HAC methods (\averageLinkage, \completeLinkage) show competitive performance. Our worst-case theoretical algorithm \oursatalgo also shows substantial gains with low variance. Even the simplest 1D random projection technique (\randomCut) gives non-trivial results on all non-deep learning datasets (\glass, \spambase, \aloi, \covType).}
\label{tab:additional_ckmm_other_datasets}
\end{table*}
\begin{table*}[p]
\centering
\setlength{\tabcolsep}{0.2em}
\begin{tabular}{cPPPPP}
\hline 
\MtabSetup & \Mglass & \Mspambase & \Maloi & \McovType & \MimageNetAdditional \\
\hline 
\hline 
 
\textbf{\ouralgo} & \boldmath{\cellcolor[rgb]{0.10, 0.52, 0.25}\multiline{$.96 \pm .00$\\$1.0 \pm .00$}} & \boldmath{\cellcolor[rgb]{0.10, 0.51, 0.24}\multiline{$.97 \pm .00$\\$1.0 \pm .00$}} & \boldmath{\cellcolor[rgb]{0.28, 0.68, 0.38}\multiline{$.62 \pm .00$\\$.97 \pm .00$}} & \boldmath{\cellcolor[rgb]{0.20, 0.61, 0.32}\multiline{$.76 \pm .00$\\$.99 \pm .00$}} & \boldmath{\cellcolor[rgb]{0.26, 0.67, 0.37}\multiline{$.64 \pm .00$\\$.99 \pm .00$}}\\
\hline
\averageLinkage & \boldmath{\cellcolor[rgb]{0.10, 0.52, 0.25}\multiline{$.96 \pm .00$\\$1.0 \pm .00$}} & \boldmath{\cellcolor[rgb]{0.10, 0.51, 0.24}\multiline{$.97 \pm .00$\\$1.0 \pm .00$}} & -- & -- & --\\
\hline
\completeLinkage & \cellcolor[rgb]{0.16, 0.57, 0.29}\multiline{$.86 \pm .00$\\$1.0 \pm .00$} & \cellcolor[rgb]{0.11, 0.52, 0.25}\multiline{$.95 \pm .00$\\$1.0 \pm .00$} & -- & -- & --\\
\hline
\textbf{\bbisection} & \boldmath{\cellcolor[rgb]{0.10, 0.52, 0.25}\multiline{$.96 \pm .00$\\$1.0 \pm .00$}} & \boldmath{\cellcolor[rgb]{0.10, 0.51, 0.24}\multiline{$.97 \pm .00$\\$1.0 \pm .00$}} & \boldmath{\cellcolor[rgb]{0.28, 0.68, 0.38}\multiline{$.62 \pm .00$\\$.97 \pm .00$}} & \boldmath{\cellcolor[rgb]{0.20, 0.61, 0.32}\multiline{$.76 \pm .00$\\$.99 \pm .00$}} & \cellcolor[rgb]{0.79, 0.92, 0.77}\multiline{$.24 \pm .00$\\$.98 \pm .00$}\\
\hline
\textbf{\oursatalgo} & \boldmath{\cellcolor[rgb]{0.10, 0.52, 0.25}\multiline{$.96 \pm .00$\\$1.0 \pm .00$}} & \boldmath{\cellcolor[rgb]{0.10, 0.51, 0.24}\multiline{$.97 \pm .00$\\$1.0 \pm .00$}} & \boldmath{\cellcolor[rgb]{0.28, 0.68, 0.38}\multiline{$.62 \pm .00$\\$.97 \pm .00$}} & \boldmath{\cellcolor[rgb]{0.20, 0.61, 0.32}\multiline{$.76 \pm .00$\\$.99 \pm .00$}} & \cellcolor[rgb]{0.79, 0.92, 0.77}\multiline{$.24 \pm .00$\\$.98 \pm .00$}\\
\hline
\wards & \cellcolor[rgb]{0.17, 0.58, 0.29}\multiline{$.83 \pm .00$\\$1.0 \pm .00$} & \cellcolor[rgb]{0.20, 0.61, 0.32}\multiline{$.77 \pm .00$\\$.99 \pm .00$} & -- & -- & --\\
\hline
\hdbscan & \cellcolor[rgb]{0.14, 0.55, 0.27}\multiline{$.89 \pm .00$\\$1.0 \pm .00$} & \cellcolor[rgb]{0.17, 0.58, 0.29}\multiline{$.84 \pm .00$\\$.99 \pm .00$} & \cellcolor[rgb]{0.63, 0.85, 0.61}\multiline{$.37 \pm .05$\\$.95 \pm .00$} & \cellcolor[rgb]{0.62, 0.85, 0.60}\multiline{$.38 \pm .02$\\$.98 \pm .00$} & --\\
\hline
\affinityClustering & \cellcolor[rgb]{0.23, 0.64, 0.35}\multiline{$.70 \pm .00$\\$1.0 \pm .00$} & \cellcolor[rgb]{0.25, 0.66, 0.36}\multiline{$.67 \pm .00$\\$.98 \pm .00$} & -- & -- & --\\
\hline
\multiline{\birch\\non-binary} & \cellcolor[rgb]{0.20, 0.62, 0.32}\multiline{$.75 \pm .17$\\$1.0 \pm .00$} & \cellcolor[rgb]{0.73, 0.89, 0.70}\multiline{$.29 \pm .30$\\$.97 \pm .01$} & \cellcolor[rgb]{0.48, 0.78, 0.48}\multiline{$.48 \pm .06$\\$.96 \pm .00$} & \cellcolor[rgb]{0.45, 0.76, 0.46}\multiline{$.50 \pm .07$\\$.98 \pm .00$} & \cellcolor[rgb]{0.87, 0.95, 0.85}\multiline{$.15 \pm .13$\\$.97 \pm .00$}\\
\hline
\grinch & \cellcolor[rgb]{0.14, 0.55, 0.27}\multiline{$.89 \pm .06$\\$1.0 \pm .00$} & \cellcolor[rgb]{0.57, 0.82, 0.56}\multiline{$.42 \pm .26$\\$.97 \pm .01$} & \cellcolor[rgb]{0.65, 0.86, 0.63}\multiline{$.36 \pm .01$\\$.95 \pm .00$} & \cellcolor[rgb]{0.43, 0.75, 0.45}\multiline{$.52 \pm .14$\\$.98 \pm .01$} & \cellcolor[rgb]{0.95, 0.98, 0.94}\multiline{$.03 \pm .01$\\$.97 \pm .00$}\\
\hline
\hkmeans & \cellcolor[rgb]{0.17, 0.58, 0.29}\multiline{$.83 \pm .02$\\$1.0 \pm .00$} & \cellcolor[rgb]{0.97, 0.99, 0.96}\multiline{$-.15 \pm .00$\\$.95 \pm .00$} & \cellcolor[rgb]{0.28, 0.68, 0.38}\multiline{$.61 \pm .00$\\$.97 \pm .00$} & \cellcolor[rgb]{0.24, 0.66, 0.35}\multiline{$.67 \pm .00$\\$.99 \pm .00$} & \cellcolor[rgb]{0.97, 0.99, 0.96}\multiline{$-.03 \pm .00$\\$.97 \pm .00$}\\
\hline
\multiline{\ghhc\\non-binary} & \cellcolor[rgb]{0.54, 0.81, 0.53}\multiline{$.44 \pm .17$\\$1.0 \pm .00$} & \cellcolor[rgb]{0.97, 0.99, 0.96}\multiline{$-.12 \pm .04$\\$.95 \pm .00$} & \cellcolor[rgb]{0.45, 0.77, 0.46}\multiline{$.50 \pm .04$\\$.96 \pm .00$} & \cellcolor[rgb]{0.24, 0.65, 0.35}\multiline{$.68 \pm .03$\\$.99 \pm .00$} & \cellcolor[rgb]{0.97, 0.99, 0.96}\multiline{$.00 \pm .04$\\$.97 \pm .00$}\\
\hline
\singleLinkage & \cellcolor[rgb]{0.27, 0.68, 0.37}\multiline{$.63 \pm .00$\\$1.0 \pm .00$} & \cellcolor[rgb]{0.92, 0.97, 0.90}\multiline{$.09 \pm .00$\\$.96 \pm .00$} & -- & -- & --\\
\hline
\perch & \cellcolor[rgb]{0.83, 0.93, 0.81}\multiline{$.20 \pm .06$\\$1.0 \pm .00$} & \cellcolor[rgb]{0.93, 0.97, 0.92}\multiline{$.07 \pm .08$\\$.96 \pm .00$} & \cellcolor[rgb]{0.68, 0.87, 0.66}\multiline{$.33 \pm .06$\\$.95 \pm .00$} & \cellcolor[rgb]{0.28, 0.68, 0.38}\multiline{$.63 \pm .04$\\$.99 \pm .00$} & \cellcolor[rgb]{0.96, 0.98, 0.95}\multiline{$.02 \pm .01$\\$.97 \pm .00$}\\
\hline
\randomCut & \cellcolor[rgb]{0.57, 0.82, 0.56}\multiline{$.42 \pm .14$\\$1.0 \pm .00$} & \cellcolor[rgb]{0.94, 0.98, 0.92}\multiline{$.06 \pm .25$\\$.96 \pm .01$} & \cellcolor[rgb]{0.81, 0.92, 0.78}\multiline{$.22 \pm .08$\\$.94 \pm .01$} & \cellcolor[rgb]{0.48, 0.78, 0.48}\multiline{$.48 \pm .10$\\$.98 \pm .00$} & \cellcolor[rgb]{0.96, 0.98, 0.95}\multiline{$.02 \pm .04$\\$.97 \pm .00$}\\
\hline
\Random & \cellcolor[rgb]{0.97, 0.99, 0.96}\multiline{$.00 \pm .01$\\$1.0 \pm .00$} & \cellcolor[rgb]{0.97, 0.99, 0.96}\multiline{$.00 \pm .00$\\$.95 \pm .00$} & \cellcolor[rgb]{0.97, 0.99, 0.96}\multiline{$.00 \pm .00$\\$.92 \pm .00$} & \cellcolor[rgb]{0.97, 0.99, 0.96}\multiline{$.00 \pm .00$\\$.96 \pm .00$} & \cellcolor[rgb]{0.97, 0.99, 0.96}\multiline{$.00 \pm .00$\\$.97 \pm .00$}\\

\hline 
$n\approx$ & $214$ & $4601$ & $108 \cdot 10^3$ & $581 \cdot 10^3$ & $1.3 \cdot 10^6$\\
$d$ & $10$ & $57$ & $128$ & $54$ & $2048$ \\
\#classes & $6$ & $2$ & $1000$ & $7$ & $1000$\\
\hline
\end{tabular}

\caption{\small Normalized/unnormalized ($\aproxourmw[\normchar]$/$\aproxourmw$) similarity-based MW objectives (generalized to allow non-binary trees) under cosine similarity with standard deviation for additional datasets. On $\aproxourmw[\normchar]$ \ouralgo outperforms other approaches on medium and large datasets by 1-49\% and shows low variance. Among other scalable algorithms, \hkmeans shows high-quality results with high variance. Among non-scalable algorithms basic deterministic HAC methods (\averageLinkage, \completeLinkage, \wards) and robust single linkage (\hdbscan) show competitive performance. Our worst-case theoretical algorithm \oursatalgo works as good as \bbisection. Even the simplest 1D random projection technique (\randomCut) gives non-trivial results on multiple general datasets (\glass, \aloi, \covType). Somewhat surprisingly, some methods (\hkmeans, \ghhc) perform worse than \random for some datasets (\spambase, \imageNetAdditional). This might be due to the fact that for these datasets optimization under cosine similarities is substantially different from optimization using distances.}
\label{tab:additional_mw_other_datasets}
\end{table*}

\clearpage

    \section{Results for dendrogram purity}\label{sec:dp}

\emph{Dendrogram purity} ($\purity$)~\citep*{HellerG05} is a hierarchical clustering objective defined as maximizing:
$${\purity(\tree) = \frac 1 {\sum_{i=1}^\clustercount |C_i|^2} \sum_{i=1}^\clustercount \sum_{e_1,e_2 \in C_i} \frac {|C_i \cap \lca_\tree(e_1, e_2)|} {|\lca_\tree(e_1, e_2)|}},$$
where $C_1, \dots C_K$ are subsets corresponding to a ground truth clustering of the dataset.
Intuitively, for each cluster $C$, for each pair of elements $e_1, e_2 \in C$, we compute the fraction of elements in $\lca_\tree(e_1, e_2)$ which also belong to $C$.

Dendrogram purity is a widely used measure of the quality of a hierarchical clustering.
However, despite its widespread use, it has a number of shortcomings as a quality metric for HC and is arguably more suitable in the case when the hierarchy is ultimately used to extract a single flat clustering.
In particular, dendrogram purity:
\begin{itemize}
    \item Requires information about ground truth clusters to be available. This is typically not the case in realistic machine learning applications, where the notion of classes can be either ambiguous, hard to obtain or otherwise unavailable. We contrast this the MW/CKMM objectives considered in this work which only requires a similarity/distance measure to be defined for the input data points.
    \item Can't be stated as a clear optimization problem -- if ground truth clusters are given as a part of the input then the problem then optimization becomes trivial. Otherwise, the objective can't be passed to the algorithm.
    \item Is designed how well a given tree represents a certain ground truth flat clustering. It doesn't say anything about the hierarchical relationship of these clusters within the tree -- a perfect DP score can be achieved by imposing an arbitrary hierarchy on top of the ground truth clusters.
    \item Due to its normalization is not well-suited for datasets where the ground truth clusters are highly imbalanced.
\end{itemize}

Due to the reasons described above we believe that MW and CKMM objectives are better suitable for evaluating the quality of hierarchical clustering since they attempt to capture the quality of the entire tree rather than of its flattening. Hence, we primarily focus on these objectives in our experiments. Nevertheless, in order to facilitate comparisoin with the previous work which uses dendrogram purity we also present full experimental results under this measure.

Results on datasets consisting of high-dimensional deep embedding vectors are given in Table~\ref{tab:additional_dp} and Table~\ref{tab:additional_dp_other_datasets} shows results on some other popular machine learning datasets.
Despite not being tailored to the dendrogram purity performance measure our algorithms show highly competitive results, outperforming other algorithms on most datasets.

\begin{table*}[p]
\centering

\setlength{\tabcolsep}{0.2em}
\begin{tabular}{cPPPPP}
\hline 
\MtabSetup & \MimagenetSmall & \MimagenetBig  & \MimagenetVTwo & \Mnabirds & \MsstTwo \\
\hline

\textbf{\ouralgo} & \cellcolor[rgb]{0.85, 0.94, 0.83}$.18 \pm .00$ & \cellcolor[rgb]{0.90, 0.96, 0.88}$.12 \pm .00$ & \cellcolor[rgb]{0.51, 0.79, 0.51}$.46 \pm .00$ & \boldmath{\cellcolor[rgb]{0.91, 0.96, 0.89}$.11 \pm .00$} & \boldmath{\cellcolor[rgb]{0.22, 0.63, 0.34}$.72 \pm .00$}\\
\textbf{\oursatalgo} & \cellcolor[rgb]{0.85, 0.94, 0.83}$.18 \pm .00$ & \cellcolor[rgb]{0.90, 0.96, 0.88}$.13 \pm .00$ & \cellcolor[rgb]{0.51, 0.79, 0.51}$.46 \pm .00$ & \boldmath{\cellcolor[rgb]{0.91, 0.96, 0.89}$.11 \pm .00$} & \cellcolor[rgb]{0.23, 0.64, 0.35}$.70 \pm .04$\\
\textbf{\bbisection} & \cellcolor[rgb]{0.85, 0.94, 0.83}$.18 \pm .00$ & \cellcolor[rgb]{0.90, 0.96, 0.88}$.12 \pm .00$ & \cellcolor[rgb]{0.51, 0.79, 0.51}$.46 \pm .00$ & \boldmath{\cellcolor[rgb]{0.91, 0.96, 0.89}$.11 \pm .00$} & \cellcolor[rgb]{0.23, 0.64, 0.35}$.70 \pm .04$\\
\wards & -- & -- & \boldmath{\cellcolor[rgb]{0.49, 0.78, 0.49}$.48 \pm .00$} & -- & --\\
\averageLinkage & -- & -- & \cellcolor[rgb]{0.51, 0.79, 0.51}$.46 \pm .00$ & -- & --\\
\grinch & \boldmath{\cellcolor[rgb]{0.83, 0.93, 0.81}$.20 \pm .00$} & \boldmath{\cellcolor[rgb]{0.57, 0.82, 0.56}$.42 \pm .00$} & \cellcolor[rgb]{0.60, 0.83, 0.58}$.40 \pm .00$ & \cellcolor[rgb]{0.92, 0.97, 0.91}$.08 \pm .00$ & \cellcolor[rgb]{0.33, 0.71, 0.40}$.57 \pm .02$\\
\hkmeans & \cellcolor[rgb]{0.84, 0.94, 0.81}$.19 \pm .00$ & \cellcolor[rgb]{0.90, 0.96, 0.89}$.11 \pm .00$ & \cellcolor[rgb]{0.70, 0.88, 0.68}$.32 \pm .00$ & \cellcolor[rgb]{0.92, 0.97, 0.91}$.08 \pm .00$ & \cellcolor[rgb]{0.23, 0.65, 0.35}$.69 \pm .05$\\
\ghhc & \boldmath{\cellcolor[rgb]{0.83, 0.93, 0.80}$.20 \pm .00$} & \cellcolor[rgb]{0.62, 0.85, 0.60}$.38 \pm .00$ & \cellcolor[rgb]{0.78, 0.91, 0.76}$.25 \pm .00$ & \cellcolor[rgb]{0.93, 0.97, 0.91}$.08 \pm .03$ & \cellcolor[rgb]{0.33, 0.71, 0.40}$.57 \pm .01$\\
\singleLinkage & -- & -- & \cellcolor[rgb]{0.60, 0.83, 0.58}$.40 \pm .00$ & -- & --\\
\affinityClustering & -- & -- & \cellcolor[rgb]{0.62, 0.84, 0.60}$.38 \pm .00$ & -- & --\\
\perch & \cellcolor[rgb]{0.92, 0.97, 0.90}$.09 \pm .01$ & \cellcolor[rgb]{0.83, 0.93, 0.80}$.20 \pm .01$ & \cellcolor[rgb]{0.62, 0.85, 0.60}$.38 \pm .00$ & \cellcolor[rgb]{0.92, 0.97, 0.91}$.08 \pm .00$ & \cellcolor[rgb]{0.34, 0.71, 0.41}$.57 \pm .02$\\
\birch & \cellcolor[rgb]{0.85, 0.94, 0.82}$.18 \pm .00$ & \cellcolor[rgb]{0.93, 0.97, 0.92}$.06 \pm .01$ & \cellcolor[rgb]{0.85, 0.94, 0.83}$.18 \pm .00$ & \cellcolor[rgb]{0.95, 0.98, 0.93}$.04 \pm .00$ & \cellcolor[rgb]{0.33, 0.71, 0.40}$.58 \pm .02$\\
\completeLinkage & -- & -- & \cellcolor[rgb]{0.82, 0.93, 0.79}$.21 \pm .00$ & -- & --\\
\hdbscan & -- & -- & \cellcolor[rgb]{0.66, 0.86, 0.63}$.36 \pm .00$ & \cellcolor[rgb]{0.96, 0.99, 0.95}$.01 \pm .00$ & \cellcolor[rgb]{0.41, 0.75, 0.44}$.52 \pm .00$\\
\randomCut & \cellcolor[rgb]{0.97, 0.99, 0.96}$.00 \pm .00$ & \cellcolor[rgb]{0.97, 0.99, 0.96}$.00 \pm .00$ & \cellcolor[rgb]{0.91, 0.97, 0.89}$.10 \pm .00$ & \cellcolor[rgb]{0.96, 0.99, 0.95}$.01 \pm .00$ & \cellcolor[rgb]{0.43, 0.75, 0.45}$.52 \pm .01$\\
\Random & \cellcolor[rgb]{0.97, 0.99, 0.96}$.00 \pm .00$ & \cellcolor[rgb]{0.97, 0.99, 0.96}$.00 \pm .00$ & \cellcolor[rgb]{0.94, 0.98, 0.93}$.05 \pm .05$ & \cellcolor[rgb]{0.96, 0.99, 0.95}$.01 \pm .00$ & \cellcolor[rgb]{0.43, 0.76, 0.45}$.51 \pm .00$\\

\hline
$n\approx$ & $1.2\cdot10^6$ & $1.2\cdot10^6$ & $10^4$ & $5 \cdot 10^4$ &  $7 \cdot 10^4$ \\
$d$ & $512$ & $2048$ & $512$ & $512$ & $768$\\
\#classes & $10^3$ & $10^3$ & $10^3$ & $555$  & $2$\\
\hline
\end{tabular}
\caption{
Dendrogram purity ($\purity$) with standard deviation for datasets with ground truth classes. All algorithms show stable results for all datasets except \sstTwo on which \oursatalgo and \hkmeans have higher variance.}
\label{tab:additional_dp}
\end{table*}
\begin{table*}[p]
\centering
\setlength{\tabcolsep}{0.2em}
\begin{tabular}{cPPPP}
\hline 
\MtabSetup & \Mglass & \Mspambase & \Maloi & \McovType \\
\hline 
 
\textbf{\ouralgo} & \cellcolor[rgb]{0.41, 0.75, 0.44}$.53 \pm .00$ & \boldmath{\cellcolor[rgb]{0.20, 0.62, 0.32}$.75 \pm .00$} & \cellcolor[rgb]{0.58, 0.83, 0.57}$.41 \pm .00$ & \boldmath{\cellcolor[rgb]{0.48, 0.78, 0.48}$.48 \pm .00$}\\
\textbf{\oursatalgo} & \cellcolor[rgb]{0.41, 0.75, 0.44}$.53 \pm .00$ & \boldmath{\cellcolor[rgb]{0.20, 0.62, 0.32}$.75 \pm .00$} & \cellcolor[rgb]{0.58, 0.83, 0.57}$.41 \pm .00$ & \boldmath{\cellcolor[rgb]{0.48, 0.78, 0.48}$.48 \pm .00$}\\
\textbf{\bbisection} & \cellcolor[rgb]{0.41, 0.75, 0.44}$.53 \pm .00$ & \boldmath{\cellcolor[rgb]{0.20, 0.62, 0.32}$.75 \pm .00$} & \cellcolor[rgb]{0.58, 0.83, 0.57}$.41 \pm .00$ & \boldmath{\cellcolor[rgb]{0.48, 0.78, 0.48}$.48 \pm .00$}\\
\averageLinkage & \cellcolor[rgb]{0.44, 0.76, 0.46}$.51 \pm .00$ & \boldmath{\cellcolor[rgb]{0.20, 0.62, 0.32}$.75 \pm .00$} & -- & --\\
\wards & \cellcolor[rgb]{0.43, 0.76, 0.45}$.51 \pm .00$ & \cellcolor[rgb]{0.22, 0.64, 0.34}$.71 \pm .00$ & -- & --\\
\singleLinkage & \boldmath{\cellcolor[rgb]{0.38, 0.73, 0.43}$.54 \pm .00$} & \cellcolor[rgb]{0.26, 0.67, 0.37}$.65 \pm .00$ & -- & --\\
\hdbscan & \cellcolor[rgb]{0.46, 0.77, 0.47}$.49 \pm .01$ & \cellcolor[rgb]{0.28, 0.68, 0.38}$.61 \pm .02$ & \boldmath{\cellcolor[rgb]{0.29, 0.69, 0.38}$.61 \pm .00$} & \cellcolor[rgb]{0.51, 0.80, 0.51}$.46 \pm .00$\\
\completeLinkage & \cellcolor[rgb]{0.48, 0.78, 0.48}$.48 \pm .00$ & \cellcolor[rgb]{0.25, 0.67, 0.36}$.66 \pm .00$ & -- & --\\
\grinch & \cellcolor[rgb]{0.41, 0.75, 0.44}$.53 \pm .01$ & \cellcolor[rgb]{0.29, 0.69, 0.38}$.60 \pm .02$ & \cellcolor[rgb]{0.44, 0.76, 0.46}$.51 \pm .00$ & \cellcolor[rgb]{0.52, 0.80, 0.52}$.45 \pm .01$\\
\ghhc & \cellcolor[rgb]{0.49, 0.79, 0.49}$.47 \pm .02$ & \cellcolor[rgb]{0.28, 0.68, 0.38}$.61 \pm .01$ & \cellcolor[rgb]{0.50, 0.79, 0.50}$.47 \pm .00$ & \cellcolor[rgb]{0.52, 0.80, 0.52}$.45 \pm .00$\\
\perch & \cellcolor[rgb]{0.49, 0.79, 0.49}$.47 \pm .02$ & \cellcolor[rgb]{0.29, 0.69, 0.38}$.60 \pm .02$ & \cellcolor[rgb]{0.53, 0.80, 0.52}$.45 \pm .00$ & \cellcolor[rgb]{0.52, 0.80, 0.52}$.45 \pm .00$\\
\affinityClustering & \cellcolor[rgb]{0.41, 0.75, 0.44}$.52 \pm .00$ & \cellcolor[rgb]{0.34, 0.71, 0.41}$.57 \pm .00$ & -- & --\\
\hkmeans & \cellcolor[rgb]{0.43, 0.75, 0.45}$.52 \pm .01$ & \cellcolor[rgb]{0.27, 0.68, 0.37}$.63 \pm .00$ & \cellcolor[rgb]{0.71, 0.89, 0.69}$.31 \pm .00$ & \cellcolor[rgb]{0.54, 0.81, 0.53}$.44 \pm .00$\\
\randomCut & \cellcolor[rgb]{0.60, 0.84, 0.58}$.40 \pm .02$ & \cellcolor[rgb]{0.29, 0.69, 0.38}$.61 \pm .02$ & \cellcolor[rgb]{0.96, 0.99, 0.95}$.01 \pm .00$ & \cellcolor[rgb]{0.54, 0.81, 0.53}$.45 \pm .01$\\
\birch & \cellcolor[rgb]{0.58, 0.83, 0.57}$.41 \pm .02$ & \cellcolor[rgb]{0.38, 0.73, 0.43}$.54 \pm .00$ & \cellcolor[rgb]{0.89, 0.96, 0.87}$.13 \pm .01$ & \cellcolor[rgb]{0.54, 0.81, 0.53}$.44 \pm .00$\\
\Random & \cellcolor[rgb]{0.67, 0.87, 0.65}$.34 \pm .01$ & \cellcolor[rgb]{0.38, 0.73, 0.43}$.54 \pm .00$ & \cellcolor[rgb]{0.97, 0.99, 0.96}$.01 \pm .00$ & \cellcolor[rgb]{0.55, 0.81, 0.54}$.44 \pm .00$\\

\hline
$n\approx$ & $214$ & $4601$ & $108 \cdot 10^3$ & $581 \cdot 10^3$ \\
$d$ & $10$ & $57$ & $128$ & $54$ \\
\#classes & $6$ & $2$ & $1000$ & $7$\\
\hline
\end{tabular}
\caption{Dendrogram purity ($\purity$) with standard deviation for datasets with ground truth classes for additional datasets. All algorithms have low variance.  Our \ouralgo and \oursatalgo show high-quality results. Nearest neighbor based algorithms (\hdbscan, \grinch) show good performance for flat quality measure $\purity$. Other HAC methods (\averageLinkage, \wards, \singleLinkage) show competitive performance.}
\label{tab:additional_dp_other_datasets}
\end{table*}

\clearpage

    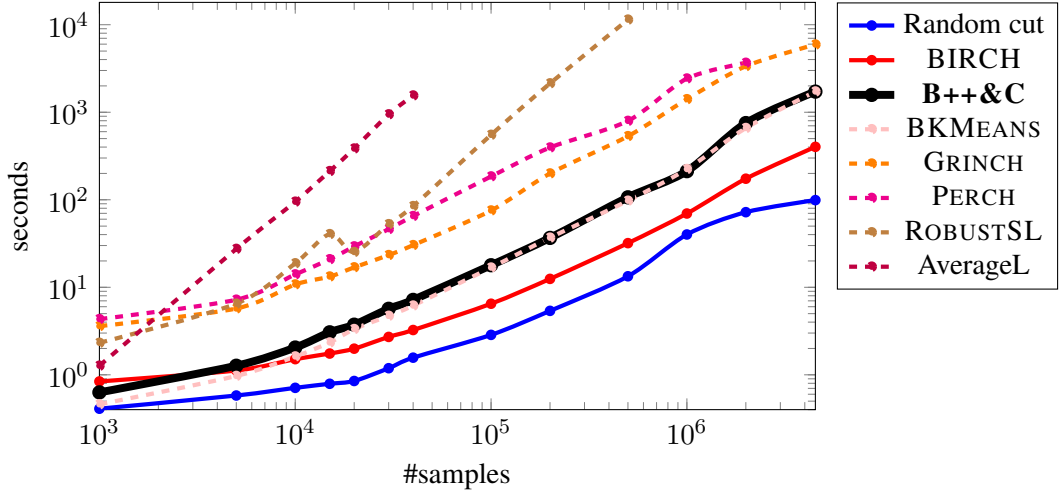
\begin{figure*}[htb!]
\centering

\begin{tikzpicture}
\begin{axis}
[xlabel={\#samples},ylabel={seconds}, 
xmode=log, ymode=log,
legend pos=outer north east,
width=11.1cm, height=7cm,
cycle list={
{blue},
{red},
{black},
{pink},
{orange},
{magenta},
{brown},
{purple},
},
xmin=1000,xmax=4530000,ymin=0.4,ymax=18000]

\addplot+[smooth,mark=*,mark options={scale=0.6},ultra thick] plot coordinates
{(1000, 0.41) (5000, 0.58) (10000, 0.71) (15000, 0.79) (20000, 0.85) (30000, 1.19) (40000, 1.57) (100000, 2.86) (200000, 5.38) (500000, 13.36) (1000000, 40.12) (2000000, 72.17) (4530000, 99.08)};
\addlegendentry{Random cut}

\addplot+[smooth,mark=*,mark options={scale=0.6},ultra thick] plot coordinates
{(1000, 0.84) (5000, 1.12) (10000, 1.51) (15000, 1.75) (20000, 1.99) (30000, 2.71) (40000, 3.25) (100000, 6.49) (200000, 12.48) (500000, 31.95) (1000000, 69.68) (2000000, 174.08) (4530000, 402.45)};
\addlegendentry{\birch}

\addplot+[smooth,mark=*,mark options={scale=0.6},ultra thick,line width=3pt] plot coordinates
{(1000, 0.63) (5000, 1.27) (10000, 2.05) (15000, 3.06) (20000, 3.74) (30000, 5.68) (40000, 7.18) (100000, 17.87) (200000, 36.67) (500000, 106.68) (1000000, 212.29) (2000000, 752.48) (4530000, 1721.40)};
\addlegendentry{\textbf{\ouralgo}}

\addplot+[smooth,mark=*,mark options={scale=0.6},ultra thick,dashed] plot coordinates
{(1000, 0.46) (5000, 0.97) (10000, 1.62) (15000, 2.34) (20000, 3.33) (30000, 4.72) (40000, 6.18) (100000, 16.79) (200000, 36.89) (500000, 97.48) (1000000, 224.37) (2000000, 660.03) (4530000, 1725.40)};
\addlegendentry{\hkmeans}

\addplot+[smooth,mark=*,mark options={scale=0.6},ultra thick,dashed] plot coordinates
{(1000, 3.60) (5000, 5.72) (10000, 10.77) (15000, 13.26) (20000, 16.89) (30000, 23.32) (40000, 30.30) (100000, 75.39) (200000, 199.76) (500000, 532.71) (1000000, 1408.46) (2000000, 3359.78) (4530000, 5937.47)};
\addlegendentry{\grinch}

\addplot+[smooth,mark=*,mark options={scale=0.6},ultra thick,dashed] plot coordinates
{(1000, 4.33) (5000, 7.21) (10000, 14.03) (15000, 20.90) (20000, 29.31) (30000, 46.17) (40000, 65.11) (100000, 184.68) (200000, 395.28) (500000, 796.27) (1000000, 2426.96) (2000000, 3671.11)};
\addlegendentry{\perch}

\addplot+[smooth,mark=*,mark options={scale=0.6},ultra thick,dashed] plot coordinates
{(1000, 2.32) (5000, 6.50) (10000, 18.79) (15000, 40.46) (20000, 25.29) (30000, 52.57) (40000, 84.55) (100000, 551.85) (200000, 2137.00) (500000, 11427.65)};
\addlegendentry{\hdbscan}

\addplot+[smooth,mark=*,mark options={scale=0.6},ultra thick,dashed] plot coordinates
{(1000, 1.27) (5000, 27.45) (10000, 94.96) (15000, 212.36) (20000, 384.69) (30000, 944.96) (40000, 1552.92)};
\addlegendentry{AverageL}

\end{axis}
\end{tikzpicture}
\caption{Scalability on subsampled \wikipedia depending on the sample size. \randomCut and \birch are the fastest, while \ouralgo and \hkmeans are close and show similar performance. Other algorithms are substantially worse E.g. the next algorithm \grinch is $3$ times slower. \hdbscan scales on small and medium high-dimensional datasets. HAC methods work only for small datasets. Hyperparameters are chosen for the best performance. Missing points are due to timeouts (5 hours) or memory limits (90 Gb).}
\label{plt:scalability}
\end{figure*}
\section{Scalability}\label{sec:scalability}

We separately show scalability of various methods on the largest used dataset (\wikipedia $n \approx 4.5\cdot10^6$, $d=100$) depending on the sample size (Figure~\ref{plt:scalability}).
We show running time of the top-5 fastest approaches (\randomCut, \birch, \ouralgo, \hkmeans, \grinch) and three examples approaches that scale differently: \perch, Robust Single linkage (\hdbscan) and a quadratic HAC algorithm (\avglink).
Since \oursatalgo uses gradient-descent based implementation with inverse kernel trick, it has similar performance to \ouralgo and therefore is not shown on the plot.
\randomCut and \birch show the best running time, and \ouralgo performs as well as \hkmeans.

\clearpage

\fi
\end{document}